\documentclass[12pt,final]{l4dc2021} 
\makeatletter
\def\set@curr@file#1{\def\@curr@file{#1}} %temp workaround for 2019 latex release
\makeatother
\usepackage[load-configurations=version-1]{siunitx} % newer version
\title[Finite-time System Identification and Adaptive Control in ARX Systems]{Finite-time System Identification and Adaptive Control in Autoregressive Exogenous Systems}
\usepackage{times}

\usepackage{booktabs}
\usepackage{amsmath}
\usepackage{amssymb}
\usepackage{nccmath}
\usepackage{autobreak}
\usepackage{mathtools}
\usepackage{xcolor}

\usepackage{wrapfig}

\usepackage{enumitem}
\usepackage{mwe}

\usepackage[export]{adjustbox}
\usepackage{times}
\usepackage{algorithm}
\usepackage{algorithmic}

\newcommand{\bigzero}{\mbox{\normalfont\Large\bfseries 0}}
\newcommand\x{\times}
\usepackage{mdframed}
\newcommand{\ttt}{\tilde{\Theta}}
\newcommand{\tth}{\hat{\Theta}}

\newcommand{\G}{\mathcal{G}}
\usepackage{tikz}

\usetikzlibrary{shapes,arrows}
\usetikzlibrary{arrows,calc,positioning}
\usetikzlibrary{fit}
\tikzstyle[nome]=[anchor=west]
\tikzset{
    block/.style = {draw, rectangle,
        minimum height=0.4cm,
        minimum width=0.5cm},
    input/.style = {coordinate,node distance=1cm},
    output/.style = {coordinate,node distance=4cm},
    arrow/.style={draw, -latex,node distance=2cm},
    pinstyle/.style = {pin edge={latex-, black,node distance=2cm}},
    sum/.style = {draw, circle,minimum size=1pt, node distance=1cm},}

\newcommand{\OFU}{\textsc{\small{OFU}}\xspace}

\newcommand{\strong}{\underline\alpha_{loss}}
\newcommand{\smooth}{\wb\alpha_{loss}}
\newcommand{\Markov}{\mathbf{G}}
\newcommand{\Mcontrol}{\mathbf{M}}
\newcommand{\Mcontrolset}{\mathcal{M}}
\newcommand{\proj}{proj}

\newcommand{\nature}{b}
\newcommand{\nat}{\wb b}
\newcommand{\Tburn}{{T_{burn}}}
\newcommand{\Tw}{{T_{w}}}

\newcommand{\Tmax}{{T_{\max}}}

\newcommand{\Sys}{\textsc{\small{SysId-ARX}}\xspace}
\newcommand{\Gcl}{\mathcal{G}^{cl}}
\newcommand{\Gol}{\mathcal{G}^{ol}}

\newcommand{\LDC}{\textsc{\small{LDC}}\xspace}
\newcommand{\DFC}{\textsc{\small{DFC}}\xspace}

\DeclareMathOperator*{\argmin}{arg\,min}

\newcommand{\reg}{\textsc{\small{Regret}}\xspace}
\newcommand{\OO}{\mathcal{O}}

\DeclareMathOperator{\Tr}{Tr}
\newcommand{\ta}{\tilde{A}}
\newcommand{\tb}{\tilde{B}}
\newcommand{\tc}{\tilde{C}}
\newcommand{\tp}{\mathbf{\tilde{P}}}
\newcommand{\tpb}{\mathbf{\bar{\tilde{P}}}}

\newcommand{\tk}{\tilde{K}}
\newcommand{\tm}{\tilde{M}}
\newcommand{\tf}{\tilde{F}}

\newcommand{\val}{J}

\newcommand{\T}{\mathcal T}

\newcommand{\R}{\mathbb{R}}

\newcommand{\boldR}{\mathbb R}

\newcommand{\wh}{\widehat}
\newcommand{\wb}{\overline}

\newtheorem{assumption}{Assumption}[section]
\newtheorem{theorem*}{Theorem}[section]
\newtheorem*{condition}{Persistence of Excitation of $\Mcontrol \in \Mcontrolset_r$ on ARX Systems}

\newtheorem*{condition1}{Persistence of Excitation of Optimal Control Policy on ARX Systems}
\author{%
 \Name{Sahin Lale}$\mathbf{^1}$ \Email{alale@caltech.edu} \\
 \Name{Kamyar Azizzadenesheli}$\mathbf{^2}$  \Email{kamyar@purdue.edu} \\
 \Name{Babak Hassibi}$\mathbf{^1}$ \Email{hassibi@caltech.edu}\\
 \Name{Anima Anandkumar}$\mathbf{^1}$ \Email{anima@caltech.edu}\\
 \addr $\mathbf{^1}$California Institute of Technology\\
 $\mathbf{^2}$Purdue University
}

\begin{document}

\maketitle

\begin{abstract}%
   Autoregressive exogenous (ARX) systems are the general class of input-output dynamical system used for modeling stochastic linear dynamical system (LDS) including partially observable LDS such as LQG systems. In this work, we study the problem of system identification and adaptive control of unknown ARX systems. We provide finite-time learning guarantees for the ARX systems under both open-loop and closed-loop data collection. Using these guarantees, we design adaptive control algorithms for unknown ARX systems with arbitrary strongly convex or convex quadratic regulating costs. Under strongly convex cost functions, we design an adaptive control algorithm based on online gradient descent to design and update the controllers that are constructed via a convex controller reparametrization. We show that our algorithm has $\Tilde{O}(\sqrt{T})$ regret via explore and commit approach and if the model estimates are updated in epochs using closed-loop data collection, it attains the optimal regret of $\text{polylog}(T)$ after $T$ time-steps of interaction. For the case of convex quadratic cost functions, we propose an adaptive control algorithm that deploys the optimism in the face of uncertainty principle to design the controller. In this setting, we show that the explore and commit approach has a regret upper bound of $\Tilde{O}(T^{2/3})$, and the adaptive control with continuous model estimate updates attains $\Tilde{O}(\sqrt{T})$ regret after $T$ time-steps.  
   
\end{abstract}


% \section{Introduction}

% This is where the content of your paper goes.
% \begin{itemize}
%   \item Use the \texttt{\textbackslash documentclass[anon,12pt]\{alt2021\}} option during submission process -- this automatically hides the author names listed under \texttt{\textbackslash altauthor}. Do not include author names in the remainder of the text, and to the extent possible, avoid directly identifying the authors. You should still include all relevant references, including your own, and any other relevant discussion, even if this might allow a reviewer to infer the author identities. Use \texttt{\textbackslash documentclass[final,12pt]\{alt2021\}} only during camera-ready submission.
% \item The \textsf{jmlr} class automatically loads \textsf{natbib}
% and automatically sets the bibliography style, so you don't need to
% use \verb|\bibliographystyle|.
% This sample file has the citations defined in the accompanying
% BibTeX file \texttt{jmlr-sample.bib}. For a parenthetical
% citation use \verb|\citep|. For example: ``ALT 2020 proceedings
% \citep{ALT2020}". For a textual citation use
% \verb|\citet|. For example: ``The proceedings were edited by \citet{ALT2020}''.
% Both commands may take a comma-separated list.

% These commands have optional arguments and have a starred
% version. See the \textsf{natbib} documentation for further
% details.\footnote{Either \texttt{texdoc natbib} or
% \url{http://www.ctan.org/pkg/natbib}}

% \end{itemize}

% % Acknowledgments---Will not appear in anonymized version
% \acks{We thank a bunch of people.}

\section{Introduction}\label{Intro}

\textbf{Autoregressive Exogenous (ARX) Systems: } ARX systems are central dynamical systems in time-series modelings. They represent stochastic linear dynamical systems (LDS) in the input-output form which have a wide range of applicability to real dynamical systems and amenability for precise analysis. Due to their ability to approximate linear systems in a parametric model structure, ARX systems have been crucial in many areas including chemical engineering, power engineering, medicine, economics, and neuroscience \citep{norquay1998model,bacher2009online,fetics1999parametric,huang2009hybrid,burke2005parametric}. The ARX systems have corresponding linear time-invariant (LTI) state-space representations and in their most general form, they can be represented as follows,
\begin{align}
    {x}_{t+1}={A} {x}_t+B u_t+F y_t, \qquad \qquad
    y_{t}=C x_t+e_t. \label{predictor}
\end{align}
The dynamics are governed by $\Theta \!=\! (A, B, C, F)$ where $x_t$ is the internal state, $y_t$ is the output, $u_t$ is the input and $e_t$ is the measurement noise. Notice that by knowing the initial condition $x_0$ and $\Theta$, one can recover the state sequence. These models provide a \emph{general} representation of LDS with \emph{arbitrary} stochastic disturbances. In particular, via different distributions of $e_t$, they are able to model partially observed LDS (PO-LDS) with various process and measurement noises. For instance, LQG control systems, which are the canonical settings in control, can be modeled as ARX systems. In an LQG control system, the process and measurement noises have Gaussian distributions which corresponds (in predictive form) to an ARX system, where $e_t$ has a particular Gaussian distribution determined by the state-space parameters and noise distributions \citep{kailath2000linear}. 

\noindent \paragraph{System Identification and Adaptive Control:} They are the central problems in control theory and reinforcement learning \citep{lai1982least}. System identification aims to learn the unknown dynamics of the system from the collected data, whereas adaptive control pursues the goal of minimizing the cumulative control cost of dynamical systems with unknown dynamics. Thus, adaptive control inherently includes the system identification process to design a favorable controller. The data collection to achieve these tasks can be performed via independent control inputs yielding open-loop data collection, or via feedback controllers resulting in closed-loop data collection~\citep{ljung1999system}. 

\noindent \paragraph{Finite-time System Identification and Adaptive Control:} In contrast to classical results in both of these problems that analyze the asymptotic performances, recently, there has been a flurry of studies that consider the finite-time performance and learning guarantees in both. In finite-time system identification setting pioneered by \citet{campi2002finite,campi2005guaranteed}, currently, the main focus has been on obtaining the optimal learning rate of $1/\sqrt{T}$ after $T$ samples. Using open-loop data collection to avoid correlations in the inputs and outputs, \citet{oymak2018non,sarkar2019finite,tsiamis2019finite,simchowitz2019learning} suggest methods that achieve this rate for stable LDS. However, due to the difficulty in handling the correlations caused by the feedback controller, the closed-loop system identification guarantees are scarce. Recently, \citet{lale2020logarithmic} propose the first finite-time system identification algorithm that attains the optimal learning rate guarantee for both open and closed-loop data collection.

In finite-time adaptive control, the efforts have been centered around achieving sub-linear regret which measures the difference between the cumulative cost of the adaptive controller and the optimal controller that knows the system dynamics. Most of the prior works follow the explore and commit approach. This approach proposes to first use open-loop data collection to solely explore the system and then estimate the system dynamics and fix a policy to be applied for the remaining time-steps \citep{lale2020regret,mania2019certainty,simchowitz2020improper}. The recent introduction of the first finite-time closed-loop system identification algorithm in \citet{lale2020logarithmic} allowed the design of ``truly'' adaptive control algorithms that naturally use past experiences to improve the model estimates and the controller continuously. Deploying closed-loop data collection, \citet{lale2020root,lale2020logarithmic}
provide adaptive control algorithms for PO-LDS that achieve optimal regret results.

\begin{table}
\centering
% \captionsetup{justification=centering}
\caption{Comparison with prior works for PO-LDS. Our results extend similar regret guarantees to general ARX systems with sub-Gaussian noise disturbances, subsuming the prior works.
E\&C $\coloneqq$ Explore-and-commit approach $\quad$ CLU $\coloneqq$ Closed-loop model estimate updates}
\vspace{-1em}
 \begin{tabular}{c c c c c c } 
 \toprule
 \textbf{Work} & \textbf{Regret} & \textbf{Setting} &  \textbf{Cost} &
  \textbf{Noise} & \textbf{Method} \\ 
 \midrule
 \citet{mania2019certainty} & $\sqrt{T}$ & PO-LDS & Str. Convex & Gaussian & E\&C  \\
% %  \hline
 \citet{simchowitz2020improper} & $\sqrt{T}$ & PO-LDS & Str. Convex & Semi-adversarial & E\&C\\
% %  \hline
 \citet{lale2020logarithmic} & polylog$(T)$ & PO-LDS & Str. Convex & Gaussian & CLU   \\
 \citet{lale2020regret} & $T^{2/3}$ & PO-LDS & Convex  & Gaussian & E\&C   \\ 
%  \hline
 \citet{lale2020root} & $\sqrt{T}$ & PO-LDS & Convex & Gaussian & CLU \\
% %   \hline
\textbf{Theorem \ref{thm:reg_s_cvx_exp_commit} } & $\sqrt{T}$ & ARX & Str. Convex & Sub-Gaussian & E\&C \\
\textbf{Theorem \ref{thm:reg_s_cvx_adapt} } & polylog$(T)$ & ARX & Str. Convex & Sub-Gaussian & CLU \\
\textbf{Theorem \ref{thm:reg_cvx_exp_commit} } & $T^{2/3}$ & ARX & Convex & Sub-Gaussian & E\&C \\
\textbf{Theorem \ref{thm:reg_cvx_adapt} } & $\sqrt{T}$ & ARX & Convex & Sub-Gaussian & CLU \\
\bottomrule
\end{tabular}
\label{table:1}
\vspace{-0.4em}
\end{table}

\noindent \paragraph{Contributions:} In this work, we study finite-time system identification and adaptive control problems in ARX modeled systems with sub-Gaussian noise. First, we state the finite-time guarantees for learning the ARX systems that hold for both open and closed-loop data collection. Deploying the least-squares problem introduced in \citet{lale2020logarithmic}, we show that the estimation error of model parameters decays with $\Tilde{O}(1/\sqrt{T})$ rate after collecting $T$ samples with persistent excitation.

Secondly, we study the adaptive control problem in ARX modeled systems with sub-Gaussian noise. Leveraging the finite-time system identification results, we propose adaptive control frameworks for the ARX systems with arbitrary strongly convex or convex quadratic cost functions:

\begin{enumerate} [wide, labelwidth=!, labelindent=0pt, topsep=0pt,itemsep=-0.5ex,font=\bfseries]
    \item \textbf{ARX systems with strongly convex cost functions:} For this cost function setting, which can possibly be time-varying, we provide an adaptive control algorithm framework that deploys online learning for controller design and exploits the strong convexity. Using online gradient descent with a convex policy reparametrization of linear controllers, we show that adaptive control problem turns into an online convex optimization problem and optimal regret results can be achieved in this setting. To this end, we first show that the explore and commit approach, which fixes the model estimate after open-loop data collection, attains regret of $\Tilde{O}(\sqrt{T})$ after $T$ time-steps of interaction via the proposed framework. Here $\Tilde{O}(\cdot)$ presents the order up to logarithmic terms. We then show that if the model estimates are updated in epochs using the data collected in closed-loop, this adaptive control framework of ARX systems yields the optimal regret rate of $\text{polylog}(T)$. 
    \item \textbf{ARX models with fixed convex quadratic cost function:} For this setting, we propose an adaptive control framework that deploys the principle of optimism in the face of uncertainty (\OFU) \citep{auer2002using} to balance exploration vs. exploitation trade-off in the controller design. The \OFU principle prescribes to use the optimal policy of the model that has the lowest optimal cost, \textit{i.e.} the optimistic model, within the plausible set of systems according to system identification guarantees. We show that using this framework with the explore and commit approach yields regret of $\Tilde{O}(T^{2/3})$. Ultimately, we prove that the adaptive control based on \OFU principle attains regret of $\Tilde{O}(\sqrt{T})$ if the model estimates are continuously updated using closed-loop data in ARX systems.
\end{enumerate}

These results subsume the prior works in PO-LDS and extend them to the general class of ARX systems with sub-Gaussian noise which can be adopted in various real-world time-series modelings (Table \ref{table:1}).

\section{Preliminaries}\label{prelim}
The Euclidean norm of a vector $x$ is denoted as $\|x\|_2$. For a given matrix $A$, $\| A \|_2$ denotes its spectral norm, $\| A\|_F$ is its Frobenius norm, $A^\top$ is its transpose, $A^{\dagger}$ is its Moore-Penrose inverse, and $\Tr(A)$ is the trace. $\rho(A)$ denotes the spectral radius of $A$, \textit{i.e.}, the largest absolute value of its eigenvalues. The j-th singular value of a rank-$n$ matrix $A$ is denoted by $\sigma_j(A)$, where $\sigma_{\max}(A )\!\coloneqq \!\!\sigma_1(A) \!\geq\! \sigma_2(A) \!\geq\! \ldots \!\geq\! \sigma_n(A) \!\coloneqq\! \sigma_{\min}(A) \!>\! 0$. $I$ is the identity matrix with appropriate dimensions. $\mathcal{N}(\mu, \Sigma)$ denotes a multivariate normal distribution with mean vector $\mu$ and covariance matrix $\Sigma$.

Consider the unknown ARX model of $\Theta$ given in (\ref{predictor}).
% ,
% \begin{align}
%     {x}_{t+1}&={A} {x}_t+B u_t+F y_t, \nonumber \\
%     y_{t}&=C x_t+e_t, \label{state-space}
% \end{align}
% where $x_t \in \mathbb{R}^{n}$ is the state of the system, $u_t \in \mathbb{R}^{p}$ is the control input, the observation $y_t \in \mathbb{R}^{m}$ is the output of the system and $e_t \in \mathbb{R}^{m}$ is the measurement noise. 
At each time-step $t$, the system is at state $x_t$ and the agent observes $y_t$. Then, the agent applies a control input $u_t$, observes the loss function $\ell_t$, pays the cost of $c_t=\ell_t(y_{t},u_{t})$, and the system evolves to a new $x_{t+1}$ at time step $t+1$.
% The following is the assumption on the noise characteristics.
\begin{assumption}[Sub-Gaussian Noise] \label{general_noise}
There exists a filtration $\left(\mathcal{F}_{t}\right)$ such that for all $t\geq 0$, and $j\in[0,\ldots, m]$, $e_{t,j}${}s are $R^2$-sub-Gaussian, i.e., for any $\gamma \in \mathbb{R}$, $\mathbb{E}\left[\exp \left(\gamma e_{t, j}\right) | \mathcal{F}_{t-1}\right] \leq \exp \left(\gamma^{2} R^2 / 2\right)$ and $\mathbb{E}\left[e_{t} e_{t}^{\top} | \mathcal{F}_{t-1}\right]= \Sigma_E \succ \sigma_e^2 I$ for some $\sigma_e^2 > 0$.
\end{assumption}
Following general construction of ARX models we assume that $A$ is stable such that $\Phi(A) = \sup _{\tau \geq 0} \left\|A^{\tau}\right\|/\rho(A)^{\tau}$ is finite. This is a mild assumption and captures extensive number of systems including detectable partially observable linear dynamical systems \citep{kailath2000linear}.

\section{System Identification}\label{estimation}
Using the dynamics in (\ref{predictor}), for any positive integer $h$, the output of the system can be written as
\begin{equation}
    y_t = \sum\nolimits_{k=0}^{h-1} CA^k \left(Bu_{t-k-1}  \!+\! Fy_{t-k-1} \right) + e_t + CA^{h}x_{t-h}. \label{markov_rollout}
\end{equation}
The behavior of an ARX system is uniquely governed by its Markov parameters. 

\begin{definition} [Markov Parameters] \label{def:markov}
The set of matrices that maps the previous inputs to the output is called input-to-output Markov parameters and the ones that map the previous outputs to the output are denoted as output-to-output Markov parameters of the system $\Theta$. In particular, the matrices that map inputs and outputs to the output in (\ref{markov_rollout}) are the first $h$ parameters of the Markov operator, $\Markov \!=\!\lbrace G_{u\rightarrow y}^{i },G_{y\rightarrow y}^{i} \rbrace_{i\geq 1}$ where $\forall i\!\geq\!1$,  $G_{u\rightarrow y}^{i }\!=\!CA^{i-1}B$ and $G_{y\rightarrow y}^{i}\!=\!CA^{i-1}F $ which are unique.
\end{definition}

\noindent Let $\Markov_{\mathbf{u\rightarrow y}}(h) \!=\![G_{u\rightarrow y}^{1} ~\! G_{u\rightarrow y}^{2}\! ~\!\ldots\!~ G_{u\rightarrow y}^{h}] \!\in\! \R^{m\! \times \! hp}$ and $\Markov_{\mathbf{y\rightarrow y}}(h) \!=\![G_{y\rightarrow y}^{1} ~\! G_{y\rightarrow y}^{2}\! ~\!\ldots\!~ G_{y\rightarrow y}^{h}] \!\in\! \R^{m\! \times \! hm}$ denote the $h$-length Markov parameters matrices. Consider the following $h$-length operator $\mathcal{G}$ and the subsequences of $h$ input-output pairs from the data collected, either open or closed-loop or both,
\begin{equation}
     \mathcal{G} = [\Markov_{\mathbf{u\rightarrow y}}(h) ~~ \Markov_{\mathbf{y\rightarrow y}}(h) ] \in \mathbb{R}^{m \times h(m+p)}, \quad \phi_i \!=\! \![ u_{i-1}^\top \ldots u_{i-h}^\top ~~ y_{i-1}^\top \ldots y_{i-h}^\top ]^\top \in \mathbb{R}^{h(m+p)} 
\end{equation}
for $ h\leq i \leq t$. Using $\mathcal{G}$, at each time step $t$, the output of the system can be written as 
\begin{equation}
	y_t = \mathcal{G} \phi_t + e_t + CA^h x_{t-h}. \label{output}
\end{equation}
Since $A$ is stable, for $h = c_h \log(T)$, for some problem dependent constant $c_h$ and total execution duration of $T$, the last term in (\ref{output}) provides a negligible bias term of $1/T^2$. Therefore, we solve the following regularized least squares problem to estimate the Markov parameters of the system:
\begin{equation}\label{new_lse}
\wh{\mathcal{G}}_t = \argmin_{\mathcal{G}} \lambda \|X\|_F^2 +\! \sum\nolimits_{i=h}^t \|y_i - \mathcal{G} \phi_i\|^2_2.
\end{equation}
The problem in (\ref{new_lse}) is first introduced in \citet{lale2020logarithmic} to recover LQG systems in predictor form, which is a special case of ARX systems with sub-Gaussian noise. The following learning guarantee for (\ref{new_lse}) follows from Theorem 3 of \citet{lale2020logarithmic}, which is presented for i.i.d. Gaussian innovation terms yet holds for sub-Gaussian measurement disturbances of ARX systems.
\begin{theorem}[Learning Markov Parameters of ARX Systems]
	\label{theo:closedloopid}
	Let $\wh{\mathcal{G}}_{t}$ be the solution to (\ref{new_lse}) at time $t$. For the given choice of $h$, define $V_t = \lambda I + \sum_{i=h}^{t} \phi_i \phi_i^\top$.
	Let $\|\mathcal{G}\|_F \leq S$. For $\delta \in (0,1)$, with probability at least $1-\delta$, for all $t\leq T$, $\mathcal{G}$ lies in the set $\mathcal{C}_{\mathcal{G}}(t)$, where 
	\begin{equation*}
	\mathcal{C}_{\mathcal{G}}(t) = \{ \mathcal{G}': \Tr((\wh{\mathcal{G}}_{t} - \mathcal{G}')V_t(\wh{\mathcal{G}}_{t}-\mathcal{G}')^{\top}) \leq \beta_t \},
	\end{equation*}
	for $\beta_t = \big(\sqrt{m R \log (\delta^{-1} \operatorname{det}(V_t )^{1 / 2} \operatorname{det}(\lambda I)^{-1 / 2})} + S\sqrt{\lambda} + t \sqrt{h} / T^2 \big)^2$. Furthermore, for persistently exciting inputs, \textit{i.e.,} $\sigma_{\min}(V_t) \geq \sigma_\star^2 t $ for some $\sigma_\star \!>\! 0$, and bounded $\phi_i$, with high probability, the least square estimate $\wh{\mathcal{G}}_{t}$ obeys $\|\wh{\mathcal{G}}_{t} - \mathcal{G} \|_F = \Tilde{\OO}(1/\sqrt{t}) $
% 	\begin{equation}\label{fro_bound}
% \|\wh{\mathcal{G}}_{t} - \mathcal{G} \|_F \leq \sigma_\star^{-1} t^{-1/2} \sqrt{m R \left( \log(1/\delta) + \frac{h(m+p)}{2} \log  \left(\frac{\lambda(m+p) + t \Upsilon^2}{\lambda(m+p) }\right)\right)} + S\sqrt{\lambda} +\frac{\sqrt{h}}{T}
% \end{equation}
\end{theorem}
% we get 
% $\sqrt{\sigma_{\min}(V_t)}\|\wh{\mathcal{G}}_{t} - \mathcal{G} \|_F \leq \sqrt{m R \left( \log(1/\delta) + \log(\operatorname{det}\left(V_t \right)^{1 / 2}  \operatorname{det}(\lambda I)^{-1 / 2} ) \right)} + S\sqrt{\lambda} +\frac{t \sqrt{h}}{T^2} $. 
This result shows that under persistent of excitation, the least squares problem (\ref{new_lse}) provides consistent estimates and the estimation error decays with the optimal rate. Note that both input-to-output and output-to-output Markov parameters of ARX system are submatrices of $\mathcal{G}$. Therefore, the given bound trivially holds for $\|\Markov_{\mathbf{u\rightarrow y}}(h) - \wh\Markov_{\mathbf{u\rightarrow y}}(h) \|_F$ and $\|\Markov_{\mathbf{y\rightarrow y}}(h) - \wh\Markov_{\mathbf{y\rightarrow y}}(h) \|_F$.

\section{Adaptive Control of ARX Systems with Strongly Convex Cost }

In this section, we will first introduce linear dynamic controllers (LDC) and provide a convex policy reparametrization, disturbance feedback controllers (\DFC) \citep{simchowitz2020improper,lale2020logarithmic}, to approximate LDC controllers. We then provide the details of the setting of ARX systems regarding the loss and regret definition. Finally, we consider two variants of an algorithm that uses \DFC policies in adaptive control of ARX system and provide the regret performances. 

\noindent \textbf{Linear Dynamic Controllers (\LDC):}
An \LDC, $\pi$, is a linear controller with internal state dynamics $s^\pi_{t+1} = A_\pi s_t^\pi + B_\pi y_t$ and $u^\pi_{t} = C_\pi s_t^\pi + D_\pi y_t$ where $s^\pi_t\in \boldR^s$ is the state of the controller, $y_t$ is the input to the controller, \textit{i.e.} the observation from the system, and $u^\pi_t$ is the output of the controller. $(A_\pi, B_\pi, C_\pi, D_\pi)$ control the internal dynamics of the \LDC. \LDC include a large number of controllers including $H_2$ and $H_\infty$ controllers of fully and partially observable LDS \citep{hassibi1999indefinite}. The optimal control law for ARX models with quadratic cost is also an \LDC (Section \ref{sec:convex_quadratic}).

\noindent \textbf{Output uncertainties 
$\nat_t(\mathcal{G})$:} The output can be decomposed to its components via $\Markov$ as follows,
\begin{align*} 
    y_t &= \sum\nolimits_{k=0}^{t-1} G_{u\rightarrow y}^{k+1} u_{t-k-1} + G_{y\rightarrow y}^{k+1} y_{t-k-1}+ CA^t x_0 + e_t .
\end{align*}

The output uncertainties of ARX system at time $t$ is denoted as follows:

\begin{equation} 
   \nat_t(\mathcal{G}) = y_t - \left( \sum\nolimits_{k=0}^{t-1} G_{u\rightarrow y}^{k+1} u_{t-k-1} + G_{y\rightarrow y}^{k+1} y_{t-k-1} \right)= CA^t x_0 + e_t.
\end{equation}

This definition is similar to Nature's output adopted in \citet{simchowitz2020improper,lale2020logarithmic}. It represents the only unknown components on the output. Notice that, one can identify the uncertainty in the output at any time step uniquely using the history of inputs, outputs and the Markov parameters. This gives the ability of counterfactual reasoning, \textit{i.e.}, consider what the output would have been, if the agent had taken different sequence of inputs and observed different outputs. 

\subsection{Adaptive Control Setting}
\noindent \textbf{Disturbance Response Controllers (\DFC):} For adaptive control of ARX systems with strongly convex cost functions, we adopt a convex policy parametrization called \DFC. A \DFC of length $h'$ is defined as a set of parameters, $\Mcontrol(h') \coloneqq \lbrace M^{[i]} \rbrace_{i=0}^{h'-1}$ acting on the last $h'$ output uncertainties, \textit{i.e.}, 
\begin{equation}\label{DFC_policy}
    u^\Mcontrol_t = \sum\nolimits_{i=0}^{h'-1}M^{[i]}\nat_{t-i}(\mathcal{G}).
\end{equation}

This convex policy parameterization follows the classical Youla parameterization \citep{youla1976modern} and used for adaptive control of PO-LDS in \citet{simchowitz2020improper,lale2020logarithmic}. \DFC policies are truncated approximations of LDC policies and for any LDC policy there exists a \DFC policy which provides equivalent performance (see Appendix A)% \ref{apx:Policies})
. 

Define the closed, convex and compact sets of \DFC{}s, $\Mcontrolset$ and $\Mcontrolset_r$, such that the controllers $\Mcontrol(h_0') = \lbrace M^{[i]} \rbrace_{i=0}^{h_0'-1} \in \Mcontrolset$ are bounded and $\Mcontrolset_r$ is an $r$-expansion of $\Mcontrolset$, \textit{i.e.}, for any $\Mcontrol(h_0') \in \Mcontrolset$ we have, $\sum\nolimits_{i\geq 0}^{h_0'-1}\!\|M^{[i]}\| \!\leq\! \kappa_\psi$ and $\Mcontrolset_r = \lbrace \Mcontrol(h')=\Mcontrol(h_0')+ \Delta : \Mcontrol(h_0') \in \Mcontrolset, \sum\nolimits_{i\geq 0}^{h'-1}\|\Delta^{[i]}\| \leq r \kappa_\psi \rbrace$,
where $h'_0 = \lfloor \frac{h'}{2} \rfloor - h$. Therefore, all controllers $\Mcontrol(h') \in \Mcontrolset_r$ are also bounded $\sum\nolimits_{i\geq 0}^{h'-1}\|M^{[i]}\| \leq \kappa_\psi (1+r)$. Throughout the interaction with the system, the agent has access to $\Mcontrolset_r$. 

\noindent \paragraph{Loss function:} The loss function $\ell_t(\cdot,\cdot)$ is strongly convex, smooth, sub-quadratic and Lipschitz with a parameter $L$, such that for all $t$, $0\!\prec\!\strong I\preceq\nabla^2\ell_t(\cdot,\cdot)\preceq\smooth I$ for a finite constant $\smooth$ and for any $\Gamma$ with $\|u\|,\|u'\|, \|y\|,\|y'\| \leq \Gamma$, we have,
\begin{align}
\label{asm:lipschitzloss}|\ell_t(y,u)-\ell_t(y',u') |\leq L\Gamma(\|y-y'\|+\|u-u'\|) \enskip \text{ and } \enskip |\ell_t(y,u)|\leq L\Gamma^2. \end{align}
\noindent \textbf{Regret definition:} Let $\Mcontrol_\star$ be the optimal, in hindsight, \DFC policy in the given set $\Mcontrolset$, \textit{i.e.}, $\Mcontrol_\star \!=\!  \argmin_{\Mcontrol \in \Mcontrolset} \sum\nolimits_{t=1}^{T} \ell_t(y_{t}^\Mcontrol,u_{t}^\Mcontrol)$. For ARX systems with strongly convex loss function, the adaptive control algorithm's performance is evaluated by its regret with respect to $\Mcontrol_\star$ after $T$ steps of interaction and it is denoted as $\reg(T) = \sum\nolimits_{t=1}^T c_t- \ell_t(y^{\Mcontrol_\star},u^{\Mcontrol_\star}).$
% \begin{equation}\label{regret_def}
% \end{equation}

The proposed algorithm for the ARX systems with strongly convex cost is given in Algorithm \ref{algo_strong}. It has two possible approaches depending on the persistence of excitation of given \DFC set $\Mcontrolset_r$: explore and commit approach or adaptive control with closed-loop estimate updates.

\begin{algorithm}[t] 
\caption{Adaptive Control of ARX Systems with Strongly Convex Cost}
  \begin{algorithmic}[1]
 \STATE \textbf{Input:} $\text{ID}$, $T$, $h$, $h'$ $\Tw$, $\tau$, $S>0$, $\delta > 0$, $\eta_t$\\
\STATE \textbf{if} $\text{ID}$ = $\text{Explore \& Commit}$ \textbf{then} Set $T_{\text{warm}} = T_w$, \textbf{else} Set $T_{\text{warm}} = \tau$ \\
 ------ \textsc{\small{Warm-Up}} ------------------------------------------------ \\
\FOR{$t = 0, 1, \ldots, T_{\text{warm}}$}
\STATE Deploy $u_t \!\sim\! \mathcal{N}(0,\sigma_u^2 I)$ and store $\mathcal{D}_{T_{\text{warm}}} \!=\! \lbrace y_t,u_t \rbrace_{t=1}^{T_{\text{warm}}}$ and set $\Mcontrol_t$ as any member of $\Mcontrolset_r$ \\
\ENDFOR 
------ \textsc{\small{Adaptive Control}} ----------------------------------- \\
\FOR{$i = 0, 1, \ldots$}
    \STATE Calculate $\wh{\mathcal{G}}_i$ via (\ref{new_lse}) using $\mathcal{D}_i = \lbrace y_t,u_t \rbrace_{t=1}^{2^i T_{\text{warm}}}$
    \STATE \textbf{if} $\text{ID}$ = $\text{Explore \& Commit}$ \textbf{then} Set $\wh{\mathcal{G}}_i = \wh{\mathcal{G}}_0$ $\rightarrow$ \textsc{\small{In E\&C, only $\wh{\mathcal{G}}_0$ used for control}} \\
    \STATE Compute $\nat_j(\wh{\mathcal{G}}_i) :=y_j - ( \sum\nolimits_{k=0}^{h-1} \wh{G}_{u\rightarrow y}^{k+1} u_{j-k-1} + \wh{G}_{y\rightarrow y}^{k+1} y_{j-k-1} )$, $\forall j \leq t$
    \FOR{$t = 2^i T_{\text{warm}}, \ldots, 2^{i+1}T_{\text{warm}}-1$ }
        \STATE Observe $y_t$, and compute $\nat_t(\wh{\mathcal{G}}_i) :=y_t - ( \sum\nolimits_{k=0}^{h-1} \wh{G}_{u\rightarrow y}^{k+1} u_{t-k-1} + \wh{G}_{y\rightarrow y}^{k+1} y_{t-k-1} )$
        \STATE Commit to $u_t^{\Mcontrol_t} = \sum_{j=0}^{H'-1}M_t^{[j]} \nat_{t-j}(\wh{\mathcal{G}}_i)$, observe $\ell_t$, and pay a cost of $\ell_t(y_t,u_t^{\Mcontrol_t})$
        % \STATE Observe the loss function $\ell_t$
        \STATE Update $\Mcontrol_{t+1}=\proj_{\Mcontrolset_r}\left(\Mcontrol_t-\eta_t\nabla f_t\left(\Mcontrol_t,\wh{\mathcal{G}}_i\right)
        % |_{\Mcontrol=\Mcontrol_{t}}
        \right)$,~ $\mathcal{D}_{t+1}=\mathcal{D}_{t}\cup\lbrace y_t,u_t\rbrace$
    \ENDFOR
\ENDFOR
  \end{algorithmic}
 \label{algo_strong} 
\end{algorithm}

\subsection{Adaptive Control via Explore and Commit Approach}

In the explore and commit approach, Algorithm \ref{algo_strong} has two phases: an exploration (warm-up) phase with the duration of $\Tw = \OO(\sqrt{T})$ and an exploitation phase for the remaining $T-\Tw$ time-steps. 

\noindent \textbf{Warm-up:}  During the warm-up period, Algorithm \ref{algo_strong} applies $u_t \sim \mathcal{N}(0,\sigma_u^2 I)$ in order to recover the Markov parameters of the system. The duration of warm-up $T_w$ is chosen to guarantee reliable estimate of Markov parameters of ARX system and the stability of \DFC controllers in exploitation phase. The exact duration of warm-up is given in Appendix C%\ref{apx:warm_strong}
. 

\noindent \textbf{Exploitation:} At the end of warm-up, Algorithm \ref{algo_strong} estimates the Markov parameters of ARX system, $\mathcal{G}$, using the data gathered in warm-up. It deploys the regularized least-squares estimation of (\ref{new_lse}) to obtain $\wh{\mathcal{G}}$. At each time-step $t$, Algorithm \ref{algo_strong} uses this estimate and the past inputs to approximate the output uncertainties, $\nat_t(\wh{\mathcal{G}}) = y_t - \sum\nolimits_{k=0}^{h-1} \wh{G}_{u\rightarrow y}^{k+1} u_{t-k-1} + \wh{G}_{y\rightarrow y}^{k+1} y_{t-k-1}$. These approximate output uncertainties are then used to execute a \DFC policy $\Mcontrol_t \in \Mcontrolset_r$ as given in (\ref{DFC_policy}). Upon applying the control input, the algorithm observes the output of the system along with the loss function $\ell_t(\cdot,\cdot)$ and pays the cost of $c_t = \ell_t(y_t, u_t^{\Mcontrol_t})$. At each time-step, Algorithm \ref{algo_strong} employs the counterfactual reasoning introduced in \citet{simchowitz2020improper} to compute a counterfactual loss. Briefly, it considers what the loss would be if the current DFC policy has been applied from the beginning. This provides a noisy metric to evaluate the performance of the current \DFC policy. The details of the counterfactual reasoning are in Appendix E%\ref{apx:strong_convex_control}
. Finally, Algorithm \ref{algo_strong} deploys projected online gradient descent on the counterfactual loss to update and keep the \DFC policy within the given set $\Mcontrolset_r$ for the next time-step. This process is repeated for the remaining $T-T_w$ time-steps. 

Note that deploying \DFC policies turns adaptive control problem into an online convex optimization problem which is computationally and statistically efficient. Moreover, using online gradient descent for controller updates exploits the strong convexity grants the following regret rate. 
\begin{theorem}\label{thm:reg_s_cvx_exp_commit}
Given $\Mcontrolset_r$, a closed, compact and convex set of \DFC policies, Algorithm \ref{algo_strong} with explore and commit approach attains $\reg(T)=\Tilde{\OO}(\sqrt{T}$) with high probability.
\end{theorem}
The proof is in Appendix E%\ref{apx:strong_convex_control}
. In the proof, we first show that the choice of $\Tw$ guarantees that the open-loop data is persistently exciting and the Markov parameter estimates are refined. Then, we show that the estimates of the output uncertainties, the \DFC policy inputs and the outputs of the ARX system are bounded. Following the regret decomposition of Theorem 5 of \citet{simchowitz2020improper}, we show that with the choice of $\Tw$, the regret of running gradient descent on strongly convex losses scales quadratically with the Markov parameters estimation error. This roughly gives $\reg(T) \!=\! \Tilde{\OO} \left(\Tw \!+\! (T\!-\!\Tw)/(\sqrt{\Tw})^2 \right)$ which is minimized by $\Tw \!=\! \OO(\sqrt{T})$, giving the advertised bound.

\subsection{Adaptive Control with Closed-Loop Model Estimate Updates}

Prior to describing Algorithm \ref{algo_strong} with closed-loop model estimate updates, we need a further condition on the sets $\Mcontrolset$ and $\Mcontrolset_r$, such that the \DFC policies in these sets persistently excite the underlying ARX system. The exact definition of the persistence of excitation is given in Appendix B%\ref{apx:persistence}
. Note that this condition is mild and briefly implies having a full row rank condition on a significantly wide matrix that maps past $e_t$ to inputs and outputs. One can also show that if a controller satisfies this, then there exists a neighborhood around it that consists of persistently exciting controllers. In the adaptive control with closed-loop model estimates approach, Algorithm \ref{algo_strong} also has two phases: a fixed length warm-up phase and an adaptive control phase in epochs.

\noindent \textbf{Warm-up:} Algorithm \ref{algo_strong} applies $u_t\! \sim\! \mathcal{N}(0,\sigma_u^2 I)$ for a fixed duration of $\tau$ that solely depends on the underlying system. This phase guarantees the access to a refined first estimate of the system, the persistence of excitation and the stability of the controllers during adaptive control.  

\noindent \textbf{Adaptive control in epochs:} After warm-up, Algorithm \ref{algo_strong} starts controlling the system and operates in epochs with doubling length, \textit{i.e.}, the $i$'th epoch is of duration $2^{i-1}\tau$ for $i\!\geq\!1$. Unlike the explore and commit approach, at the beginning of each epoch, it uses all the data gathered so far to estimate the Markov parameters via (\ref{new_lse}). It then uses this estimate throughout the epoch to approximate the output uncertainties and  implement the \DFC policies. At each time step, the \DFC policies are updated via projected online gradient descent on the computed counterfactual loss. The main difference from the explore and commit approach is that Algorithm \ref{algo_strong} updates the model estimates during adaptive control which further refines the estimates and improves the controllers. 
\begin{theorem}\label{thm:reg_s_cvx_adapt}
Given $\Mcontrolset_r$ with \DFC{}s that persistently excite the underlying ARX system, Algorithm \ref{algo_strong} with closed-loop model estimate updates attains $\reg(T)=\text{polylog}(T)$, with high probability.
\end{theorem}

The proof is in Appendix E %\ref{apx:strong_convex_control}
and it follows similarly with Theorem \ref{thm:reg_s_cvx_exp_commit}. One major difference that allows to achieve the optimal regret rate is the use of data collected during adaptive control to improve the Markov parameter estimates. This approach roughly gives the following decomposition 
$\reg(T) = \OO (\tau + \text{polylog}(T) \sum_{i=1}^{\log(T)} 2^{i-1} \tau  / (\sqrt{2^{i-1}\tau})^2  )$. Notice that unlike explore and commit approach, the estimation error decays at each epoch gives the advertised logarithmic regret.

\section{Adaptive Control of ARX Systems with Convex Quadratic Cost } \label{sec:convex_quadratic}

In this section, we present the setting of ARX systems with convex quadratic cost and the regret definition that competes against the optimal controller for this setting. Finally, we propose an optimism based adaptive control algorithm with two variants and provide the regret guarantees.

\subsection{Adaptive Control Setting} 
The unknown ARX system belongs to a set $\mathcal{S}$ which consists of systems that are $(A,B)$ and $(A,F)$ controllable and $(A,C)$ observable. The ARX system has quadratic cost on $u_t$ and $y_t$, \textit{i.e.}, $c_t = y_t^\top Q y_t + u_t^\top R u_t$ where $Q \succeq 0$ and $R\succ 0$, hence the cost is convex but not strongly convex. For this ARX system, the minimum average expected cost problem is given as follows
\begin{equation*} 
J_{\star}(\Theta) \!=\!  \lim _{T \rightarrow \infty} \min _{u=[u_{1}, \ldots, u_{T}]} \frac{1}{T} \mathbb{E}\left[\sum\nolimits_{t=1}^{T} y_{t}^{\top} Q y_{t}+u_{t}^{\top} R u_{t}\right].
\end{equation*}
Using the average cost optimality equation, one can derive the optimal control law for this problem (Appendix G)%\ref{apx:bellman})
. The optimal control law of ARX systems, $\pi^*$, is a linear feedback policy,
\begin{equation} \label{arx_optimal}
    u_t^* = K_x^* x_t + K_y^* y_t = - (R + B^\top \mathbf{P} B)^{-1} B^\top \mathbf{P}  \left( A x_t + F y_t \right)
\end{equation}
where $\mathbf{P}$ is the unique positive semidefinite solution to the discrete-time algebraic Riccati equation:
\begin{equation} \label{dare}
    \mathbf{P} = C^\top Q C + (A+FC)^\top \mathbf{P} (A+FC) - (A+FC)^\top \mathbf{P} B (R+B^\top \mathbf{P} B)^{-1} B^\top \mathbf{P} (A+FC).
\end{equation}
Note that $\pi^*$ is an LDC policy with the optimal minimum average expected cost of $J_{\star}(\Theta) = \Tr(\Sigma_E (Q \!+\! F^\top (\mathbf{P} \!-\! \mathbf{P} B (R \!+\! B^\top \mathbf{P} B)^{-1} B^\top \mathbf{P})F))$. We assume that the systems in the set $\mathcal{S}$ are \emph{contractible} such that the optimal controller produces contractive closed-loop system dynamics for the state and the output, \textit{i.e.} $\|A + B K_x^* \| \leq \rho < 1$ and $\|F + B K_y^* \| \leq \upsilon < 1$. Finally, the regret measure in this setting is $\reg(T) = \sum_{t=0}^T (c_t - J_*(\Theta))$. 

\noindent \textbf{Optimism in the face of uncertainty (\OFU) principle:} \OFU principle has been widely adopted in sequential decision making tasks in order to balance exploration and exploitation. It suggests to estimate the model up to confidence interval and proposes to act according to the optimal controller of the model that has the lowest optimal cost within the confidence interval, \textit{i.e.}, the optimistic model. For adaptive control in this setting, we deploy the controllers designed via \OFU principle. 

The proposed algorithm for the ARX systems with convex quadratic cost is given in Algorithm \ref{algo_quad}. It has two variants depending on the persistence of excitation of the optimal controller $\pi^*$: explore and commit approach or adaptive control with closed-loop estimate updates.

\begin{algorithm}[t] 
\caption{Adaptive Control of ARX Systems with Convex Quadratic Cost}
  \begin{algorithmic}[1]
 \STATE \textbf{Input:} $\text{ID}$, $T$, $\Tw$, $\tau$, $h$, $S>0$, $\delta > 0$, $n$, $m$, $p$, $Q$, $R$, $\rho$, $\upsilon$ \\
\STATE \textbf{if} $\text{ID}$ = $\text{Explore \& Commit}$ \textbf{then} Set $T_{\text{warm}} = T_w$, \textbf{else} Set $T_{\text{warm}} = \tau$ \\
 ------ \textsc{\small{Warm-Up}} ------------------------------------------------ \\
\FOR{$t = 0, 1, \ldots, T_{\text{warm}}$}
\STATE Deploy $u_t \!\sim\! \mathcal{N}(0,\sigma_u^2 I)$ and store $\mathcal{D}_0 \!=\! \lbrace y_t,u_t \rbrace_{t=1}^{T_{\text{warm}}}$ \\
\ENDFOR 
------ \textsc{\small{Adaptive Control}} ----------------------------------- \\
\FOR{$i = 0, 1, \ldots$}
    \STATE Calculate $\wh{\mathcal{G}}_i$ via (\ref{new_lse}) using $\mathcal{D}_i = \lbrace y_t,u_t \rbrace_{t=1}^{2^i T_{\text{warm}}}$
    \STATE Deploy \Sys($h,\wh{\mathcal{G}}_i,n$) for $\hat{A}_i, \hat{B}_i, \hat{C}_i, \hat{F}_i$
    \STATE Construct $\mathcal{C}_i \coloneqq \{\mathcal{C}_A(i), \mathcal{C}_B(i), \mathcal{C}_C(i),$ $\mathcal{C}_F(i)\}$ s.t. w.h.p.
    $(A,B,C,F) \!\in\! \mathcal{C}_i$ \\
    \STATE Find a $\ttt_i = (\ta_i,\tb_i,\tc_i,\tilde{F}_i ) \in \mathcal{C}_i \cap \mathcal{S} $ s.t. \quad
    $ J(\ttt_i) \leq \inf_{\Theta' \in \mathcal{C}_i \cap \mathcal{S}} J(\Theta') + T^{-1}$\\
    \STATE \textbf{if} $\text{ID}$ = $\text{Explore \& Commit}$ \textbf{then} Set $\ttt_i = \ttt_0$ $\rightarrow$ \textsc{\small{In E\&C, only $\ttt_0$ used for control}} \\
     \FOR{$t = 2^i T_{\text{warm}}, \ldots, 2^{i+1}T_{\text{warm}}-1$ }
            \STATE Execute the optimal controller for $\ttt_i$
    \ENDFOR
\ENDFOR
  \end{algorithmic}
 \label{algo_quad} 
\end{algorithm}

\subsection{Adaptive Control via Explore and Commit Approach}
Similar to prior setting, in the explore and commit approach, Algorithm \ref{algo_quad} has two phases: an exploration (warm-up) phase with the duration of $\Tw = \OO(T^{2/3})$ and an exploitation phase. 

\noindent \textbf{Warm-up:} Algorithm \ref{algo_quad} uses $u_t \!\sim\! \mathcal{N}(0, \sigma_u^2 I)$ for exploration. The exact $\Tw$ is given in Appendix D %\ref{apx:warm_quad}
and it guarantees reliable estimation of system parameters and the stability of \OFU based controller. 

\noindent \paragraph{Exploitation:} At the end of warm-up, Algorithm \ref{algo_quad} estimates the Markov parameters of ARX system via (\ref{new_lse}) and constructs confidence sets ($\mathcal{C}_A, \mathcal{C}_B, \mathcal{C}_C, \mathcal{C}_F$) for the system parameters up to similarity transform using $\Sys$, a variant of Ho-Kalman realization algorithm~\citep{ho1966effective}. The procedure follows similarly with SYS-ID of \citet{lale2020root} and the details are given in Appendix F%\ref{apx:quad_control}
. Algorithm \ref{algo_quad} then deploys the \OFU principle and chooses the optimistic system parameters, $\ttt$, that lie in the intersection of the confidence sets and $\mathcal{S}$. Finally, Algorithm \ref{algo_quad} constructs the optimal control law for $\ttt$ via (\ref{arx_optimal}) and (\ref{dare}) and executes it for the remaining $T-\Tw$ time-steps. 

\begin{theorem}\label{thm:reg_cvx_exp_commit}
Given an unknown ARX system with convex quadratic cost, Algorithm \ref{algo_quad} with explore and commit approach attains $\reg(T) = \tilde{\OO}(T^{2/3})$, with high probability. 
\end{theorem}
The proof is in Appendix F%\ref{apx:quad_control}
. In the proof, we first show that the choice of $\Tw$ guarantees persistence of excitation in open-loop data and the stability of inputs and outputs. Then, we derive the Bellman optimality equation for ARX systems which we use for decomposing regret via \OFU principle. This roughly gives $\reg(T) \!=\! \Tilde{\OO} \left(\Tw \!+\! (T\!-\!\Tw)/\sqrt{\Tw} \right)$ which is minimized by $\Tw \!=\! \OO(T^{2/3})$.

\subsection{Adaptive Control with Closed-Loop Model Estimate Updates}
Before describing Algorithm \ref{algo_quad} with closed-loop model estimate updates, we need a further condition such that the optimal controller for the underlying ARX system persistently excited the system. This is again a mild condition and briefly implies that a significantly wide matrix which maps the past $e_t$ to inputs and outputs and formed via optimal controller is full row rank. The precise condition is given in Appendix B%\ref{apx:persistence}
. Note that if the system parameter estimates are accurate enough, the controller designed with system parameter estimates persistently excite the ARX system. Similar to strongly convex cost setting, in the adaptive control with closed-loop estimates approach, Algorithm \ref{algo_quad} has two phases: a fixed length warm-up phase and an adaptive control in epochs.

\noindent \textbf{Warm-up:} Algorithm \ref{algo_quad} uses $u_t \sim \mathcal{N}(0, \sigma_u^2 I)$ for a fixed warm-up duration $\tau$ which grants refined estimates of the system parameters, persistence of excitation and stability for adaptive control phase. 

\noindent \textbf{Adaptive control in epochs:} After warm-up, Algorithm \ref{algo_quad} starts adaptive control in doubling length epochs, \textit{i.e.}, $i$'th epoch has the duration of $2^{i-1}\tau$. At the beginning of $i$'th epoch, it estimates the system parameters via (\ref{new_lse}), constructs the confidence sets and deploys \OFU principle to recover an optimistic model, $\ttt_i$. Finally, it executes the optimal control law for $\ttt_i$ until the end of epoch $i$. Thus, the main difference from explore and commit approach is the use of closed-loop data to further refine the model estimates. This improves the regret performance and the proof is in Appendix F%\ref{apx:quad_control}
.  

\begin{theorem}\label{thm:reg_cvx_adapt}
Given an unknown ARX system with convex quadratic cost whose optimal controller persistently excites the system, Algorithm \ref{algo_quad} with closed-loop model estimate updates attains $\reg(T) = \tilde{\OO}(\sqrt{T})$, with high probability.
\end{theorem}

% \section{Controller and Regret }\label{Policyclass}
% \input{NewDraft/4-ControllerRegret}

% \section{\textsc{AdaptOn}} \label{sec:algorithm}
% \input{NewDraft/5-Algorithm}

% \section{Regret Analysis}\label{analysis}
% \input{NewDraft/6-RegretAnalysis.tex}

\section{Related Works}\label{relatedworks}
\noindent \textbf{System Identification:} The classical open or closed-loop system identification methods mostly consider the asymptotic performance of the proposed algorithms or demonstrate positive and negative empirical studies~\citep{verhaegen1994identification,forssell1999closed,van1997closed,ljung1999system}. These works mostly consider LQR or LQG systems in their state-space form. However, \citet{chiuso2005consistency, jansson2003subspace} provide asymptotic studies of closed-loop system identification of LQG systems in predictive form which corresponds to the exact ARX systems formulation of LQG. Moreover, the ARX systems, in particular, have been studied extensively in system identification perspective due to their input-output form~\citep{diversi2010identification,bercu2010usefulness, sanandaji2011compressive,stojanovic2016optimal}. In these works, the authors discuss the role of persistence excitation in consistent asymptotic recovery of ARX system parameters. On the other hand, the finite-time learning guarantees, which is the focus of this work, are not known.   

\noindent \paragraph{Adaptive Control:} The classical works in adaptive control also study the asymptotic performance of the designed controllers~\citep{lai1982least,lai1987asymptotically,fiechter1997pac}. In the ARX systems setting, \citet{prandini2000adaptive,prandini2000self,campi1998adaptive} study the asymptotic convergence to optimal controller of ARX systems using an early interpretation of \OFU principle. The current paper is the finite-time counterpart of these studies and completes an important part of the picture in adaptive control of ARX systems by providing optimal regret guarantees. It also extends the prior efforts in adaptive control of LQR and LQG systems in regret minimization perspective to the general ARX systems setting~\citep{abbasi2011regret,dean2018regret,abeille2018improved,agarwal2019online,agarwal2019logarithmic,cohen2019learning,faradonbeh2018input,faradonbeh2020adaptive,faradonbeh2020optimism, lale2020explore,lale2020logarithmic,lale2020regret,lale2020root, mania2019certainty,simchowitz2020naive,simchowitz2020improper}.

% \section{Conclusion}\label{conclusion}
% \input{Predictive/7-Conclusion}

% \acks{S. Lale is supported in part by DARPA PAI. K. Azizzadenesheli gratefully acknowledge the financial support of Raytheon and Amazon Web Services. B. Hassibi is supported in part by the National Science Foundation under grants CNS-0932428, CCF-1018927, CCF-1423663 and CCF-1409204, by a grant from Qualcomm Inc., by NASA’s Jet Propulsion Laboratory through the President and Director’s Fund, and by King Abdullah University of Science and Technology. A. Anandkumar is supported in part by Bren endowed chair, DARPA PAIHR00111890035 and LwLL grants, Raytheon, Microsoft, Google, and Adobe faculty fellowships.}

\bibliography{main}

\begin{thebibliography}{53}
\providecommand{\natexlab}[1]{#1}
\providecommand{\url}[1]{\texttt{#1}}
\expandafter\ifx\csname urlstyle\endcsname\relax
  \providecommand{\doi}[1]{doi: #1}\else
  \providecommand{\doi}{doi: \begingroup \urlstyle{rm}\Url}\fi

\bibitem[Abbasi-Yadkori and Szepesv{\'a}ri(2011)]{abbasi2011regret}
Yasin Abbasi-Yadkori and Csaba Szepesv{\'a}ri.
\newblock Regret bounds for the adaptive control of linear quadratic systems.
\newblock In \emph{Proceedings of the 24th Annual Conference on Learning
  Theory}, pages 1--26, 2011.

\bibitem[Abbasi-Yadkori et~al.(2011)Abbasi-Yadkori, P{\'a}l, and
  Szepesv{\'a}ri]{abbasi2011improved}
Yasin Abbasi-Yadkori, D{\'a}vid P{\'a}l, and Csaba Szepesv{\'a}ri.
\newblock Improved algorithms for linear stochastic bandits.
\newblock In \emph{Advances in Neural Information Processing Systems}, pages
  2312--2320, 2011.

\bibitem[Abeille and Lazaric(2018)]{abeille2018improved}
Marc Abeille and Alessandro Lazaric.
\newblock Improved regret bounds for thompson sampling in linear quadratic
  control problems.
\newblock In \emph{International Conference on Machine Learning}, pages 1--9,
  2018.

\bibitem[Agarwal et~al.(2019{\natexlab{a}})Agarwal, Bullins, Hazan, Kakade, and
  Singh]{agarwal2019online}
Naman Agarwal, Brian Bullins, Elad Hazan, Sham~M Kakade, and Karan Singh.
\newblock Online control with adversarial disturbances.
\newblock \emph{arXiv preprint arXiv:1902.08721}, 2019{\natexlab{a}}.

\bibitem[Agarwal et~al.(2019{\natexlab{b}})Agarwal, Hazan, and
  Singh]{agarwal2019logarithmic}
Naman Agarwal, Elad Hazan, and Karan Singh.
\newblock Logarithmic regret for online control.
\newblock In \emph{Advances in Neural Information Processing Systems}, pages
  10175--10184, 2019{\natexlab{b}}.

\bibitem[Auer(2002)]{auer2002using}
Peter Auer.
\newblock Using confidence bounds for exploitation-exploration trade-offs.
\newblock \emph{Journal of Machine Learning Research}, 3\penalty0
  (Nov):\penalty0 397--422, 2002.

\bibitem[Bacher et~al.(2009)Bacher, Madsen, and Nielsen]{bacher2009online}
Peder Bacher, Henrik Madsen, and Henrik~Aalborg Nielsen.
\newblock Online short-term solar power forecasting.
\newblock \emph{Solar energy}, 83\penalty0 (10):\penalty0 1772--1783, 2009.

\bibitem[Bercu and Vazquez(2010)]{bercu2010usefulness}
Bernard Bercu and Victor Vazquez.
\newblock On the usefulness of persistent excitation in arx adaptive tracking.
\newblock \emph{International Journal of Control}, 83\penalty0 (6):\penalty0
  1145--1154, 2010.

\bibitem[Bertsekas(1995)]{bertsekas1995dynamic}
Dimitri~P Bertsekas.
\newblock \emph{Dynamic programming and optimal control}, volume~2.
\newblock Athena scientific Belmont, MA, 1995.

\bibitem[Burke et~al.(2005)Burke, Kelly, De~Chazal, Reilly, and
  Finucane]{burke2005parametric}
Dave~P Burke, Simon~P Kelly, Philip De~Chazal, Richard~B Reilly, and Ciar{\'a}n
  Finucane.
\newblock A parametric feature extraction and classification strategy for
  brain-computer interfacing.
\newblock \emph{IEEE Transactions on Neural Systems and Rehabilitation
  Engineering}, 13\penalty0 (1):\penalty0 12--17, 2005.

\bibitem[Campi and Kumar(1998)]{campi1998adaptive}
Marco~C Campi and PR~Kumar.
\newblock Adaptive linear quadratic gaussian control: the cost-biased approach
  revisited.
\newblock \emph{SIAM Journal on Control and Optimization}, 36\penalty0
  (6):\penalty0 1890--1907, 1998.

\bibitem[Campi and Weyer(2002)]{campi2002finite}
Marco~C Campi and Erik Weyer.
\newblock Finite sample properties of system identification methods.
\newblock \emph{IEEE Transactions on Automatic Control}, 47\penalty0
  (8):\penalty0 1329--1334, 2002.

\bibitem[Campi and Weyer(2005)]{campi2005guaranteed}
Marco~C Campi and Erik Weyer.
\newblock Guaranteed non-asymptotic confidence regions in system
  identification.
\newblock \emph{Automatica}, 41\penalty0 (10):\penalty0 1751--1764, 2005.

\bibitem[Chiuso and Picci(2005)]{chiuso2005consistency}
Alessandro Chiuso and Giorgio Picci.
\newblock Consistency analysis of some closed-loop subspace identification
  methods.
\newblock \emph{Automatica}, 41\penalty0 (3):\penalty0 377--391, 2005.

\bibitem[Cohen et~al.(2018)Cohen, Hassidim, Koren, Lazic, Mansour, and
  Talwar]{cohen2018online}
Alon Cohen, Avinatan Hassidim, Tomer Koren, Nevena Lazic, Yishay Mansour, and
  Kunal Talwar.
\newblock Online linear quadratic control.
\newblock \emph{arXiv preprint arXiv:1806.07104}, 2018.

\bibitem[Cohen et~al.(2019)Cohen, Koren, and Mansour]{cohen2019learning}
Alon Cohen, Tomer Koren, and Yishay Mansour.
\newblock Learning linear-quadratic regulators efficiently with only $
  \sqrt{T}$ regret.
\newblock \emph{arXiv preprint arXiv:1902.06223}, 2019.

\bibitem[Dean et~al.(2018)Dean, Mania, Matni, Recht, and Tu]{dean2018regret}
Sarah Dean, Horia Mania, Nikolai Matni, Benjamin Recht, and Stephen Tu.
\newblock Regret bounds for robust adaptive control of the linear quadratic
  regulator.
\newblock In \emph{Advances in Neural Information Processing Systems}, pages
  4188--4197, 2018.

\bibitem[Diversi et~al.(2010)Diversi, Guidorzi, and
  Soverini]{diversi2010identification}
Roberto Diversi, Roberto Guidorzi, and Umberto Soverini.
\newblock Identification of arx and ararx models in the presence of input and
  output noises.
\newblock \emph{European Journal of Control}, 16\penalty0 (3):\penalty0
  242--255, 2010.

\bibitem[Faradonbeh et~al.(2018)Faradonbeh, Tewari, and
  Michailidis]{faradonbeh2018input}
Mohamad Kazem~Shirani Faradonbeh, Ambuj Tewari, and George Michailidis.
\newblock Input perturbations for adaptive regulation and learning.
\newblock \emph{arXiv preprint arXiv:1811.04258}, 2018.

\bibitem[Faradonbeh et~al.(2020{\natexlab{a}})Faradonbeh, Tewari, and
  Michailidis]{faradonbeh2020adaptive}
Mohamad Kazem~Shirani Faradonbeh, Ambuj Tewari, and George Michailidis.
\newblock On adaptive linear--quadratic regulators.
\newblock \emph{Automatica}, 117:\penalty0 108982, 2020{\natexlab{a}}.

\bibitem[Faradonbeh et~al.(2020{\natexlab{b}})Faradonbeh, Tewari, and
  Michailidis]{faradonbeh2020optimism}
Mohamad Kazem~Shirani Faradonbeh, Ambuj Tewari, and George Michailidis.
\newblock Optimism-based adaptive regulation of linear-quadratic systems.
\newblock \emph{IEEE Transactions on Automatic Control}, 2020{\natexlab{b}}.

\bibitem[Fetics et~al.(1999)Fetics, Nevo, Chen, and Kass]{fetics1999parametric}
Barry Fetics, Erez Nevo, Chen-Huan Chen, and David~A Kass.
\newblock Parametric model derivation of transfer function for noninvasive
  estimation of aortic pressure by radial tonometry.
\newblock \emph{IEEE Transactions on Biomedical Engineering}, 46\penalty0
  (6):\penalty0 698--706, 1999.

\bibitem[Fiechter(1997)]{fiechter1997pac}
Claude-Nicolas Fiechter.
\newblock Pac adaptive control of linear systems.
\newblock In \emph{Annual Workshop on Computational Learning Theory:
  Proceedings of the tenth annual conference on Computational learning theory},
  volume~6, pages 72--80. Citeseer, 1997.

\bibitem[Forssell and Ljung(1999)]{forssell1999closed}
Urban Forssell and Lennart Ljung.
\newblock Closed-loop identification revisited.
\newblock \emph{Automatica}, 35\penalty0 (7):\penalty0 1215--1241, 1999.

\bibitem[Hassibi et~al.(1999)Hassibi, Sayed, and
  Kailath]{hassibi1999indefinite}
Babak Hassibi, Ali~H Sayed, and Thomas Kailath.
\newblock \emph{Indefinite-Quadratic Estimation and Control: A Unified Approach
  to H2 and H-infinity Theories}, volume~16.
\newblock SIAM, 1999.

\bibitem[Ho and K{\'a}lm{\'a}n(1966)]{ho1966effective}
BL~Ho and Rudolf~E K{\'a}lm{\'a}n.
\newblock Effective construction of linear state-variable models from
  input/output functions.
\newblock \emph{at-Automatisierungstechnik}, 14\penalty0 (1-12):\penalty0
  545--548, 1966.

\bibitem[Huang and Jane(2009)]{huang2009hybrid}
Kuang~Yu Huang and Chuen-Jiuan Jane.
\newblock A hybrid model for stock market forecasting and portfolio selection
  based on arx, grey system and rs theories.
\newblock \emph{Expert systems with applications}, 36\penalty0 (3):\penalty0
  5387--5392, 2009.

\bibitem[Jaksch et~al.(2010)Jaksch, Ortner, and Auer]{jaksch2010near}
Thomas Jaksch, Ronald Ortner, and Peter Auer.
\newblock Near-optimal regret bounds for reinforcement learning.
\newblock \emph{Journal of Machine Learning Research}, 11\penalty0
  (Apr):\penalty0 1563--1600, 2010.

\bibitem[Jansson(2003)]{jansson2003subspace}
Magnus Jansson.
\newblock Subspace identification and arx modeling.
\newblock \emph{IFAC Proceedings Volumes}, 36\penalty0 (16):\penalty0
  1585--1590, 2003.

\bibitem[Kailath et~al.(2000)Kailath, Sayed, and Hassibi]{kailath2000linear}
Thomas Kailath, Ali~H Sayed, and Babak Hassibi.
\newblock Linear estimation, 2000.

\bibitem[Lai and Wei(1987)]{lai1987asymptotically}
Tze~Leung Lai and Ching-Zong Wei.
\newblock Asymptotically efficient self-tuning regulators.
\newblock \emph{SIAM Journal on Control and Optimization}, 25\penalty0
  (2):\penalty0 466--481, 1987.

\bibitem[Lai et~al.(1982)Lai, Wei, et~al.]{lai1982least}
Tze~Leung Lai, Ching~Zong Wei, et~al.
\newblock Least squares estimates in stochastic regression models with
  applications to identification and control of dynamic systems.
\newblock \emph{The Annals of Statistics}, 10\penalty0 (1):\penalty0 154--166,
  1982.

\bibitem[Lale et~al.(2020{\natexlab{a}})Lale, Azizzadenesheli, Hassibi, and
  Anandkumar]{lale2020explore}
Sahin Lale, Kamyar Azizzadenesheli, Babak Hassibi, and Anima Anandkumar.
\newblock Explore more and improve regret in linear quadratic regulators.
\newblock \emph{arXiv preprint arXiv:2007.12291}, 2020{\natexlab{a}}.

\bibitem[Lale et~al.(2020{\natexlab{b}})Lale, Azizzadenesheli, Hassibi, and
  Anandkumar]{lale2020logarithmic}
Sahin Lale, Kamyar Azizzadenesheli, Babak Hassibi, and Anima Anandkumar.
\newblock Logarithmic regret bound in partially observable linear dynamical
  systems.
\newblock \emph{In Advances in Neural Information Processing Systems, volume
  33}, 2020{\natexlab{b}}.

\bibitem[Lale et~al.(2020{\natexlab{c}})Lale, Azizzadenesheli, Hassibi, and
  Anandkumar]{lale2020regret}
Sahin Lale, Kamyar Azizzadenesheli, Babak Hassibi, and Anima Anandkumar.
\newblock Regret minimization in partially observable linear quadratic control.
\newblock \emph{arXiv preprint arXiv:2002.00082}, 2020{\natexlab{c}}.

\bibitem[Lale et~al.(2021)Lale, Azizzadenesheli, Hassibi, and
  Anandkumar]{lale2020root}
Sahin Lale, Kamyar Azizzadenesheli, Babak Hassibi, and Anima Anandkumar.
\newblock Adaptive control and regret minimization in linear quadratic gaussian
  (lqg) setting.
\newblock In \emph{2021 American Control Conference (ACC)}, pages 2517--2522,
  2021.

\bibitem[Ljung(1999)]{ljung1999system}
Lennart Ljung.
\newblock System identification.
\newblock \emph{Wiley Encyclopedia of Electrical and Electronics Engineering},
  pages 1--19, 1999.

\bibitem[Mania et~al.(2019)Mania, Tu, and Recht]{mania2019certainty}
Horia Mania, Stephen Tu, and Benjamin Recht.
\newblock Certainty equivalent control of lqr is efficient.
\newblock \emph{arXiv preprint arXiv:1902.07826}, 2019.

\bibitem[Norquay et~al.(1998)Norquay, Palazoglu, and
  Romagnoli]{norquay1998model}
Sandra~J Norquay, Ahmet Palazoglu, and Jos{\'e}A Romagnoli.
\newblock Model predictive control based on wiener models.
\newblock \emph{Chemical Engineering Science}, 53\penalty0 (1):\penalty0
  75--84, 1998.

\bibitem[Oymak and Ozay(2018)]{oymak2018non}
Samet Oymak and Necmiye Ozay.
\newblock Non-asymptotic identification of lti systems from a single
  trajectory.
\newblock \emph{arXiv preprint arXiv:1806.05722}, 2018.

\bibitem[Prandini and Campi(2000{\natexlab{a}})]{prandini2000adaptive}
Maria Prandini and Marco~C Campi.
\newblock Adaptive lqg control of input-output systems---a cost-biased
  approach.
\newblock \emph{SIAM Journal on Control and Optimization}, 39\penalty0
  (5):\penalty0 1499--1519, 2000{\natexlab{a}}.

\bibitem[Prandini and Campi(2000{\natexlab{b}})]{prandini2000self}
MARIA Prandini and MC~Campi.
\newblock A self-optimizing adaptive lqg control scheme for input-output
  systems.
\newblock In \emph{Proceedings of the 39th IEEE Conference on Decision and
  Control (Cat. No. 00CH37187)}, volume~2, pages 1110--1115. IEEE,
  2000{\natexlab{b}}.

\bibitem[Sanandaji et~al.(2011)Sanandaji, Vincent, Wakin, T{\'o}th, and
  Poolla]{sanandaji2011compressive}
Borhan~M Sanandaji, Tyrone~L Vincent, Michael~B Wakin, Roland T{\'o}th, and
  Kameshwar Poolla.
\newblock Compressive system identification of lti and ltv arx models.
\newblock In \emph{2011 50th IEEE Conference on Decision and Control and
  European Control Conference}, pages 791--798. IEEE, 2011.

\bibitem[Sarkar et~al.(2019)Sarkar, Rakhlin, and Dahleh]{sarkar2019finite}
Tuhin Sarkar, Alexander Rakhlin, and Munther~A Dahleh.
\newblock Finite-time system identification for partially observed lti systems
  of unknown order.
\newblock \emph{arXiv preprint arXiv:1902.01848}, 2019.

\bibitem[Simchowitz and Foster(2020)]{simchowitz2020naive}
Max Simchowitz and Dylan~J Foster.
\newblock Naive exploration is optimal for online lqr.
\newblock \emph{arXiv preprint arXiv:2001.09576}, 2020.

\bibitem[Simchowitz et~al.(2019)Simchowitz, Boczar, and
  Recht]{simchowitz2019learning}
Max Simchowitz, Ross Boczar, and Benjamin Recht.
\newblock Learning linear dynamical systems with semi-parametric least squares.
\newblock \emph{arXiv preprint arXiv:1902.00768}, 2019.

\bibitem[Simchowitz et~al.(2020)Simchowitz, Singh, and
  Hazan]{simchowitz2020improper}
Max Simchowitz, Karan Singh, and Elad Hazan.
\newblock Improper learning for non-stochastic control.
\newblock \emph{arXiv preprint arXiv:2001.09254}, 2020.

\bibitem[Stojanovic et~al.(2016)Stojanovic, Nedic, Prsic, and
  Dubonjic]{stojanovic2016optimal}
Vladimir Stojanovic, Novak Nedic, Dragan Prsic, and Ljubisa Dubonjic.
\newblock Optimal experiment design for identification of arx models with
  constrained output in non-gaussian noise.
\newblock \emph{Applied Mathematical Modelling}, 40\penalty0 (13-14):\penalty0
  6676--6689, 2016.

\bibitem[Tropp(2012)]{tropp2012user}
Joel~A Tropp.
\newblock User-friendly tail bounds for sums of random matrices.
\newblock \emph{Foundations of computational mathematics}, 12\penalty0
  (4):\penalty0 389--434, 2012.

\bibitem[Tsiamis and Pappas(2019)]{tsiamis2019finite}
Anastasios Tsiamis and George~J Pappas.
\newblock Finite sample analysis of stochastic system identification.
\newblock \emph{arXiv preprint arXiv:1903.09122}, 2019.

\bibitem[Van~Overschee and De~Moor(1997)]{van1997closed}
Peter Van~Overschee and Bart De~Moor.
\newblock Closed loop subspace system identification.
\newblock In \emph{Proceedings of the 36th IEEE Conference on Decision and
  Control}, volume~2, pages 1848--1853. IEEE, 1997.

\bibitem[Verhaegen(1994)]{verhaegen1994identification}
Michel Verhaegen.
\newblock Identification of the deterministic part of mimo state space models
  given in innovations form from input-output data.
\newblock \emph{Automatica}, 30\penalty0 (1):\penalty0 61--74, 1994.

\bibitem[Youla et~al.(1976)Youla, Jabr, and Bongiorno]{youla1976modern}
Dante Youla, Hamid Jabr, and Jr~Bongiorno.
\newblock Modern wiener-hopf design of optimal controllers--part ii: The
  multivariable case.
\newblock \emph{IEEE Transactions on Automatic Control}, 21\penalty0
  (3):\penalty0 319--338, 1976.

\end{thebibliography}

\appendix
\newpage
\begin{center}
{\huge Appendix}
\end{center}

In Appendix \ref{apx:Policies}, after introducing some technical properties that would be used for proofs in regret guarantees of Algorithm \ref{algo_strong}, we show that the performance of \LDC policies can be well-approximated by \DFC policies. We provide the precise definition of persistence excitation for both warm-up and adaptive control periods in Appendix \ref{apx:persistence}. In Appendix \ref{apx:warm_strong} and \ref{apx:warm_quad}, we give precise warm-up durations for Algorithm \ref{algo_strong} and \ref{algo_quad} respectively. The technical details of Algorithm \ref{algo_strong} as well as the proofs of Theorems \ref{thm:reg_s_cvx_exp_commit} and \ref{thm:reg_s_cvx_adapt} are given in Appendix \ref{apx:strong_convex_control}. The details of Algorithm \ref{algo_quad} and the proofs of Theorems \ref{thm:reg_cvx_exp_commit} and \ref{thm:reg_cvx_adapt} are given in Appendix \ref{apx:quad_control} where the proofs built on the Bellman optimality equation for ARX systems provided in Appendix \ref{apx:bellman}.

\section{\LDC Policies and \DFC Policies} \label{apx:Policies}
Recall that LDC policies have the following construction:
\begin{equation}\label{eq:LDC}
    s^\pi_{t+1} = A_\pi s_t^\pi + B_\pi y_t^\pi, \qquad u^\pi_{t} = C_\pi s_t^\pi + D_\pi y_t^\pi.
\end{equation}
Therefore, using the ARX system (\ref{predictor}), we get 
\begin{align}\label{eq:PolicyDynamics}
\begin{bmatrix}
x^\pi_{t+\!1}\\
s^\pi_{t+\!1}
\end{bmatrix}
&\!\!=\!\!
\underbrace{\begin{aligned}
\begin{bmatrix}
A + FC + BD_\pi C & BC_{\pi}\\
B_{\pi}C & A_{\pi}
\end{bmatrix}\end{aligned}}_\text{$A'_{\pi}$}\!
\begin{bmatrix}
x^\pi_{t}\\
s^\pi_{t}
\end{bmatrix}
+\underbrace{\begin{aligned}\begin{bmatrix}
F+BD_{\pi}\\
B_{\pi}
\end{bmatrix}\end{aligned}}_\text{$B'_{\pi}$} \!
\begin{bmatrix}
e_{t}
\end{bmatrix},
\begin{bmatrix}
y^\pi_{t}\\
u^\pi_{t}
\end{bmatrix}=
\underbrace{\begin{aligned}\begin{bmatrix}
C & 0_{s\times d}\\
D_{\pi}C &C_{\pi}
\end{bmatrix}\end{aligned}}_\text{$C'_{\pi}$}\!
\begin{bmatrix}
x^\pi_{t}\\
s^\pi_{t}
\end{bmatrix}
+\underbrace{\begin{aligned}\begin{bmatrix}
I_{d}\\
D_{\pi}
\end{bmatrix}\end{aligned}}_\text{$D'_{\pi}$}\!
\begin{bmatrix}
e_{t}
\end{bmatrix}
\end{align}

where $(A_\pi',B_\pi',C_\pi',D_\pi')$ define the induced closed-loop system. The Markov operator for the system $(A_\pi',B_\pi',C_\pi',D_\pi')$ can be defined as $\mathbf{G_\pi'}=\lbrace {G'_\pi}^{i}\rbrace_{i=0}$, where ${G'_\pi}^{0}=D_\pi'$, and $\forall i>0$,  ${G'_\pi}^{i}=C_\pi'A_\pi'^{i-1}B_\pi'$. 

\begin{definition}[Proper Decay Function] $\psi : \mathbb{N} \rightarrow \mathbb{R}_{\geq 0}$ is a proper decay function if $\psi$ is non-increasing and $\lim_{h'\rightarrow \infty}\psi(h')=0$. For a Markov operator $\Markov$, $\psi_\Markov(h)$ defines the induced decay function on $\Markov$, \textit{i.e.}, $\psi_\Markov(h) \coloneqq \sum_{i\geq h}\|{G}^{i}\|$.
\end{definition}
This decay represents the effect of past system inputs on the system output. For stable (open or closed-loop) systems, the Markov operator can be bounded trivially. This brings the following policy class to consider for ARX systems. 

\begin{definition}[\LDC policies with proper decay function] $\Pi(\psi)$ denotes the class of \LDC policies associated with a proper decay function $\psi$, such that for all $\pi\in\Pi(\psi)$, and all $h\geq 0$, $\sum_{i\geq h}\|{G'_\pi}^{i}\|\leq \psi(h)$. 
\end{definition}
In order to provide clean analysis, for the policy class $\Pi(\psi)$, let $\kappa_\psi \coloneqq \psi(0)$ such that $\sum_{i\geq 0}\|{G'_\pi}^{i}\|\leq \kappa_\psi$. This class corresponds to stabilizing LDC policies. Moreover, for the open-loop system, let $\kappa_\Markov \geq \sum_{i\geq 0}\|{G}^{i}\|$ for $\kappa_\Markov \! \geq\! 1$ where $G^{i} = C (A+FC)^i B$. This follows from the assumption of $A+FC$ is stable.  

Thus, the output of the LDC policy $u_t^{\pi}$ has the following expanded  
\begin{align*}
u_t^{\pi}  &= 
\begin{bmatrix}
D_\pi C & C_\pi
\end{bmatrix}
\begin{bmatrix}
x_t^\pi \\ s_t^\pi
\end{bmatrix}
+
 D_\pi
e_{t}\\
%%%%%%%%%
&=\begin{bmatrix}
D_\pi C & C_\pi
\end{bmatrix}
\left(
A_\pi'
\begin{bmatrix}
x^\pi_{t-1}\\
s^\pi_{t-1}
\end{bmatrix}
+B_\pi'
e_{t-1}\right)
+
 D_\pi
e_{t}\\
%%%%%%%%%
&=
\sum_{k=0}^{t-1}\begin{bmatrix}
D_\pi C & C_\pi
\end{bmatrix}
{A_\pi'}^kB_\pi'
e_{t-k-1}
+
D_\pi
e_{t} 
\end{align*}

Therefore, for $u^\Mcontrol_t = \sum_{i=0}^{h'-1}M^{[k]}\nat_{t-i}(\mathcal{G})$, by setting $M^{[0]}=D_\pi$, and $M^{[k]}=\begin{bmatrix}
D_\pi C & C_\pi
\end{bmatrix}
{A_\pi'}^kB_\pi'$, we have 

\begin{align*}
    u_t^{\pi}-u^\Mcontrol_t = \sum_{k=h'}^{t-1}\begin{bmatrix}
D_\pi C & C_\pi
\end{bmatrix}
{A_\pi'}^kB_\pi'
e_{t-k-1}
\end{align*}

Let $\|\nat_{t}(\mathcal{G}) \| \leq \kappa_\nature $ for all $t$. Using Cauchy Schwarz inequality we have
\begin{align*}
    \|u^\pi_t - u^\Mcontrol_t\| \leq \left \|\sum_{k=h'}^{t}\begin{bmatrix}
D_\pi C & C_\pi
\end{bmatrix}
{A_\pi'}^kB_\pi'\nat_{t-i} \right \|\leq \psi(h') \kappa_\nature. 
\end{align*}

Using the definition of $y_t^\pi$ and $y^\Mcontrol_t$, we have
\begin{align*}
    y_t^\pi &= C \sum\nolimits_{k=0}^{t-1} (A+FC)^kBu^\pi_{t-k-1} + C \sum\nolimits_{k=0}^{t-1} (A+FC)^kFe_{t-k-1} + e_t \\
    y^\Mcontrol_t &= C \sum\nolimits_{k=0}^{t-1} (A+FC)^kBu^M_{t-k-1} + C \sum\nolimits_{k=0}^{t-1} (A+FC)^kFe_{t-k-1} + e_t.
\end{align*}
Subtracting these two equations, we get,
\begin{align*}
    y^\pi_t - y^\Mcontrol_t = C \sum\nolimits_{k=0}^{t-1} (A+FC)^kBu^\pi_{t-k-1}-C \sum\nolimits_{k=0}^{t-1} (A+FC)^kB u^M_{t-k-1}
\end{align*}
resulting in $\|y^\pi_t - y^\Mcontrol_t\| \leq \psi(h') \kappa_\Markov\kappa_\nature $. These show that for any LDC policy, the DFC approximation of it provides reasonable performance. Therefore, one can deduce that any stabilizing \LDC policy can be well approximated by a \DFC that belongs to the following set of \DFC{}s,
\begin{equation*}
    \Mcontrolset_r = \Big \{\Mcontrol(h'):=\lbrace M^{[i]} \rbrace_{i=0}^{h'-1} : \sum\nolimits_{i\geq 0}^{h'-1}\|M^{[i]}\|\leq \kappa_\psi(1+r) \Big\},
\end{equation*}
indicating that using the class of \DFC policies as an approximation to \LDC policies is justified. 

\section{Persistence of Excitation}
\label{apx:persistence}

In this section, the precise persistence of excitation conditions of the inputs are provided. First, open-loop persistence excitation is considered following similar analysis of \cite{lale2020regret}, Appendix \ref{subsec:warmuppersist}. Then, the persistence of excitation in adaptive control is analyzed. We assume that, throughout the interaction with the system, the agent has access to a convex compact set of \DFC{}s, $\Mcontrolset_r$ which is an $r$-expansion of $\Mcontrolset$, such that $\kappa_\Mcontrolset = \kappa_\psi (1 + r)$ and all controllers $\Mcontrol \in \Mcontrolset_r $ are persistently exciting the ARX system. The persistence of excitation condition for the given set $\Mcontrolset_r$ is formally defined in Appendix \ref{subsec:persistence} and in Appendix \ref{subsec:theo:persist_adaptive}, we show that persistence of excitation is achieved by the policies that Algorithm \ref{algo_strong} and Algorithm \ref{algo_quad} deploy. In the following, $\bar{\phi}_t = S \phi_t$ for a permutation matrix $S$ that gives 
\begin{equation*}
\bar{\phi}_t = \left[ y_{t-1}^\top \enskip u_{t-1}^\top \ldots y_{t-h}^\top \enskip  u_{t-h}^\top \right]^\top \in \mathbb{R}^{(m+p)h}.
\end{equation*}

\subsection{Persistence of Excitation in Warm-up for Algorithm \ref{algo_strong} \& \ref{algo_quad}} \label{subsec:warmuppersist}

The following guarantee holds for both Algorithm 1 and 2, since their warm-ups have the same sub-routine. Recall the state-space form of the ARX system in (\ref{predictor}),
\begin{align*}
    x_{t+1} &= A x_t + B u_t + Fy_t \nonumber \\
    y_t &= C x_t + e_t.
\end{align*}
During the warm-up period, $t \leq T_{\text{warm}}$, the input to the system is $u_t \sim \mathcal{N}(0,\sigma_u^2 I)$. Let $f_t = [y_t^\top u_t^\top]^\top$. From the evolution of the system with given input we have the following:
\begin{equation*}
    f_t = \mathbf{G^o} \begin{bmatrix}
    e_t^\top & u_t^\top & e_{t-1}^\top & u_{t-1}^\top & \ldots & e_{t-h+1}^\top & u_{t-h+1}^\top
\end{bmatrix}^\top + \mathbf{r_t^o}
\end{equation*}
where 
\begin{align}
\!\!\!\mathbf{G^o}\!\! := \!\!
\begin{bmatrix}
I_{m\!\times\! m}~0_{m\!\times\! p}  & CF~~CB & C(A\!+\!FC)F~~C(A\!+\!FC)B &\ldots & \quad C(A\!+\!FC)^{h-2}F~~C(A\!+\!FC)^{h-2}B \\
0_{p\!\times\! m}~I_{p\!\times\! p}  &0_{p\!\times\! m}~0_{p\!\times\! p}  & 0_{p\!\times\! m}~~~~~0_{p\!\times\! p} &\ldots& 0_{p\!\times\! m}~~~~~0_{p\!\times\! p} 
\end{bmatrix}
\end{align}
and $\mathbf{r_t^o}$ is the residual vector that represents the effect of $[ e_i \enskip u_i ]$ for $0 \leq i<t-h$, which are independent. Notice that $\mathbf{G^o}$ is full row rank even for $h=1$, due to first block identity matrix. Using this, we can represent $\bar{\phi}_t$ as follows
\begin{align}
    \bar{\phi}_t &= \underbrace{\begin{bmatrix}
    f_{t-1} \\
    \vdots \\
    f_{t-h}
\end{bmatrix}}_{\mathbb{R}^{(m+p)h}} + 
\begin{bmatrix}
    \mathbf{r_{t-1}^o} \\
    \vdots \\
    \mathbf{r_{t-h}^o}
\end{bmatrix}
= \Gol
\underbrace{\begin{bmatrix}
    e_{t-1} \\
    u_{t-1}\\
    \vdots \\
    z_{t-2h} \\
    u_{t-2h}
\end{bmatrix}}_{\mathbb{R}^{2(m+p)h}} + 
\begin{bmatrix}
    \mathbf{r_{t-1}^o} \\
    \vdots \\
    \mathbf{r_{t-h}^o}
\end{bmatrix} \quad \text{ where } \nonumber \\
\Gol &\coloneqq \!\!
\begin{bmatrix}
    [\qquad \qquad \mathbf{G^o} \qquad \qquad] \quad 0_{m+p} \enskip 0_{m+p} \enskip 0_{m+p} \enskip \ldots \\
    0_{m+p} \enskip [\qquad \qquad \mathbf{G^o} \qquad \qquad] \qquad  0_{m+p} \enskip 0_{m+p} \enskip \ldots \\
    \ddots \\
     0_{m+p}  \enskip 0_{m+p} \enskip \ldots \quad [\qquad \qquad \mathbf{G^o} \qquad \qquad] \enskip 0_{m+p} \\
   0_{m+p} \enskip 0_{m+p} \enskip 0_{m+p} \enskip \ldots \qquad [\qquad \qquad \mathbf{G^o} \qquad \qquad]
\end{bmatrix}.
\label{Gol}
\end{align}

During the warm-up period, for all $1\leq t \leq T_{\text{warm}}$, $\Sigma(x_t) \preccurlyeq \mathbf{\Gamma_\infty}$, where $\mathbf{\Gamma_\infty}$ is the steady state covariance matrix of $x_t$ such that,
\begin{equation*}
    \mathbf{\Gamma_\infty} = \sum_{i=0}^\infty \sigma_e^2 (A+FC)^iFF^\top((A+FC)^\top)^i + \sigma_u^2 (A+FC)^iBB^\top((A+FC)^\top)^i .
\end{equation*}
From the $(A+FC)$ stability assumption, we guarantee that the steady-state is bounded. For simplicity, assume that for a finite $\Phi(A+FC)$,  $\|(A+FC)^\tau \| \leq \Phi(A+FC) \rho(A+FC)^\tau$ for all $\tau \geq 0$. This assumption is mild and can be trivially replaced by strong stability condition introduced in \citet{cohen2018online}, which is just a quantification of stability for the analysis. Using this, we get $\|\mathbf{\Gamma_\infty}\| \leq (\sigma_e^2 \|F\|^2 + \sigma_u^2 \|B\|^2) \frac{\Phi(A+FC)^2 \rho(A+FC)^2}{1-\rho(A+FC)^2}$. Therefore, for all $ \leq t \leq T_{\text{warm}}$, with probability $1-\delta/2$, we have
\begin{align}
    \label{exploration norms first}
    \| x_t \| &\leq X_{w} \coloneqq \frac{(\sigma_e \|F\| + \sigma_u \|B\|)  \Phi(A+FC) \rho(A+FC)}{\sqrt{1-\rho(A+FC)^2}}\sqrt{2n\log(12n T_{\text{warm}}/\delta)} , \\
    \| e_t \| &\leq E \coloneqq  R \sqrt{2m\log(12m T_{\text{warm}} /\delta)} , \\ 
    \| u_t \| &\leq U_{w} \coloneqq \sigma_u \sqrt{2p\log(12p T_{\text{warm}} /\delta)}, \\
    \| y_t \| &\leq \|C\| X_w + E
    \label{exploration norms last}.
\end{align}

We can conclude that during the warm-up phase, we have $\max_{i\leq t \leq T_{\text{warm}}} \|\phi_i\| \leq \Upsilon_w \sqrt{h}$ where $\Upsilon_w = \| C\| X_w + E + U_w$. With this we are ready to set the persistence of excitation guarantee for the inputs during the warm-up period. To this end define 
\begin{equation}
    T_{wp} = \frac{32 \Upsilon_w^4 h^2 \log\left(\frac{2h(m+p)}{\delta}\right)}{\sigma_{\min}^4(\Gol) \min \{\sigma_e^4, \sigma_u^4 \}}
\end{equation}

\begin{lemma}\label{lem:openlooppersistence}
If the warm-up duration $T_{\text{warm}} \geq T_{wp}$, then for $T_{wp} \leq t \leq T_{\text{warm}}$, with probability at least $1-\delta$, we have 
\begin{equation}
    \sigma_{\min} \left(\sum_{i=1}^t \phi_t \phi_t^\top \right) \geq t \frac{\sigma_o^2 \min \{\sigma_e^2, \sigma_u^2 \}}{2}
\end{equation}
where $\sigma_o \coloneqq \sigma_{\min}(\Gol)$
\end{lemma}
\begin{proof}
The proof follows similarly with \citet{lale2020logarithmic}. Using the fact that each block row of $\Gol$ is full row rank, via QR decomposition, we get 
\begin{align*}
    \Gol = \underbrace{\begin{bmatrix}
   Q^{o} & 0_{m+p} & 0_{m+p} & 0_{m+p} & \ldots \\
    0_{m+p} & Q^{o} &  0_{m+p} & 0_{m+p} & \ldots \\
    & & \ddots & & \\
     0_{m+p}  & 0_{m+p} & \ldots & Q^{o} & 0_{m+p} \\
    0_{m+p} & 0_{m+p} & 0_{m+p} & \ldots & Q^{o}
\end{bmatrix} }_{\mathbb{R}^{(m+p)h \times (m+p)h }}
\underbrace{\begin{bmatrix}
   R^{o} \!&\! 0_{m+p} \!&\! 0_{m+p} \!&\! 0_{m+p} \!&\! \ldots \\
    0_{m+p} \!&\! R^{o} \!&\!  0_{m+p} \!&\! 0_{m+p} \!&\! \ldots \\
    \!&\! \!&\! \ddots \!&\! \!&\! \\
     0_{m+p}  \!&\! 0_{m+p} \!&\! \ldots \!&\!R^{o} \!&\! 0_{m+p} \\
    0_{m+p} \!&\! 0_{m+p} \!&\! 0_{m+p} \!&\! \ldots \!&\! R^{o}
\end{bmatrix}}_{\mathbb{R}^{(m+p)h \times 2(m+p)h}}
\end{align*}
where $R^{o} = \begin{bmatrix}
       \x & \x & \x & \x & \x & \x &\ldots \\ 0 & \x & \x & \x & \x & \x & \ldots \\ & \ddots & \\ 0 & 0 & 0 & \x & \x & \x & \ldots \end{bmatrix} \in \mathbb{R}^{(m+p) \times h(m+p) }$ 
with positive number on the diagonal. Note that the first matrix in QR decomposition is full rank. Since all the rows of second matrix in QR decomposition are in row echelon form, the second matrix is also full row rank. Therefore, $\Gol$ is full row rank, which gives,
\begin{equation*}
    \mathbb{E}[\bar{\phi}_t \bar{\phi}_t^\top] \succeq \Gol \Sigma_{e,u} \mathcal{G}^{ol \top}
\end{equation*}
where $\Sigma_{e,u} \in \R^{2(m+p)h \times 2(m+p)h} = \text{diag}( \sigma_e^2, \sigma_u^2, \ldots, \sigma_e^2, \sigma_u^2)$. This gives us 
\[ 
\sigma_{\min}(\mathbb{E}[\bar{\phi}_t \bar{\phi}_t^\top]) \geq \sigma_{\min}^2(\Gol) \min \{\sigma_e^2, \sigma_u^2 \}
\]
for $t \geq T_{\text{warm}}$. Using Theorem \ref{azuma} and (\ref{exploration norms first})-(\ref{exploration norms last}), we get, 
\begin{align*}
    \lambda_{\max}\left (\sum_{i=1}^{t} \phi_i \phi_i^\top - \mathbb{E}[\phi_i \phi_i^\top]  \right) \leq 2\sqrt{2t} \Upsilon_w^2 h \sqrt{\log\left(\frac{2h(m+p)}{\delta}\right)}. 
\end{align*}
which holds with probability $1-\delta/2$. Using Weyl's inequality, during the warm-up period with probability $1-\delta$, we have 
\begin{align*}
    \sigma_{\min}\left(\sum_{i=1}^{t} \phi_i \phi_i^\top \right) \geq t \sigma_{\min}^2(\Gol) \min \{\sigma_e^2, \sigma_u^2 \} - 2\sqrt{2t} \Upsilon_w^2 h \sqrt{\log\left(\frac{2h(m+p)}{\delta}\right)}.
\end{align*}
For the given choice of $T_{wp}$, we obtain the advertised lower bound. 
\end{proof}

Recalling Theorem \ref{theo:closedloopid} and using  Lemma \ref{lem:openlooppersistence}, we get 
\begin{equation} 
    \| \wh \G_t  - \G \| \leq  \frac{\beta_t \sqrt{2}}{ \sigma_{o} \min \{\sigma_e, \sigma_u \} \sqrt{T_{\text{warm}}} }  ,\label{persistent_in_effect}
\end{equation}
at the end of warm-up, with probability at least $1-2\delta$.  

\subsection{Persistence of Excitation Condition for Algorithm \ref{algo_strong}, PE of $\Mcontrol \in \Mcontrolset_r$ } \label{subsec:persistence}
In order to derive the persistence of excitation condition, assume that the underlying system is known. Thus, we have the following inputs and outputs to the system
\begin{align*}
     u_t &=  \sum_{j=0}^{h'-1} M_t^{[j]}  \nat_{t-j}(\mathcal{G})\\
     y_t &= [G_{u\rightarrow y}^{1} \! ~\!\ldots\! ~ G_{u\rightarrow y}^{h}~ G_{y\rightarrow y}^{1} \! ~\!\ldots\! ~ G_{y\rightarrow y}^{h} ]  \left[ u_{t-1}^\top \enskip \ldots u_{t-h}^\top \enskip y_{t-1}^\top \enskip \ldots y_{t-h}^\top \right]^\top + \nat_t(\mathcal{G}) + \mathbf{r_{t}}
\end{align*}
where $\mathbf{r_{t}} = \sum_{k=h+1}^{t-1} G_{u\rightarrow y}^{k} u_{t-k} + G_{y\rightarrow y}^{k} y_{t-k} $. Using this we have the following decomposition for $\phi_t$
\begin{align*}
\phi_t
%     \begin{bmatrix}
%     y_{t-1}\\
%     \vdots \\
%     y_{t-h}\\
%     u_{t-1} \\
%     \vdots \\
%     u_{t-h}
% \end{bmatrix} 
&\!\!=\!\! 
% \underbrace{\begin{bmatrix}
%     f_{t-1} \\
%     \vdots \\
%     f_{t-H}
% \end{bmatrix}}_{\mathbb{R}^{(m+p)H}} + 
% \begin{bmatrix}
%     \mathbf{r_{t-1}^o} \\
%     \vdots \\
%     \mathbf{r_{t-H}^o}
% \end{bmatrix}
% =
\underbrace{\begin{bmatrix}
    I_p \! & \! 0_p \! & \! 0_p \! & \! 0_p \! & \! 0_p \! & \! 0_p \! & \! \ldots \! & \! \ldots \! & \! \ldots \! & \! 0_p \\
    0_p \! & \! I_p \! & \!  0_p \! & \! 0_p \! & \! 0_p \! & \! 0_p \! & \! \ldots \! & \! \ldots \! & \! \ldots \! & \! 0_p \\
    \! & \! \! & \! \ddots \! & \! \! & \! \\
    0_p \! & \! 0_p \! & \! \ldots \! & \! I_p \! & \! 0_p \! & \! \ldots \! & \! \ldots \! & \! \ldots \! & \!\ldots  \! & \! 0_p \\ 
    0_{m\times p}  \! & \! G_{u\rightarrow y}^{1} \! & \! \ldots \! & \! \ldots \! & \! G_{u\rightarrow y}^{h} \! & \! 0_{m\times p} \! & \! 0_{m\times p} \! & \! \ldots \! & \! \ldots \! & \! 0_{m\times p} \\
    0_{m \times p} \! &  0_{m \times p} \! & \! G_{u\rightarrow y}^{1} \! & \! \ldots \! & \! \ldots \! & \! G_{u\rightarrow y}^{h} \! & \! 0_{m \times p} \! & \! \ldots \! & \! \ldots \! & \! 0_{m\times p} \\
    \! & \! \! & \! \ddots \! & \! \! & \! \! & \! \! & \! \! & \! \ddots \! & \! \! & \! \\
    0_{m \times p} \! & \! \ldots \! & \! \ldots \! &
\! 0_{m \times p} \! & \! G_{u\rightarrow y}^{1} \! & \! \ldots \! & \! \ldots \! & \! \ldots \! & \! G_{u\rightarrow y}^{h-1} \! & \! G_{u\rightarrow y}^{h}
\end{bmatrix}}_{\T_\Markov \in \R^{h(m+p) \times 2hp}}\! 
\underbrace{\begin{bmatrix}
    u_{t-1}\\
    \vdots \\
    u_{t-h}\\
    \vdots \\
    u_{t-2h} 
\end{bmatrix}}_{\mathcal{U}_t}\!+\!\! \underbrace{\begin{bmatrix}
    0 \\
    \vdots \\
    0 \\
    \nat_{t-1}\\
    \vdots \\
    \nat_{t-h}
\end{bmatrix}}_{B_y(\mathcal{G})(t)} \!\!+\!\! 
\underbrace{\begin{bmatrix}
    0 \\
    \vdots \\
    0 \\
    \mathbf{r_{t-1}}\\
    \vdots \\
    \mathbf{r_{t-h}}
\end{bmatrix}}_{\mathbf{R}_t}
\end{align*}
\begin{align*}
&\mathcal{U}_t \!\!=\!\! \underbrace{\begin{bmatrix}
     M_{t-1}^{[0]}  \!&\! M_{t-1}^{[1]}  \!&\! \ldots \!&\! \ldots \!&\! M_{t-1}^{[h'-1]}  \!&\! 0_{p\times m} \!&\! 0_{p\times m} \!&\! \ldots \!&\! 0_{p\times m} \\
    0_{p \times m} \!&\! M_{t-2}^{[0]}  \!&\! \ldots \!&\! \ldots \!&\! M_{t-2}^{[h'-2]} \!&\! M_{t-2}^{[h'-1]} \!&\!\! 0_{p \times m} \!\!&\! \ldots \!&\! 0_{p\times m} \\
    \!&\! \!&\! \ddots \!&\! & & & & \ddots & & \\
    0_{p \times m} \!&\! \ldots \!&\! 0_{p \times m} \!&\! M_{t-2h}^{[0]} \!&\! \ldots \!&\! \ldots \!&\! \ldots \!\!&\!\! \ldots \!&\! M_{t-2h}^{[h'-1]}\!\!\!
\end{bmatrix}}_{\T_{\Mcontrol_{t}} \in \R^{2hp \times m(2h+h'-1) }}
\underbrace{\begin{bmatrix}
    \nat_{t-1} \\
    \nat_{t-2} \\
    \vdots \\
    \nat_{t-h'+1}\\
    \vdots \\
    \nat_{t-2h-h'+1}
\end{bmatrix}}_{B_t(\mathcal{G})}
\end{align*}
Note that $\nat_t = e_t$ whose covariance matrix is dominated by $\sigma_e^2 I_p$. Therefore,
\begin{align*}
 B_t(\mathcal{G}) &\!=\! 
\begin{bmatrix}
    e_{t-1} \\
    e_{t-2} \\
    \vdots \\
    e_{t-2h-h'+1}
\end{bmatrix} \quad B_y(\mathcal{G})(t) \!=\!  \underbrace{\begin{bmatrix}
   \begin{matrix}
      \bigzero_{ \left(hp\right) \times \left(m(2h+h'-1) \right)}
   \end{matrix} \\
   \begin{matrix}
       I_m \!&\! \ldots \!&\! 0_m \!&\!  \!&\! \!&\! \!&\! \!&\! \!&\! \!&\!  \\ \vdots
    \!&\! \ddots \!&\! \vdots  \!&\!  \!&\! \!&\! \!&\! \!&\! \bigzero_{ \left(hm\right) \times \left((h+h'-1)m \right)} \!&\! \!&\! \\
    0_m \!&\! \ldots \!&\! I_m \!&\!  \!&\! \!&\! \!&\! \!&\!\!\!  \!&\!  \!&\!   
   \end{matrix} 
\end{bmatrix}}_{\bar{\mathcal{O}}_t}  B_t(\mathcal{G}).
\end{align*}
Thus, 
\begin{equation*}
    \phi_t = \left(\T_\Markov \T_{\Mcontrol_{t}} + \bar{\mathcal{O}}_t \right)  B_t(\mathcal{G}) + \mathbf{R}_t.
\end{equation*}
The following gives the persistence of excitation condition:

	\begin{condition}
		For the given ARX system $\Theta$, for $t \geq 2h + h'$, $\T_\Markov \T_{\Mcontrol_{t}}   + \bar{\mathcal{O}}_t$ is full row rank for all $\Mcontrol \in \Mcontrolset_r$, \textit{i.e.},  
		\begin{equation} \label{precise_persistence}
		\sigma_{\min} (\T_\Markov \T_{\Mcontrol_{t}} + \bar{\mathcal{O}}_t) > \sigma_c > 0.
		\end{equation}
	\end{condition}

Note that for simplicity of the analysis, the length of the estimated Markov operator $\mathcal{G}$ in (\ref{new_lse}) and the Markov operator to recover output uncertainties $\nat(\mathcal{G})$ are chosen to be the same in the main text. However, in practice, the length of the estimator could be increased to satisfy the persistence of excitation condition. 

\subsection{Persistence of Excitation in Adaptive Control Period of Algorithm \ref{algo_strong}}
\label{subsec:theo:persist_adaptive}

In this section, we show that the Markov parameter estimates ($\hat{\mathcal{G}}_t$) throughout the adaptive control period of Algorithm \ref{algo_strong} are close enough to the underlying parameters such that the controllers designed via these estimates do not violate the persistence of excitation condition. Define $T_{\mathcal{G}}$ such that $\|\wh{\mathcal{G}}_{t} - \mathcal{G} \|_F \leq 1$. Let 

\begin{equation}
    T_{cl} = \frac{4T_{\mathcal{G}} \kappa_\Mcontrolset^2 \kappa_\Markov^2}{\left(\frac{3\sigma_c^2 \sigma_e^2 }{16h\kappa_y^3 \kappa_{\nat} } - \frac{1}{10T}\right)^2}, \quad T_c = \frac{2048 \Upsilon_c^4 h^2 \log\left(\frac{h(m+p)}{\delta}\right) + h'mp\log\left(\kappa_\Mcontrolset\sqrt{\min\{m,p\}} + \frac{2}{\epsilon} \right)}{\sigma_{c}^4 \sigma_e^4}
\end{equation}
for $\epsilon = \min \left\{1, \frac{\sigma_{c}^2 \sigma_e^2  \sqrt{\min\{m,p\}} }{68 \kappa_\nature^3 \kappa_\Markov h \left(2\kappa_\Mcontrolset^2 + 3\kappa_\Mcontrolset + 3 \right) } \right \}$.

\begin{lemma} \label{closedloop_persistence_appendix}
After $T_{c}$ time steps of adaptive control period of Algorithm \ref{algo_strong}, with probability $1-3\delta$, the following holds for the remainder of adaptive control period, 
\begin{equation}
\sigma_{\min}\left(\sum_{i=1}^{t} \phi_i \phi_i^\top \right) \geq t \frac{\sigma_{c}^2  \sigma_e^2}{16}. 
\end{equation}
\end{lemma}
\begin{proof}
During the adaptive control period, at time $t$, the input of Algorithm \ref{algo_strong} is given by
\begin{align*}
    u_t &= \sum_{j=0}^{h'-1}M_t^{[j]}  \nat_{t-j}(\G) + M_t^{[j]} \left( \nat_{t-j}(\wh \G_i) - \nat_{t-j}(\G)  \right) 
\end{align*}
where 
\begin{align}
    \nat_{t-j}(\G) &= y_{t-j} - \left(  \sum_{k=1}^{t-j-1} G_{u\rightarrow y}^{k} u_{t-j-k} + G_{y\rightarrow y}^{k} y_{t-j-k} \right)= CA^{t-j} x_0 + e_{t-j} \\ 
    \nat_{t-j}(\wh \G_i) &= y_{t-j} - \left(\sum_{k=1}^{h} \wh{G}_{u\rightarrow y}^{k} u_{t-j-k} + \wh{G}_{y\rightarrow y}^{k} y_{t-j-k} \right).
\end{align}
This gives the following input and output at time $t$:
\begin{align}
     u_t &=  \sum_{j=0}^{h'-1} M_t^{[j]}  \nat_{t-j}(\G) + \underbrace{\sum_{j=0}^{h'-1} M_t^{[j]} \left( \sum_{k=1}^{t-j-1} [G_{u\rightarrow y}^{k} - \wh G_{u\rightarrow y}^{k}] u_{t-j-k} + [G_{y\rightarrow y}^{k} - \wh G_{y\rightarrow y}^{k}] y_{t-j-k} \right) }_{u_{\Delta \nat}(t)} \label{lookatthis}\\
     y_t &= [G_{u\rightarrow y}^{1} \! ~\!\ldots\! ~ G_{u\rightarrow y}^{h}~ G_{y\rightarrow y}^{1} \! ~\!\ldots\! ~ G_{y\rightarrow y}^{h} ]  \left[ u_{t-1}^\top \enskip \ldots u_{t-h}^\top \enskip y_{t-1}^\top \enskip \ldots y_{t-h}^\top \right]^\top + \nat_t(\mathcal{G}) + \mathbf{r_{t}} \nonumber
\end{align}
where $\mathbf{r_{t}} = \sum_{k=h+1}^{t-1} G_{u\rightarrow y}^{k} u_{t-k} + G_{y\rightarrow y}^{k} y_{t-k} $. Moreover, we have Assumption \ref{asm:openloopdecay} on the choice of $h$ which is satisfied under the minimal assumption of stability. From Lemma \ref{lem:boundednature} and Lemma \ref{lem:concentrationgeneral}, we get $\|u_{\Delta \nat}(t) \| \leq \kappa_\Mcontrolset \kappa_y \epsilon_{\Markov}(1,\delta)$ for all $t$ in adaptive control period, where $\kappa_\Mcontrolset = \kappa_\psi(1+r)$. Using the definitions from Appendix \ref{subsec:persistence}, $\phi_t$ can be written as,
\begin{equation*}
    \phi_t = \left(\T_\Markov \T_{\Mcontrol_{t}} + \bar{\mathcal{O}}_t \right)  B_t(\mathcal{G}) + \mathbf{R}_t + \T_\Markov \mathcal{U}_{\Delta \nat}(t)
\end{equation*}
where $\mathcal{U}_{\Delta \nat}(t) = [ u^\top_{\Delta \nat}(t-1),~ u^\top_{\Delta \nat}(t-2),~ \ldots,~ u^\top_{\Delta \nat}(t-2h) ]^\top $. We have
\begin{align*}
    \mathbb{E}\left[\phi_t \phi_t^\top \right] &= \mathbb{E}\bigg[ \left(\T_\Markov \T_{\Mcontrol_{t}} + \bar{\mathcal{O}}_t \right)  B_t(\mathcal{G})  B_t(\mathcal{G})^\top  \left(\T_\Markov \T_{\Mcontrol_{t}}   + \bar{\mathcal{O}}_t \right)^\top + B_t(\mathcal{G})^\top \left(\T_\Markov \T_{\Mcontrol_{t}}   + \bar{\mathcal{O}}_t \right)^\top \left( \T_\Markov \mathcal{U}_{\Delta \nat}(t) + \mathbf{R}_t \right) \\
    &+ \left( \T_\Markov \mathcal{U}_{\Delta \nat}(t) + \mathbf{R}_t \right)^\top \left(\T_\Markov \T_{\Mcontrol_{t}}   + \bar{\mathcal{O}}_t \right) B_t(\mathcal{G}) + \left( \T_\Markov \mathcal{U}_{\Delta \nat}(t) + \mathbf{R}_t \right)^\top \left( \T_\Markov \mathcal{U}_{\Delta \nat}(t) + \mathbf{R}_t \right)  \bigg] 
\end{align*}
Using Lemma \ref{lem:boundednature},
\begin{align*}
    \sigma_{\min}\left(\mathbb{E}\left[\phi_t \phi_t^\top \right]\right) &\geq \sigma_c^2 \sigma_e^2 - 4h\kappa_\nature \kappa_y^3 (2\kappa_\Markov\kappa_\Mcontrolset \epsilon_\Markov(1, \delta) + 1/10T)
\end{align*}
By setting $T_{\text{warm}} \geq T_{cl}$, we get 
$\epsilon_\Markov(1, \delta) \leq \frac{1}{2\kappa_\Mcontrolset \kappa_\Markov} \left(\frac{3\sigma_c^2 \sigma_e^2 }{16h\kappa_y^3 \kappa_{\nat} } - \frac{1}{10T}\right)$ with probability at least $1-2\delta$, which gives $\sigma_{\min}\left(\mathbb{E}\left[\phi_t \phi_t^\top \right]\right) \geq \frac{\sigma_{c}^2\sigma_e^2}{4} $ for all $t\geq T_{\text{warm}}$. Let $\Upsilon_c \coloneqq (\kappa_y + \kappa_u)$. Lemma \ref{lem:boundednature} gives us that $\|\phi_t\| \leq \Upsilon_c \sqrt{h}$ with probability at least $1-2\delta$. Therefore, for a chosen $\Mcontrol \in \Mcontrolset_r$, using Theorem \ref{azuma}, we have the following with probability $1-3\delta$:
\begin{align}
    \lambda_{\max}\left (\sum_{i=1}^{t} \phi_i \phi_i^\top - \mathbb{E}[\phi_i \phi_i^\top]  \right) \leq 2\sqrt{2t} \Upsilon_c^2 h \sqrt{\log\left(\frac{h(m+p)}{\delta}\right)}.
\end{align}

Finally, a standard covering argument will be utilized to show that this holds for any chosen $\Mcontrol \in \Mcontrolset_r$. We know that from Lemma 5.4 of \citet{simchowitz2020improper}, the Euclidean diameter of $\Mcontrolset_r$ is at most $2\kappa_\Mcontrolset\sqrt{\min\{m,p\}}$, \textit{i.e.} $\| \Mcontrol_t\|_F \leq \kappa_\Mcontrolset\sqrt{\min\{m,p\}}$ for all $\Mcontrol_t \in \Mcontrolset_r$. Thus, we can upper bound the covering number as follows, 
\begin{equation*}
    \mathcal{N}(B(\kappa_\Mcontrolset\sqrt{\min\{m,p\}}), \|\cdot \|_F, \epsilon) \leq \left(\kappa_\Mcontrolset\sqrt{\min\{m,p\}} + \frac{2}{\epsilon} \right)^{h'mp}.
\end{equation*}
This gives us the following result for all centers of $\epsilon$-balls in $\| \Mcontrol_t\|_F$, for all $t\geq T_{\text{warm}}$, with probability $1-3\delta$:
\begin{align}
    \lambda_{\max}\left (\sum_{i=1}^{t} \phi_i \phi_i^\top - \mathbb{E}[\phi_i \phi_i^\top]  \right) \leq 2\sqrt{2t} \Upsilon_c^2 h \sqrt{\log\left(\frac{h(m+p)}{\delta}\right) + h'mp\log\left(\kappa_\Mcontrolset\sqrt{\min\{m,p\}} + \frac{2}{\epsilon} \right)}.
\end{align}
Consider all $\Mcontrol$ in the $\epsilon$-balls, \textit{i.e.} effect of $\epsilon$-perturbation in $\|\Mcontrol \|_F$ sets, using Weyl's inequality we have with probability at least $1-3\delta$,
\begin{align*}
    \sigma_{\min}\left(\sum_{i=1}^{t} \phi_i \phi_i^\top \right) &\geq t \left( \frac{\sigma_{c}^2\sigma_e^2 }{4}  - \frac{8 \kappa_\nature^3 \kappa_\Markov h \epsilon \left(2\kappa_\Mcontrolset^2 + 3\kappa_\Mcontrolset + 3 \right)}{\sqrt{\min\{m,p \}}} \left(1+\frac{1}{10T}\right) \right) \\
    &\quad\qquad- 2\sqrt{2t} \Upsilon_c^2 h \sqrt{\log\left(\frac{h(m+p)}{\delta}\right) + h'mp\log\left(\kappa_\Mcontrolset\sqrt{\min\{m,p\}} + \frac{2}{\epsilon} \right)}.
\end{align*}
for $\epsilon \leq 1$. Let $\epsilon = \min \left\{1, \frac{\sigma_{c}^2 \sigma_e^2  \sqrt{\min\{m,p\}} }{68 \kappa_\nature^3 \kappa_\Markov h \left(2\kappa_\Mcontrolset^2 + 3\kappa_\Mcontrolset + 3 \right) } \right \}$. For this choice of $\epsilon$, we get 

\begin{align*}
    \sigma_{\min}\left(\sum_{i=1}^{t} \phi_i \phi_i^\top \right) &\geq t \left( \frac{\sigma_{c}^2\sigma_e^2 }{8} \right)\\
    &\quad\qquad- 2\sqrt{2t} \Upsilon_c^2 h \sqrt{\log\left(\frac{h(m+p)}{\delta}\right) + h'mp\log\left(\kappa_\Mcontrolset\sqrt{\min\{m,p\}} + \frac{2}{\epsilon} \right)}.
\end{align*}

For picking $T_{\text{warm}} \geq T_c$, we can guarantee that after $T_c$ time steps in the first epoch of adaptive control, we obtain the lower bound. 
\end{proof}

Recalling Theorem \ref{theo:closedloopid} and using  Lemma \ref{closedloop_persistence_appendix}, we get 
\begin{equation} 
    \| \wh \G_t  - \G \| \leq  \frac{4\beta_t}{ \sigma_{c} \sigma_{e} \sqrt{2^{i-1}T_{\text{warm}}}}  ,\label{persistent_in_effect}
\end{equation}
for all adaptive control epoch $i$, with probability at least $1-4\delta$.

\subsection{Persistence of Excitation Condition for Algorithm \ref{algo_quad}, PE of optimal controller of ARX} \label{subsec:persistence_arx_opt}
After the warm-up phase, for $t\geq T_{warm}$, Algorithm \ref{algo_quad} executes the input of $u_t = \tk_{t}^x x_t + \tk_{t}^y y_t $. Let $f_t = [y_t^\top u_t^\top]^\top$. Using the state-space representation of ARX model, we get 
\begin{equation*}
    f_t = 
\underbrace{\begin{bmatrix}
    C \\
    \tk_{t}^x + \tk_{t}^y C
\end{bmatrix}}_{\mathbf{\tilde{\Gamma}_{t}}}
x_t + 
\underbrace{\begin{bmatrix}
    I  \\
    \tk_{t}^y
\end{bmatrix}}_{\mathbf{\tilde{\Omega}_{t}}}
e_t
\end{equation*}
Moreover, $x_t = \underbrace{[A + B\tk_{t-1}^x + FC + B\tk_{t-1}^yC]}_{\mathbf{\tilde{\Lambda}_{t-1}}} x_{t-1} + \underbrace{[F+B\tk_{t-1}^y]}_{\mathbf{\tilde{\Xi}_{t-1}}} e_{t-1} $. Thus for $f_t$, we get:
\begin{equation*}
    f_t = \mathbf{\tilde{\Gamma}_{t}} \mathbf{\tilde{\Lambda}_{t-1}} x_{t-1} +  \mathbf{\tilde{\Gamma}_{t}} \mathbf{\tilde{\Xi}_{t-1}} e_{t-1} + \mathbf{\tilde{\Omega}_{t}} e_t.
\end{equation*}
Rolling back for $h$ time steps, we get the following,
\begin{equation*}
    f_t = \mathbf{\tilde{\Gamma}_{t}} \left( \sum_{i=t-h+1}^t \left( \prod_{j=i}^{t-1} \mathbf{\tilde{\Lambda}_{j}} \right) \mathbf{\tilde{\Xi}_{i-1}} e_{i-1} \right) + \mathbf{\tilde{\Omega}_{t}} e_t + \mathbf{r_t^c}
\end{equation*}
where $\mathbf{r_t^c}$ is the residual vector that represents the effect of $e_i$ for $0 \leq i<t-h$, which are independent. Using this, one can write the full characterization of $\bar{\phi}_t$ as follows
\begin{equation*}
    \bar{\phi}_t = \begin{bmatrix}
    f_{t-1} \\
    \vdots \\
    f_{t-h}
\end{bmatrix} + \begin{bmatrix}
    \mathbf{r_{t-1}^c} \\
    \vdots \\
    \mathbf{r_{t-h}^c}
\end{bmatrix} = \Gcl_t
\underbrace{\begin{bmatrix}
    e_{t-1} \\
    e_{t-2} \\
    \vdots \\
    e_{t-2h-1}\\
    e_{t-2h}
\end{bmatrix}}_{\mathbb{R}^{2hm}} + \begin{bmatrix}
    \mathbf{r_{t-1}^c} \\
    \vdots \\
    \mathbf{r_{t-H}^c}
\end{bmatrix}
\end{equation*}
where
\begin{equation}
    \Gcl_t = \begin{bmatrix}[\qquad \enskip \mathbf{\bar{G}_{t-1}} \enskip \qquad] \qquad 0_{(m+p)\times m} \enskip 0_{(m+p)\times m} \enskip 0_{(m+p)\times m} \enskip \ldots \\
    0_{(m+p)\times m} \enskip [\qquad \enskip \mathbf{\bar{G}_{t-2}} \enskip \qquad] \qquad  0_{(m+p)\times m}  \enskip 0_{(m+p)\times m} \enskip \ldots \\
    \ddots \\
     0_{(m+p)\times m}  \enskip 0_{(m+p)\times m} \enskip \ldots \quad [\qquad \enskip \mathbf{\bar{G}_{t-h+1}} \enskip \qquad] \enskip 0_{(m+p)\times m} \\
    0_{(m+p)\times m} \enskip 0_{(m+p)\times m} \enskip 0_{(m+p)\times m} \enskip \ldots \qquad [\qquad \enskip \mathbf{\bar{G}_{t-h}} \enskip \qquad]
    \end{bmatrix}
\end{equation}
for 
\begin{equation*}
    \mathbf{\bar{G}_{t}} \!=\! \begin{bmatrix}
    \mathbf{\tilde{\Omega}_{t}}, \!&\! \mathbf{\tilde{\Gamma}_t} \mathbf{\tilde{\Xi}_{t-1}}, \!&\! \mathbf{\tilde{\Gamma}_t} \mathbf{\tilde{\Lambda}_{t-1}} \mathbf{\tilde{\Xi}_{t-2}},\ldots, \!&\! \mathbf{\tilde{\Gamma}_t} \mathbf{\tilde{\Lambda}_{t-1}} \mathbf{\tilde{\Lambda}_{t-2}} \ldots \mathbf{\tilde{\Lambda}_{t+h-1}} \mathbf{\tilde{\Xi}_{t-h}}
\end{bmatrix} \in \mathbb{R}^{(m+p) \times hm}
\end{equation*}
If the underlying system is known, then the optimal control law for the ARX system could be applied to control the system. In the following, $\Gcl$ is the closed-loop mapping of noise process to the covariates $\bar{\phi}$ via optimal policy
\begin{equation}
    \Gcl = \begin{bmatrix}[\qquad \enskip \mathbf{\bar{G}} \enskip \qquad] \qquad 0_{(m+p)\times m} \enskip 0_{(m+p)\times m} \enskip 0_{(m+p)\times m} \enskip \ldots \\
    0_{(m+p)\times m} \enskip [\qquad \enskip \mathbf{\bar{G}} \enskip \qquad] \qquad  0_{(m+p)\times m}  \enskip 0_{(m+p)\times m} \enskip \ldots \\
    \ddots \\
     0_{(m+p)\times m}  \enskip 0_{(m+p)\times m} \enskip \ldots \quad [\qquad \enskip \mathbf{\bar{G}} \enskip \qquad] \enskip 0_{(m+p)\times m} \\
    0_{(m+p)\times m} \enskip 0_{(m+p)\times m} \enskip 0_{(m+p)\times m} \enskip \ldots \qquad [\qquad \enskip \mathbf{\bar{G}} \enskip \qquad]
    \end{bmatrix}
\end{equation}
for 
\begin{equation*}
    \mathbf{\bar{G}} \!=\! \begin{bmatrix}
    \mathbf{\Omega}, \!&\! \mathbf{\Gamma} \mathbf{\Xi}, \!&\! \mathbf{\Gamma} \mathbf{\Lambda} \mathbf{\Xi},\ldots, \!&\! \mathbf{\Gamma} \mathbf{\Lambda}^{h-1} \mathbf{\Xi}
\end{bmatrix} 
\end{equation*}
where 
\begin{equation*}
    \mathbf{\Omega} = \begin{bmatrix}
    I  \\
    K_{y}^*
\end{bmatrix}, \enskip \mathbf{\Gamma} = \begin{bmatrix}
    C \\
    K_{x}^* + K_{y}^* C
\end{bmatrix}, \enskip \mathbf{\Lambda} = [A + BK_{x}^* + FC + BK_{y}^*C], \enskip \mathbf{\Xi} = [F+BK_{y}^*].
\end{equation*}
Note that $\mathbf{\bar{G}}$ corresponds to truncated closed-loop noise to covariate Markov operator. Notice that if $\mathbf{\bar{G}}$ is full row rank, following similar approach with the proof of Lemma \ref{lem:openlooppersistence}, $\Gcl$ is also full row rank. Thus, we have the following persistence of excitation condition on the optimal control law for the ARX system:  

\begin{condition1}
The length of Markov operator to estimate is chosen such that $\mathbf{\bar{G}}$ formed via optimal control policy of the given ARX system is full row rank. Thus, $\sigma_{\min}(\Gcl) > \bar{\sigma}_{c} > 0$.
\end{condition1}

\subsection{Persistence of Excitation in Adaptive Control Period of Algorithm \ref{algo_quad}}
\label{subsec:theo:persist_adaptive_optimism}

Finally, in this section we show that the Markov parameter estimates ($\hat{\mathcal{G}}_t$) throughout the adaptive control period of Algorithm \ref{algo_quad} are close enough to the underlying parameters such that the optimistic controllers designed via these estimates still persistently excite the ARX system. To this end, define 

\begin{align*}
    T_{\bar{c}} &= \frac{2048\Upsilon_c^4h^2 \left( \log\left(\frac{h(m+p)}{\delta}\right) + 2(m+p)mh^2\log\left(G_r + \frac{2}{\epsilon} \right) \right)}{\bar{\sigma}_c^4\sigma_e^4}, \\ T_{\Gcl} &= T_{param} \left(\frac{2h + 2h \kappa_{K_x}\kappa_{K_y} + 2h(h-1)\kappa_{K_x}\kappa_{K_y}}{\bar{\sigma}_c}\right)^2
\end{align*}
for $\epsilon = \min \left\{1, \frac{\bar{\sigma}_c^2\sigma_e^2 }{16\left( h\Upsilon_c\kappa_{\nat}\sqrt{2}  + h \kappa_{\nat}^2 +  \sigma_e^2/2 \right)} \right \}$, where $T_{param}$ is the number of samples required to get less than 1 estimation error on the system parameters, defined in Section \ref{apx:sysidarx}. Moreover, $\kappa_{K_x}$ and $\kappa_{K_y}$ are bounds on the optimistic controllers within $\mathcal{S}$ due to boundedness of the set. Finally, let $G$ be the upper bound on $\|\tilde{\Gcl}\|$ constructed via any ARX model parameter in the set $\mathcal{S}$ and let $G_r = G + \frac{\bar{\sigma}_c \sqrt{h(m+p)}}{2}$.

\begin{lemma} \label{closedloop_persistence_appendix_algo_quad}
After $T_{\bar{c}}$ time steps of adaptive control period of Algorithm \ref{algo_quad}, with probability $1-3\delta$, the following holds for the remainder of adaptive control period, 
\begin{equation}
\sigma_{\min}\left(\sum_{i=1}^{t} \phi_i \phi_i^\top \right) \geq t \frac{\bar{\sigma}_{c}^2  \sigma_e^2}{16}. 
\end{equation}
\end{lemma}

\begin{proof}
Let $\tilde{\Gcl}$ be the closed-loop mapping of noise process to the covariates via optimal policy for the optimistically chosen ARX parameters. Recall that via persistence of excitation condition on the optimal controller, picking $T_{warm} \geq T_{\Gcl}$ guarantees that in adaptive control period of Algorithm \ref{algo_quad}, we have $\|\Gcl_t - \Gcl \| \leq \bar{\sigma}_c /2 $ which in turn gives $\sigma_{\min}(\Gcl_t) \geq \bar{\sigma}_c /2$ via Weyl's inequality. Thus, for all $t\geq T_{warm}$, we have that 
\begin{equation*}
    \mathbb{E}[\bar{\phi}_t \bar{\phi}_t^\top] \succeq \Gcl_t \Sigma_{e} \mathcal{G}_t^{cl \top}
\end{equation*}
where $\Sigma_{e} \in \R^{2mh \times 2mh} = \text{diag}(\sigma_e^2, \ldots, \sigma_e^2)$. This gives us $\sigma_{\min}(\mathbb{E}[\bar{\phi}_t \bar{\phi}_t^\top]) \geq \frac{\bar{\sigma}_c^2\sigma_e^2}{4}$ for $t \geq T_{warm}$. Let $\|\phi_t\| \leq \Upsilon_c \sqrt{h}$ (which holds with probability at least $1-2\delta$, see Section \ref{apx:quad_control} for the bound on inputs and outputs during the execution of Algorithm \ref{algo_quad}). Therefore, for a chosen optimistic model, using Theorem \ref{azuma}, we have the following with probability $1-3\delta$:
\begin{align}
    \lambda_{\max}\left (\sum_{i=1}^{t} \phi_i \phi_i^\top - \mathbb{E}[\phi_i \phi_i^\top]  \right) \leq 2\sqrt{2t} \Upsilon_c^2 h \sqrt{\log\left(\frac{h(m+p)}{\delta}\right)}.
\end{align}
Finally, a standard covering argument will be utilized to show that this holds for any chosen $\Mcontrol \in \Mcontrolset_r$. We know that $\| \Gcl_t\|_F \leq G_r$ for all $\Gcl_t$. Thus, we can upper bound the covering number as follows, 
\begin{equation*}
    \mathcal{N}(B(G_r), \|\cdot \|_F, \epsilon) \leq \left(G_r + \frac{2}{\epsilon} \right)^{2(m+p)mh^2}.
\end{equation*}
This gives us the following result for all centers of $\epsilon$-balls in $\| \Gcl_t\|_F$, for all $t\geq T_{\text{warm}}$, with probability $1-3\delta$:
\begin{align}
    \lambda_{\max}\left (\sum_{i=1}^{t} \phi_i \phi_i^\top - \mathbb{E}[\phi_i \phi_i^\top]  \right) \leq 2\sqrt{2t} \Upsilon_c^2 h \sqrt{\log\left(\frac{h(m+p)}{\delta}\right) + 2(m+p)mh^2\log\left(G_r + \frac{2}{\epsilon} \right)}.
\end{align}
Consider all $\Gcl$ in the $\epsilon$-balls, \textit{i.e.} effect of $\epsilon$-perturbation in $\|\Gcl \|_F$ sets, using Weyl's inequality we have with probability at least $1-3\delta$,
\begin{align*}
    \sigma_{\min}\left(\sum_{i=1}^{t} \phi_i \phi_i^\top \right) &\geq t \left( \frac{\bar{\sigma}_c^2\sigma_e^2}{4} - 2\epsilon \left( h\Upsilon_c\kappa_{\nat}\sqrt{2}  + h \kappa_{\nat}^2 +  \sigma_e^2/2 \right) \right) \\
    &\quad\qquad- 2\sqrt{2t} \Upsilon_c^2 h \sqrt{\log\left(\frac{h(m+p)}{\delta}\right) + 2(m+p)mh^2\log\left(G_r + \frac{2}{\epsilon} \right)}
\end{align*}
for $\epsilon \leq 1$. Let $\epsilon = \min \left\{1, \frac{\bar{\sigma}_c^2\sigma_e^2 }{16\left( h\Upsilon_c\kappa_{\nat}\sqrt{2}  + h \kappa_{\nat}^2 +  \sigma_e^2/2 \right)} \right \}$. For this choice of $\epsilon$, we get 
\begin{align*}
    &\sigma_{\min}\left(\sum_{i=1}^{t} \phi_i \phi_i^\top \right) \geq t \left( \frac{\bar{\sigma}_c^2\sigma_e^2}{8} \right) - 2\sqrt{2t} \Upsilon_c^2 h \sqrt{\log\left(\frac{h(m+p)}{\delta}\right) + 2(m+p)mh^2\log\left(G_r + \frac{2}{\epsilon} \right)}.
\end{align*}
For picking $T_{\text{warm}} \geq T_{\bar{c}}$, we can guarantee that after $T_{\bar{c}}$ time steps in the first epoch of adaptive control, we obtain the lower bound. 
\end{proof}

\section{Warm-up Durations for Algorithm \ref{algo_strong}} \label{apx:warm_strong}

\subsection{Explore and Commit Approach}
The warm-up duration has to be chosen to guarantee:
\begin{itemize}
    \item Persistence of excitation during the warm-up period to have reliable estimates for the exploitation phase, $T_{wp} = \frac{32 \Upsilon_w^4 h^2 \log\left(\frac{2h(m+p)}{\delta}\right)}{\sigma_{\min}^4(\Gol) \min \{\sigma_e^4, \sigma_u^4 \}}$ in Section \ref{subsec:warmuppersist},
    \item Sublinear regret upper bound, $T_{reg} = \frac{ \Bar{\beta} \kappa_{\Mcontrolset} \kappa_{\nat}h \sqrt{h' T}}{\kappa_y \sigma_c \sigma_e} \sqrt{ \frac{ \kappa_\Markov^2  \kappa_{\nat}^2 \left(\smooth + L\right)^2}{\alpha} + \kappa_y^2 (h+h') \max \left\{L, \frac{48L^2}{\alpha} \right\} }$ in Section \ref{apx:strong_convex_control},
    \item Conditional strong-convexity of the expected counterfactual losses, $T_{cx} = \frac{1024h^2\Bar{\beta}^2 \kappa_{\nat}^2 \kappa_\Mcontrolset^2 \kappa_\Markov^2 h'  \strong}{\alpha\sigma_c^2 \sigma_e^2}$, in Section \ref{apx:strong_convex_control},
    \item Stability of inputs and outputs, $T_\Markov = \frac{256h^2 \Bar{\beta}^2 \kappa_\Mcontrolset^2 \kappa_\Markov^2   }{\sigma_c^2 \sigma_e^2}$ throughout Algorithm \ref{algo_strong}, in Section \ref{apx:strong_convex_control},
    \item Existence of a good comparator policy in $\Mcontrolset_r$, $T_r = \frac{64h^2 \Bar{\beta}^2 \kappa_\psi^2   }{r^2\sigma_c^2 \sigma_e^2} $, Section \ref{apx:strong_convex_control}.
\end{itemize}

Therefore, for the explore and commit approach of Algorithm \ref{algo_strong}, the warm-up duration

\begin{equation} \label{warm_up_algo_strong_exp}
    \Tw \geq \max \{h, h', T_{wp}, T_{reg}, T_{cx}, T_{\Markov}, T_r \}.
\end{equation}

\subsection{Closed-Loop Model Estimate Updates}

In the closed-loop model estimate variation of Algorithm \ref{algo_strong}, the warm-up duration does not depend on the time horizon. Instead, the warm-up phase should guarantee that: 
\begin{itemize}
    \item Persistence of excitation during the adaptive control period, $T_{cl} = \frac{4T_{\mathcal{G}} \kappa_\Mcontrolset^2 \kappa_\Markov^2}{\left(\frac{3\sigma_c^2 \sigma_e^2 }{16h\kappa_y^3 \kappa_{\nat} } - \frac{1}{10T}\right)^2}$, where $T_\G$ is the warm-up duration to get at least unit norm estimation error at the end of the warm-up phase, in Section \ref{subsec:theo:persist_adaptive},
    \item Well-refined estimate of the Markov operator at the first epoch of adaptive control, $ T_c = \frac{2048 \Upsilon_c^4 h^2 \log\left(\frac{h(m+p)}{\delta}\right) + h'mp\log\left(\kappa_\Mcontrolset\sqrt{\min\{m,p\}} + \frac{2}{\epsilon} \right)}{\sigma_{c}^4 \sigma_e^4}$, where $\epsilon= \min \left\{1, \frac{\sigma_{c}^2 \sigma_e^2  \sqrt{\min\{m,p\}} }{68 \kappa_\nature^3 \kappa_\Markov h \left(2\kappa_\Mcontrolset^2 + 3\kappa_\Mcontrolset + 3 \right) } \right \}$, in Section \ref{subsec:theo:persist_adaptive}.
\end{itemize}
Therefore, for the closed-loop model estimate variate of Algorithm \ref{algo_strong}, the warm-up duration $\tau$
\begin{equation}\label{warm_up_algo_strong_adapt}
    \tau \geq \max \{h, h', T_{wp}, T_{cx}, T_{\Markov}, T_r, T_{cl}, T_c \}.
\end{equation}

\section{Warm-up Durations for Algorithm \ref{algo_quad}} \label{apx:warm_quad}

\subsection{Explore and Commit Approach}
The warm-up duration has to be chosen to guarantee:
\begin{itemize}
   \item Persistence of excitation during the warm-up period to have reliable estimates for the exploitation phase, $T_{wp} = \frac{32 \Upsilon_w^4 h^2 \log\left(\frac{2h(m+p)}{\delta}\right)}{\sigma_{\min}^4(\Gol) \min \{\sigma_e^4, \sigma_u^4 \}}$ in Section \ref{subsec:warmuppersist},
   \item Reliable ARX system parameter estimation, $T_N = T_{\G} \frac{8h}{\sigma_n^2(\mathcal{H})}$, where $T_\G$ is the warm-up duration to get at least unit norm estimation error at the end of the warm-up phase and $\sigma_n(\mathcal{H})$ is the $n$-th singular value of Hankel matrix constructed by the Markov parameters in Section \ref{apx:sysidarx}
   \item Stability of inputs and outputs, throughout Algorithm \ref{algo_quad}, $T_u$ and $T_y$ in Section \ref{apx:quad_control}.
\end{itemize}

Therefore, for the explore and commit approach of Algorithm \ref{algo_quad}, the warm-up duration

\begin{equation} \label{warm_up_algo_exp}
\Tw \geq \max \{h, T_{wp}, T_{N}, T_{u}, T_y, T^{2/3} \} .
\end{equation}

\subsection{Closed-Loop Model Estimate Updates}
In the closed-loop model estimate variation of Algorithm \ref{algo_strong}, the warm-up duration does not depend on the time horizon. Instead, the warm-up phase should guarantee that: 
\begin{itemize}
    \item Persistence of excitation during the adaptive control period, $T_{\Gcl} = T_{param} \left(\frac{2h + 2h \kappa_{K_x}\kappa_{K_y} + 2h(h-1)\kappa_{K_x}\kappa_{K_y}}{\bar{\sigma}_c}\right)^2$, where $T_{param} = T_{\G}\frac{20nh}{\sigma_n(\mathcal{H})}$ is the warm-up duration to get at least unit norm estimation error of ARX parameters at the end of the warm-up phase, in Section \ref{subsec:theo:persist_adaptive_optimism},
    \item Well-refined estimate of the Markov operator at the first epoch of adaptive control, $ T_{\bar{c}} = \frac{2048\Upsilon_c^4h^2 \left( \log\left(\frac{h(m+p)}{\delta}\right) + 2(m+p)mh^2\log\left(G_r + \frac{2}{\epsilon} \right) \right)}{\bar{\sigma}_c^4\sigma_e^4} $, where $\epsilon = \min \left\{1, \frac{\bar{\sigma}_c^2\sigma_e^2 }{16\left( h\Upsilon_c\kappa_{\nat}\sqrt{2}  + h \kappa_{\nat}^2 +  \sigma_e^2/2 \right)} \right \}$, in Section \ref{subsec:theo:persist_adaptive_optimism}.
\end{itemize}

Therefore, for the closed-loop model estimate variate of Algorithm \ref{algo_quad}, the warm-up duration $\tau$
\begin{equation}\label{warm_up_algo_adapt}
    \tau \geq \max \{h, T_{wp}, T_{N}, T_{u}, T_y, T_{\Gcl}, T_{\bar{c}} \}.
\end{equation}

\section{Details of Strongly Convex Cost Setting}
\label{apx:strong_convex_control}

\subsection{Counterfactual Input, Output, Loss}
Algorithm \ref{algo_strong} uses counterfactual reasoning first introduced in \citet{simchowitz2020improper} to update its DFC policy. Once the loss function $\ell_t$ is received, it constructs counterfactual inputs:
\begin{equation}
    \tilde{u}_{t-j}(\Mcontrol_t, \wh \G_i) = \sum\nolimits_{l=0}^{h'-1}M_t^{[l]}\nat_{t-j-l}(\wh \G),
\end{equation}
for $1 \leq j \leq h$. This gives what would be the inputs to the system with the current policy and current markov parameter estimates. Then, counterfactual output is computed,
\begin{equation}
    \tilde{y}_t(\Mcontrol_t, \wh \G_i) = \nature_t(\wh \G_i) + \sum\nolimits_{j=1}^h \hat{G}_i^{j} \tilde{u}_{t-j}(\Mcontrol_t, \wh \G_i)
\end{equation}
where $\hat{G}_i^j = \wh C (\wh A+\wh F \wh C)^j \wh B$ obtained my the Makrov parameter estimates. This is an estimation of output of the system if the counterfactual inputs have been applied by the agent. Finally, Algorithm \ref{algo_strong} computes the counterfactual loss using $\ell_t$:
\begin{equation}
    f_t(\Mcontrol_t, \wh \G_i) = \ell_t(\tilde{y}_t(\Mcontrol_t, \wh \G_i), \tilde{u}_{t}(\Mcontrol_t, \wh \G_i) ).
\end{equation}

Algorithm deploys online projected gradient descent on $f_t$ to improve the controller at each time step:
\[\Mcontrol_{t+1}=\proj_{\Mcontrolset_r}\left(\Mcontrol_t-\eta_t\nabla f_t\left(\Mcontrol_t,\wh{\mathcal{G}}_i\right)   \right)
\]

\subsection{Estimation and Boundedness Lemmas and Main Regret Results}

In this section, we will present the exact statement for Theorem \ref{thm:reg_s_cvx_exp_commit} and Theorem \ref{thm:reg_s_cvx_adapt}. Note that the proofs follow similar nature. Before stating these results, we state the following assumption on the choice of h, which simplifies the presentation and could be easily satisfied due to open-loop stability of the ARX system.

\begin{assumption} \label{asm:openloopdecay}
The length of the estimation of Markov parameters, $h$, is chosen such that for the horizon of T, we have
\begin{equation}
    \psi_{\Markov}(h) \coloneqq \max \{ \psi_{\Markov_{u\rightarrow y}}(h), \psi_{\Markov_{y\rightarrow y}}(h)  \} \leq 1/10T,
\end{equation}
where $\psi_\Markov $ is induced decay function on Markov operator $\Markov$, \textit{i.e.}, $\psi_\Markov(h) \coloneqq \sum_{i\geq h}\|{G}^{i}\|$.
\end{assumption}

Note that combining the choice of warm-up durations given in Appendix \ref{apx:warm_strong} and the results in Appendix \ref{subsec:warmuppersist}, we guarantee that open-loop data is persistently excited and the estimation error rate at the end of the warm-up phase is $\OO(1/\sqrt{T_{warm}})$. In the following, we show that with the choice of $T_{warm}$, the Markov operator estimates are well refined. First define $\alpha$, such that 
\begin{equation}
    \alpha \!\leq\! \strong \sigma_e^2 \left(1+\left(\frac{\sigma_{\min}\left(C\right)}{1+\|A+FC\|^2}\right)^2\right).
\end{equation}
Combining Theorem \ref{theo:closedloopid}, Lemma \ref{lem:logdet} and Lemma \ref{lem:boundednature}, we also have $\beta_t \leq \Bar{\beta}$ for all $t\geq T_{warm}$ where 
\[
\Bar{\beta} = \sqrt{mR\left( \log(1/\delta) + \frac{h(m+p)}{2} \log  \left(\frac{\lambda(m+p) + \tau (\kappa_u^2 + \kappa_y^2)}{\lambda(m+p) }\right)\right)} + S\sqrt{\lambda} +\frac{\sqrt{h}}{T}.
\]
Moreover, let 
\begin{equation*}
    T_{cx} = \frac{1024h^2\Bar{\beta}^2 \kappa_{\nat}^2 \kappa_\Mcontrolset^2 \kappa_\Markov^2 h'  \strong}{\alpha\sigma_c^2 \sigma_e^2}, \quad T_\Markov = \frac{256h^2 \Bar{\beta}^2 \kappa_\Mcontrolset^2 \kappa_\Markov^2   }{\sigma_c^2 \sigma_e^2},\quad T_r =  \frac{64h^2 \Bar{\beta}^2 \kappa_\psi^2   }{r^2\sigma_c^2 \sigma_e^2}.
\end{equation*}

In the following, we show that sum of Markov parameter estimation errors are well-bounded with the particular choice of warm-up time. The following lemma will be the key in proving the regret results.

\begin{lemma}[Sum of Markov Parameter Estimation Errors] \label{lem:concentrationgeneral} Let $\Delta \G \coloneqq \max \{ \sum_{j\geq 1} \| \wh G_{i, u\rightarrow y}^{j} - G_{u\rightarrow y}^{j} \|, \sum_{j\geq 1} \| \wh G_{i, y\rightarrow y}^{j} - G_{y\rightarrow y}^{j} \|\}$. For the chosen $T_{warm}$ duration, \textit{i.e.} (\ref{warm_up_algo_strong_exp}) or (\ref{warm_up_algo_strong_adapt}), we have
\begin{equation*}
     \Delta \G \leq  \epsilon_\Markov(i,\delta) \leq \min \left\{\frac{1}{4 \kappa_{b} \kappa_{\mathcal{M}} \kappa_{\mathbf{G}}} \sqrt{\frac{\alpha}{H^{\prime} \underline{\alpha}_{l o s s}}}, \frac{1}{2 \kappa_{\mathcal{M}} \kappa_ \mathbf{G}}, \frac{r}{\kappa_\psi}\right\}
\end{equation*}
with probability at least $1-4\delta$, where $\epsilon_\Markov(i,\delta) = \frac{8h\Bar{\beta}}{ \sigma_{c} \sigma_{e} \sqrt{2^{i-1}T_{\text{warm}}}} $.
\end{lemma}
\begin{proof}

\begin{align*}
    \Delta \G &\leq \max \left \{ \sum_{j\geq h+1} \| \wh G_{i, u\rightarrow y}^{j} - G_{u\rightarrow y}^{j} \|, \sum_{j\geq h+1} \| \wh G_{i, y\rightarrow y}^{j} - G_{y\rightarrow y}^{j} \| \right \} \\
    &\quad + h \max\{ \| \Markov_{\mathbf{u\rightarrow y}}(h) - \wh\Markov_{\mathbf{u\rightarrow y}}(h) \|, \| \Markov_{\mathbf{y\rightarrow y}}(h) - \wh\Markov_{\mathbf{y\rightarrow y}}(h) \| \} \\
    &\leq 1/10T + \frac{4h\beta}{ \sigma_{c} \sigma_{e} \sqrt{2^{i-1}T_{\text{warm}}}} \leq \epsilon_\Markov(i,\delta),
\end{align*}
proving the first inequality. The second inequality is numerical and follows from the choice of $T_{warm} \geq \max \{ T_{cx}, T_\Markov, T_r \}$.
\end{proof}

Finally, we show that the estimates of the output uncertainties ($\nat_t(\wh \G_i)$), the \DFC policy inputs ($u_t^{\Mcontrol_t}$) and the outputs of the ARX system ($y_t$) are bounded. To this end, for some $\delta \in (0,1)$, define $\kappa_{\nat} = R \sqrt{2m \log \frac{2mT}{\delta}}$ and $\kappa_{u_{b}} = \sigma_u \sqrt{2p \log \frac{2pT}{\delta}}$.

\begin{lemma}[Boundedness Lemma]\label{lem:boundednature} Let $\delta \in (0,1)$. For the chosen warm-up duration of Algorithm \ref{algo_strong} with explore and commit approach, \textit{i.e.} $T_{warm} = \Tw$, and for the chosen warm-up duration of Algorithm \ref{algo_strong} with closed-loop model estimate updates, we have the following bounds $\forall t$ with probability at least $1-2\delta$,
\begin{align*}
    \|\nat_t(\G)\| &\leq \kappa_{\nat}, \qquad &&\|u_t\| \leq \kappa_u \coloneqq 2\max \{ \kappa_{u_{b}},\kappa_\Mcontrolset \kappa_{\nat} \} \\
    \|y_t\| &\leq \kappa_y \coloneqq \kappa_\nature + \kappa_\Markov\kappa_u
    && \|\nat_t(\wh{\mathcal{G}})\| \leq 2\kappa_{\nat}.
\end{align*}
\end{lemma}
This lemma follows trivially via standard sub-Gaussian vector norm inequality and the second inequality in Lemma \ref{lem:concentrationgeneral}.

In the following precise statements of Theorem \ref{thm:reg_s_cvx_exp_commit} and \ref{thm:reg_s_cvx_adapt} are given and since their proofs differ only at one place, only the proof of Theorem \ref{thm:reg_s_cvx_adapt} is given with the explanation about the difference.

\begin{theorem}[Precise Statement of Theorem \ref{thm:reg_s_cvx_exp_commit}]\label{thm:apx_reg_s_explore}
Let $h'$ satisfy $h'\geq 3h\geq 1$, $\psi(\lfloor h'/2\rfloor-h)\leq \kappa_\Mcontrolset/T$ and $\psi_\Markov(h+1)\leq 1/10T$. If the loss function follows  (\ref{asm:lipschitzloss}) for the given ARX system, then the Algorithm \ref{algo_strong} with step size $\eta_t = \frac{12}{\alpha t}$ using explore and commit approach after a warm-up period ($T_{warm}$) of $\Tw$, given in (\ref{warm_up_algo_strong_exp}), has its regret bounded as 
\begin{align*}
    &\reg(T) \lesssim ~ T_{warm} L \kappa_y^2 ~+~ \frac{L^2{h'}^3\min\lbrace{m,p\rbrace}\kappa_{\nat}^4\kappa_{\Markov}^4\kappa_{\Mcontrolset}^2}{\min\lbrace \alpha, L\kappa_{\nat}^2\kappa_{\Markov}^2\rbrace}\bigg(\!1\!+\!\frac{\smooth}{\min\lbrace{m,p\rbrace} L\kappa_\Mcontrolset}\!\bigg)\log\bigg(\frac{T}{\delta}\bigg) \\
    &+ \frac{64(T-T_{warm})~h^2\Bar{\beta}^2h' \kappa_{\nat}^2 \kappa_{\Mcontrolset}^2 }{ T_{\text{warm}} \sigma_{c}^2 \sigma_{e}^2 }  \bigg( \frac{ \kappa_\Markov^2  \kappa_{\nat}^2 \left(\smooth + L\right)^2}{\alpha} + \kappa_y^2 (h+h')  \max \Big\{L, \frac{48L^2}{\alpha} \Big\} \bigg).
\end{align*}
with probability at least $1-5\delta$. The choice of $T_{warm}$ guarantees that the regret is $\OO(\sqrt{T})$ with probability at least $1-5\delta$, \textit{i.e.}, 
\begin{align*}
    \reg(T) &\lesssim  \frac{ (1+L) \Bar{\beta} \kappa_{\Mcontrolset} \kappa_{\nat} \kappa_y h \sqrt{h'}\sqrt{T}}{ \sigma_c \sigma_e} \sqrt{ \frac{ \kappa_\Markov^2  \kappa_{\nat}^2 \left(\smooth + L\right)^2}{\alpha} + \kappa_y^2 (h+h') \max \left\{L, \frac{48L^2}{\alpha} \right\} } \\
    &+ \frac{L^2{h'}^3\min\lbrace{m,p\rbrace}\kappa_{\nat}^4\kappa_{\Markov}^4\kappa_{\Mcontrolset}^2}{\min\lbrace \alpha, L\kappa_{\nat}^2\kappa_{\Markov}^2\rbrace}\bigg(\!1\!+\!\frac{\smooth}{\min\lbrace{m,p\rbrace} L\kappa_\Mcontrolset}\!\bigg)\log\bigg(\frac{T}{\delta}\bigg) 
\end{align*}

\end{theorem}

\begin{theorem}[Precise Statement of Theorem \ref{thm:reg_s_cvx_adapt}] \label{thm:apx_reg_s_adapt}
Let $h'$ satisfy $h'\geq 3h\geq 1$, $\psi(\lfloor h'/2\rfloor-h)\leq \kappa_\Mcontrolset/T$ and $\psi_\Markov(h+1)\leq 1/10T$. If the loss function follows  (\ref{asm:lipschitzloss}) for the given ARX system, then the Algorithm \ref{algo_strong} with access to persistently exciting $\Mcontrolset_r$, with step size $\eta_t = \frac{12}{\alpha t}$ using closed-loop model estimate updates via doubling epochs after a warm-up period ($T_{warm}$) of $\tau$, given in (\ref{warm_up_algo_strong_adapt}), has its regret bounded as 

\begin{align*}
    &\reg(T) \lesssim ~ T_{warm} L \kappa_y^2 ~+~ \frac{L^2{h'}^3\min\lbrace{m,p\rbrace}\kappa_{\nat}^4\kappa_{\Markov}^4\kappa_{\Mcontrolset}^2}{\min\lbrace \alpha, L\kappa_{\nat}^2\kappa_{\Markov}^2\rbrace}\bigg(\!1\!+\!\frac{\smooth}{\min\lbrace{m,p\rbrace} L\kappa_\Mcontrolset}\!\bigg)\log\bigg(\frac{T}{\delta}\bigg) \\
    &+ \sum_{t=T_{warm}+1}^{T}  \epsilon^2_\Markov \!\!\left( \Big \lceil \log_2 \big( \frac{t}{T_{warm}}\big) \Big \rceil, \delta \right)  h' \kappa_{\nat}^2 \kappa_{\Mcontrolset}^2 \bigg( \frac{ \kappa_\Markov^2  \kappa_{\nat}^2 \left(\smooth + L\right)^2}{\alpha} + \kappa_y^2 (h+h')  \max \Big\{L, \frac{48L^2}{\alpha} \Big\} \bigg).
\end{align*}
with probability at least $1-5\delta$. The choice of $T_{warm}$ guarantees that the regret is $O(polylog(T))$ with probability at least $1-5\delta$.
\end{theorem}
Notice that both theorems have same regret decomposition but the one with closed-loop estimates improve the estimation error during adaptive control which leads to significantly improved regret rate. The following gives the proof for both. 
\begin{proof}

The proof follows the regret decomposition of \citet{simchowitz2020improper}. We will first study the error in gradients on the counterfactual losses. Let $y_t^{pred}$ denote the prediction of output if the true system uncertainty and Markov parameters are known, \textit{i.e.} true counterfactual output of the system. Moreover, let $f_t^{pred}(M)$ denote the true counterfactual loss calculated by true counterfactual outputs and inputs as defined in Definition 8.1 of \citet{simchowitz2020improper}. They only have truncation error due to representation up to $h$. Using Lemma \ref{Lemma8.1}, we have that for any epoch $i$ and at any time step $t\in[t_i,~ \ldots,~t_{i+1}-1]$, the gradient $f_t^{pred}(M)$ is close to the gradient of the loss function of Algorithm \ref{algo_strong}:

\begin{equation} \label{gradientdif}
\left\|\nabla f_t\left(\Mcontrol,\wh \Markov_i,\nature_1(\wh \Markov_i),\ldots,\nature_t(\wh \Markov_i)\right)-\nabla f_{t}^{\text {pred }}(\Mcontrol)\right\|_{\mathrm{F}} \leq C_{\text {approx}} \epsilon_\Markov(i, \delta),
\end{equation}
where $C_{\text {approx }} \coloneqq \sqrt{h'} \kappa_\Markov \kappa_{\Mcontrolset} \kappa_{\nat}^2 \left(16\smooth +24 L\right)$. 

Pick a comparison controller $\Mcontrol_{comp} \in \Mcontrolset_r(h', \kappa_\Mcontrolset )$. For the competing set $\Mcontrolset(h_0', \kappa_\psi)$, we have the following regret decomposition: 
\begin{align} 
    \reg(T)& \leq \underbrace{ \left(\sum_{t=1}^{T_{warm}} \ell_{t}\left(y_{t}, u_{t}\right)\right) }_{\text{warm-up regret}} + \underbrace{\left(\sum_{t=T_{warm}+1}^{T} \ell_{t}\left(y_{t}, u_{t}\right)-\sum_{t=T_{warm}+1}^{T} F_{t}^{\text{pred}}\left[\Mcontrol_{t: t-h}\right]\right)}_{\text{algorithm truncation error}} \nonumber \\
    &+ \underbrace{\left(\sum_{t=T_{warm}+1}^{T} F_{t}^{\text{pred}}\left[\Mcontrol_{t: t-h}\right]-\sum_{t=T_{warm}+1}^{T} f_{t}^{\mathrm{pred}}\left(\Mcontrol_{comp}\right)\right)}_{f^{\mathrm{pred}} \text{ policy regret}} \nonumber \\
    &+ \underbrace{\left( \sum_{t=T_{warm}+1}^{T} f_{t}^{\mathrm{pred}}\left(\Mcontrol_{comp}\right)-\inf_{\Mcontrol \in \mathcal{M}} \sum_{t=T_{warm}+1}^{T} f_{t}\left(\Mcontrol, \G, \nat_1( \Markov),\ldots,\nat_t(\Markov) \right) \right)}_{\text{comparator approximation error}} \nonumber \\
    &+ \underbrace{\left( \inf_{\Mcontrol \in \mathcal{M}} \sum_{t=T_{warm}+1}^{T} f_{t}\left(\Mcontrol, \G, \nat_1( \Markov),\ldots,\nat_t(\Markov) \right) - \inf_{\Mcontrol \in \mathcal{M}} \sum_{t=T_{warm}+1}^{T} \ell_{t}\left(y^\Mcontrol_t, u^\Mcontrol_t\right) \right) }_{\text{comparator truncation error}}
\end{align}
Each term will be analyzed separately in the following.\\

\noindent \textbf{Warm-up Regret:} From (\ref{asm:lipschitzloss}) and Lemma \ref{lem:boundednature}, we get $\sum_{t=1}^{\Tburn} \ell_{t}\left(y_{t}, u_{t}\right) \leq \Tburn L \kappa_y^2 $. \\

\noindent \textbf{Algorithm Truncation Error:} From (\ref{asm:lipschitzloss}), we get 
\begin{align*}
    &\sum_{t=T_{warm}+1}^{T} \ell_{t}\left(y_{t}, u_{t}\right)-\sum_{t=T_{warm}+1}^{T} F_{t}^{\text{pred}}\left[\Mcontrol_{t: t-h}\right] \\
    &\leq \sum_{t=T_{warm}+1}^{T}\left|\ell_{t}\left(y_{t}, u_{t}\right)-\ell_{t}\left(\nat_t(\Markov) + \sum_{i=1}^h G_{u\rightarrow y}^{i} u_{t-i} + G_{y\rightarrow y}^{i} y_{t-i}, u_t\right)\right| \\
    &\leq\sum_{t=T_{warm}+1}^{T} L\kappa_y \left\|y_t - \nat_t(\Markov) + \sum_{i=1}^h G_{u\rightarrow y}^{i} u_{t-i} + G_{y\rightarrow y}^{i} y_{t-i}  \right \| \\
    &\leq \sum_{t=T_{warm}+1}^{T} L\kappa_y \left\| \sum_{i=h+1} G_{u\rightarrow y}^{i} u_{t-i} + G_{y\rightarrow y}^{i} y_{t-i} \right \| \\
    &\leq T L \kappa_y^2 \psi_\Markov(h+1)
\end{align*}
Since $\psi_\Markov(h+1)\leq 1/10T$, we get $\sum_{t=T_{warm}+1}^{T} \ell_{t}\left(y_{t}, u_{t}\right)-\sum_{t=T_{warm}+1}^{T} F_{t}^{\text{pred}}\left[\Mcontrol_{t: t-h}\right] \leq L \kappa_y^2 /10$.\\

\noindent \textbf{Comparator Truncation Error:} Similarly, we have the following bound
\begin{align*}
    \inf_{\Mcontrol \in \mathcal{M}} \sum_{t=T_{warm}+1}^{T} f_{t}\left(\Mcontrol, \G, \nat_1( \Markov),\ldots,\nat_t(\Markov) \right) - \inf_{\Mcontrol \in \mathcal{M}} \sum_{t=T_{warm}+1}^{T} \ell_{t}\left(y^\Mcontrol_t, u^\Mcontrol_t\right) 
    &\leq T L\kappa_\Markov \kappa_\Mcontrolset^2 \kappa_{\nat}^2 \psi_\Markov(h+1) \\
    &\leq L\kappa_\Markov \kappa_\Mcontrolset^2 \kappa_{\nat}^2 / 10
\end{align*}

\noindent $\mathbf{f^{\mathrm{pred}} \textbf{ Policy Regret}:}$ In order to bound this term, we will use Theorem \ref{theo:fpred}. However, Theorem \ref{theo:fpred} requires several strong convexity, Lipschitzness and smoothness properties as stated in the theorem. First recall Lemma \ref{lem:concentrationgeneral}, which guarantees that warm-up period is long enough to get well-refined Markov parameter estimates. Utilizing this closeness and trivially tweaking Lemmas \ref{lem:smooth}-\ref{lem:lipschitz} to ARX setting provide the required conditions. 

Finally, using (\ref{gradientdif}), we get the following adaptation of Theorem \ref{theo:fpred}:
\begin{lemma}
For step size $\eta = \frac{12}{\alpha t}$, the following bound holds with probability $1-\delta$:
\begin{align*}
   &\sum_{t=T_{warm}+1}^{T} F_{t}^{\text{pred}}\left[\Mcontrol_{t: t-h}\right]-\sum_{t=T_{warm}+1}^{T} f_{t}^{\mathrm{pred}}\left(\Mcontrol_{comp}\right) + \frac{\alpha}{48} \sum_{t = T_{warm} + 1}^T \!\!\!\!\!\!\!\|\Mcontrol_t - \Mcontrol_{comp}\|_F^2 \\
   &\lesssim \frac{L^2{h'}^3\min\lbrace{m,p\rbrace}\kappa_{\nat}^4\kappa_{\Markov}^4\kappa_{\Mcontrolset}^2}{\min\lbrace \alpha, L\kappa_{\nat}^2\kappa_{\Markov}^2\rbrace}\left(\!1\!+\!\frac{\smooth}{\min\lbrace{m,p\rbrace} L\kappa_\Mcontrolset}\!\right)\log\left(\frac{T}{\delta}\right) + \frac{1}{\alpha} \! \sum_{t=T_{warm}+1}^T \!\!\!\!\!\!\! C_{\text{approx}}^2 \epsilon^2_\Markov \!\!\left( \left \lceil \log_2 \left( \frac{t}{T_{warm}}\right) \right \rceil, \delta \right)
\end{align*}
\end{lemma}
\begin{proof}
The right hand side of Theorem \ref{theo:fpred} can be bounded via following proof steps of Theorem 4 of \citet{simchowitz2020improper}: 
\begin{align}
    &\sum_{t=T_{warm}+1}^{T} F_{t}^{\text{pred}}\left[\Mcontrol_{t: t-h}\right]-\sum_{t=T_{warm}+1}^{T} f_{t}^{\mathrm{pred}}\left(\Mcontrol_{comp}\right) \!\!-\!\! \left(\frac{6}{\alpha}\!\! \sum_{t=k+1}^{T}\!\!\left\|\boldsymbol{\epsilon}_{t}\right\|_{2}^{2}-\frac{\alpha}{48} \sum_{t=1}^{T}\left\|\Mcontrol_t - \Mcontrol_{comp} \right\|_{F}^{2}\right) \\
    &\qquad \qquad \qquad \qquad \qquad \lesssim \frac{L^2{h'}^3\min\lbrace{m,p\rbrace}\kappa_{\nat}^4\kappa_{\Markov}^4\kappa_{\Mcontrolset}^2}{\min\lbrace \alpha, L\kappa_{\nat}^2\kappa_{\Markov}^2\rbrace}\left(\!1\!+\!\frac{\smooth}{\min\lbrace{m,p\rbrace} L\kappa_\Mcontrolset}\!\right)\log\left(\frac{T}{\delta}\right) \nonumber 
\end{align}
\begin{align}
    &\mathbf{f^{\mathrm{pred}} \textbf{p.r.} }  + \frac{\alpha}{48} \sum_{t=1}^{T}\left\|\Mcontrol_t - \Mcontrol_{comp} \right\|_{F}^{2} \lesssim \frac{L^2{h'}^3\min\lbrace{m,p\rbrace}\kappa_{\nat}^4\kappa_{\Markov}^4\kappa_{\Mcontrolset}^2}{\min\lbrace \alpha, L\kappa_{\nat}^2\kappa_{\Markov}^2\rbrace}\left(\!1\!+\!\frac{\smooth}{\min\lbrace{m,p\rbrace} L\kappa_\Mcontrolset}\!\right)\log\left(\frac{T}{\delta}\right)  \nonumber \\
    &\qquad\qquad\qquad\qquad\qquad\qquad\qquad\qquad +\frac{1}{\alpha} \! \sum_{t=T_{warm}+1}^T \!\!\!\!\!\!\! C_{\text{approx}}^2 \epsilon^2_\Markov \!\!\left( \left \lceil \log_2 \left( \frac{t}{T_{warm}}\right) \right \rceil, \delta \right), \label{usegradientdiff} 
\end{align}
where the last line follows from (\ref{gradientdif}). 
\end{proof}
\newpage
\noindent \textbf{Comparator Approximation Error:} We can bound this using Lemma G.2 of \citet{lale2020logarithmic}:
\begin{lemma} 
Suppose that $h' \geq 2h_{0}'-1+h$. Then for all $\gamma>0$,
\begin{align*}
&\sum_{t=T_{warm}+1}^{T} f_{t}^{\mathrm{pred}}\left(\Mcontrol_{comp}\right)-\inf_{\Mcontrol \in \mathcal{M}} \sum_{t=T_{warm}+1}^{T} f_{t}\left(\Mcontrol, \G, \nat_1( \Markov),\ldots,\nat_t(\Markov) \right) \leq 4L \kappa_y \kappa_u \kappa_\Mcontrolset \\
&\qquad +\!\!\!\!\! \sum_{t=T_{warm}+1}^{T} \left[ \gamma \left\|\Mcontrol_{t} \!-\! \Mcontrol_{\mathrm{comp}}\right\|_{F}^{2} + 8\kappa_y^2 \kappa_{\nat}^2 \kappa_{\Mcontrolset}^2 (h+h') \max \left\{L, \frac{L^2}{\gamma} \right\}\epsilon^2_\Markov\left( \left \lceil \log_2 \left( \frac{t}{\Tburn}\right) \right \rceil, \delta \right) \right]
\end{align*}
\end{lemma}
\vspace{3em}
Combining all the terms bounded above, with $\tau = \frac{\alpha}{48}$ gives 
\begin{align*}
    &\reg(T) \\
    &\lesssim \Tburn L \kappa_y^2 + L \kappa_y^2/10 + L\kappa_\Markov \kappa_\Mcontrolset^2 \kappa_{\nat}^2 / 10 + 2 L \kappa_\Mcontrolset \kappa_y \kappa_\Markov \kappa_\nature + 4L \kappa_y \kappa_u \kappa_\Mcontrolset \\
    & \frac{L^2{h'}^3\min\lbrace{m,p\rbrace}\kappa_{\nat}^4\kappa_{\Markov}^4\kappa_{\Mcontrolset}^2}{\min\lbrace \alpha, L\kappa_{\nat}^2\kappa_{\Markov}^2\rbrace}\bigg(\!1\!+\!\frac{\smooth}{\min\lbrace{m,p\rbrace} L\kappa_\Mcontrolset}\!\bigg)\log\bigg(\frac{T}{\delta}\bigg) + \frac{1}{\alpha} \! \sum_{t=T_{warm}+1}^T \!\!\!\!\!\!\! C_{\text{approx}}^2 \epsilon^2_\Markov \!\!\left( \left \lceil \log_2 \left( \frac{t}{T_{warm}}\right) \right \rceil, \delta \right) \\
    &+ \sum_{t=T_{warm}+1}^{T} 8\kappa_y^2 \kappa_{\nat}^2 \kappa_{\Mcontrolset}^2 (h+h') \max \left\{L, \frac{48L^2}{\alpha} \right\}\epsilon^2_\Markov\left( \left \lceil \log_2 \left( \frac{t}{T_{warm}}\right) \right \rceil, \delta \right)  \\
    &\lesssim ~ T_{warm} L \kappa_y^2 ~+~ \frac{L^2{h'}^3\min\lbrace{m,p\rbrace}\kappa_{\nat}^4\kappa_{\Markov}^4\kappa_{\Mcontrolset}^2}{\min\lbrace \alpha, L\kappa_{\nat}^2\kappa_{\Markov}^2\rbrace}\bigg(\!1\!+\!\frac{\smooth}{\min\lbrace{m,p\rbrace} L\kappa_\Mcontrolset}\!\bigg)\log\bigg(\frac{T}{\delta}\bigg) \\
    &+ \sum_{t=T_{warm}+1}^{T}  \epsilon^2_\Markov \!\!\left( \Big \lceil \log_2 \big( \frac{t}{T_{warm}}\big) \Big \rceil, \delta \right)  h' \kappa_{\nat}^2 \kappa_{\Mcontrolset}^2 \bigg( \frac{ \kappa_\Markov^2  \kappa_{\nat}^2 \left(\smooth + L\right)^2}{\alpha} + \kappa_y^2 (h+h')  \max \Big\{L, \frac{48L^2}{\alpha} \Big\} \bigg)
\end{align*}

Note that for the explore and commit variant of Algorithm \ref{algo_strong}, $\epsilon^2_\Markov \!\!\left( \Big \lceil \log_2 \big( \frac{t}{T_{warm}}\big) \Big \rceil, \delta \right)$ should be replaced with $\epsilon^2_\Markov \!\!\left( 1 , \delta \right)$ since the Markov parameter estimates are not updated during the adaptive control. With the choice of $T_{warm} = \Tw$, in particular due to $T_{reg}$, we get the advertised bound on Theorem \ref{thm:apx_reg_s_explore}. However, for Algorithm \ref{algo_strong} with closed-loop model estimate updates we have doubling epoch updates. Using Lemma \ref{lem:concentrationgeneral}, we have that at any time step $t$ during the $i$'th epoch, \textit{i.e.}, $t\in[ t_i,\ldots, t_{i}-1]$, $\epsilon^2_\Markov(i, \delta) = \OO(polylog(T)/2^{i-1} T_{warm})$. Thus we get
\begin{align*}
    \sum_{t=T_{warm}+1}^{T}  \epsilon^2_\Markov \!\!\left( \Big \lceil \log_2 \big( \frac{t}{T_{warm}}\big) \Big \rceil, \delta \right) =\sum_{i=1}^{ \log (T)} 2^{i-1} T_{warm} \epsilon^2_\Markov \left( i, \delta \right) \leq \OO \left(polylog(T)\right)
\end{align*}
where the first equality is due to at most $\log (T)$ updates during adaptive control phase. Since all terms in the regret decomposition is $\OO \left(polylog(T)\right)$, we obtain the advertised regret upper bound in Theorem \ref{thm:apx_reg_s_adapt}.
\end{proof}

\section{Details of Convex Quadratic Cost Setting}
\label{apx:quad_control}

\subsection{Details of Parameter estimation, \Sys($h,\wh{\mathcal{G}}_i,n$)}
\label{apx:sysidarx}

Once the Markov parameters ($\wh \G_t$) are estimated, Algorithm \ref{algo_quad} constructs confidence sets for the unknown ARX model parameters and chooses an optimistic controller among these confidence sets. Algorithm \ref{algo_quad} uses \Sys to recover model parameters. \Sys is similar to SYS-ID of \citet{lale2020logarithmic} and internally follows a method similar to Ho-Kalman method~\citep{ho1966effective}, \Sys is given in Algorithm \ref{SYSID}.

\begin{algorithm}[tbh] 
 \caption{\Sys}
  \begin{algorithmic}[1] 
  \STATE {\bfseries Input:} $\mathbf{\G_{t}}$, $h$, $n$, $d_1, d_2$ such that $d_1 + d_2 + 1 = h$ \\
  \STATE Form two $d_1 \times (d_2+1)$ Hankel matrices $\mathbf{\mathcal{H}_{\wh \Markov_{\mathbf{y\rightarrow y}}}}$ and $\mathbf{\mathcal{H}_{\wh \Markov_{\mathbf{u\rightarrow y}}}}$  from $\mathbf{\G_{t}}$ and construct $\hat{\mathcal{H}}_t = \left[ \mathbf{\mathcal{H}_{\wh \Markov_{\mathbf{y\rightarrow y}}}}, \enskip \mathbf{\mathcal{H}_{\wh \Markov_{\mathbf{u\rightarrow y}}}} \right] \in \R^{md_1 \times (m+p)(d_2+1)}$
  \STATE Obtain $\hat{\mathcal{H}}_t^-$ by discarding $(d_2 + 1)$th and $(2d_2 + 2)$th block columns of $\hat{\mathcal{H}}_t$
  \STATE Using SVD obtain the best rank-$n$ approximation of $\hat{\mathcal{H}}_t^-$ denoted as $\hat{\mathcal{N}}_t \in \R^{m d_1 \times (m+p) d_2}$, 
  \STATE Obtain  $\mathbf{U_t},\mathbf{\Sigma_t},\mathbf{V_t} = \text{SVD}(\hat{\mathcal{N}}_t)$\STATE Construct $\mathbf{\hat{O}_t} = \mathbf{U_t}\mathbf{\Sigma_t}^{1/2} \in \R^{md_1 \times n}$
  \STATE Construct $[\mathbf{\hat{C}_{F_t}}, \enskip \mathbf{\hat{C}_{B_t}}] = \mathbf{\Sigma_t}^{1/2}\mathbf{V_t} \in \R^{n \times (m+p)d_2}$
  \STATE Obtain $\hat{C}_t\in \R^{m\times n}$, the first $m$ rows of $\mathbf{\hat{O}_t}$
  \STATE Obtain $\hat{B}_t\in \R^{n\times p}$, the first $p$ columns of $\mathbf{\hat{C}_{B_t}}$
  \STATE Obtain $\hat{F}_t\in \R^{n\times m}$, the first $m$ columns of $\mathbf{\hat{C}_{F_t}}$
  \STATE Obtain $\hat{\mathcal{H}}_t^+$ by discarding $1$st and $(d_2+2)$th block columns of $\hat{\mathcal{H}}_t$
  \STATE Obtain $\hat{A}_t = \mathbf{\hat{O}_t}^\dagger \enskip \hat{\mathcal{H}}_t^+ \enskip [\mathbf{\hat{C}_{F_t}}, \quad \mathbf{\hat{C}_{B_t}}]^\dagger$
  \end{algorithmic}
 \label{SYSID}  
\end{algorithm}

Let $T_N = T_{\G} \frac{8h}{\sigma_n^2(\mathcal{H})}$ where $T_{\G}$ is the amount of samples required to get less than unit norm estimation error on Markov parameter estimates. With the choice of $T_{warm} \geq T_N$, we get $\sigma_{\min}(\mathcal{N}) \geq 2\|\mathcal{N}-\hat{\mathcal{N}}\|$ where $\mathcal{N}$ is the rank-n approximation of $\mathcal{H}$, \textit{i.e.} Hankel matrix formed via true underlying ARX parameters. Combining this with Lemma C.2 of \citet{lale2020root}, which is a slightly modified version of Lemma B.1 of \citet{oymak2018non}, we get the guarantee that there exists a unitary transform $\mathbf{T}$ such that 
\begin{equation} \label{boundonSVD}
    \left\|\mathbf{\hat{O}_t} - \mathbf{O}\mathbf{T}  \right\|_F^2 + \left\|[\mathbf{\hat{C}_{F_t}} \enskip \mathbf{\hat{C}_{B_t}}] -  \mathbf{T}^\top  [\mathbf{C_F} \enskip \mathbf{C_B}] \right\|_F^2 \leq \frac{10n \| \mathcal{N} - \hat{\mathcal{N}}_t\|^2}{\sigma_n(\mathcal{N}) }
\end{equation}

Using Lemma 5.2 of \citet{oymak2018non}, we get 
\begin{equation*}
    \| \mathcal{N} - \hat{\mathcal{N}}_t\| \leq \sqrt{2h}\|\mathbf{\wh \G_{t} } - \G\|.
\end{equation*}

Note that $\hat{C}_t - C\mathbf{T}$ is a submatrix of $\mathbf{\hat{O}_t} - \mathbf{O}\mathbf{T}$, $\hat{B}_t - \mathbf{T}^\top B$ is a submatrix of $\mathbf{\hat{C}_{B_t}} - \mathbf{T}^\top \mathbf{C_{B}}$ and $\hat{F}_t - \mathbf{T}^\top F$ is a submatrix of $\mathbf{\hat{C}_{F_t}} - \mathbf{T}^\top \mathbf{C_{F}}$. Thus, we get 

\begin{align*}
    \|\hat{B}_t - \mathbf{T}^\top B \|, \|\hat{C}_t - C\mathbf{T}\|, \|\hat{F}_t - \mathbf{T}^\top F\|  \leq \frac{\sqrt{20nh} \|\mathbf{\hat{\G}_{t} } - \mathbf{\G}\| }{\sqrt{\sigma_n(\mathcal{N})}}
\end{align*}

Following similar analysis with \citet{lale2020root}, we obtain 
\begin{align*}
    \|\hat{A}_t - \mathbf{T}^\top A \mathbf{T} \|_F 
    \leq \frac{31\sqrt{2nh}  \|\mathcal{H}^+ \|  }{2\sigma_n^{2}(\mathcal{N})} \|\mathbf{\hat{\G}_{t} } - \mathbf{\G}\| +  \frac{13 \sqrt{nh}}{2\sqrt{2}\sigma_n(\mathcal{N})} \|\mathbf{\hat{\G}_{t} } - \mathbf{\G}\|
\end{align*}
Combining these with Theorem \ref{theo:closedloopid} give the confidence sets required for Algorithm \ref{algo_quad}: $\mathcal{C}_i \coloneqq \{\mathcal{C}_A(i), \mathcal{C}_B(i), \mathcal{C}_C(i),$ $\mathcal{C}_F(i)\}$. Note that even though the estimated system parameters are recovered up to a similarity transformation, the cost ($J(\cdot)$) achieved by set of parameters with the same similarity transformation is fixed, allowing us the search for optimistic cost, \textit{i.e.} 
\begin{equation*}
    J(\ttt_i) \leq \inf_{\Theta' \in \mathcal{C}_i \cap \mathcal{S}} J(\Theta') + T^{-1}
\end{equation*}

Moreover, following the iterative closed-loop dynamics approach developed with Lemma 4.1 of \citet{lale2020root} or Lemma 4.2 of \citet{lale2020regret}, one can show that if $T_{warm} \geq T_u$, then 
\begin{enumerate}
    \item $\Theta \in (\mathcal{C}_A(i) \times \mathcal{C}_B(i) \times \mathcal{C}_C(i) \times \mathcal{C}_F(i)) $
    \item $\| x_t\| \leq \tilde{\mathcal{X}}$,  $\| y_t \| \leq \tilde{\mathcal{Y}}$
    \item $\|u_t \| \leq \kappa_{K_x} \tilde{\mathcal{X}} + \kappa_{K_y} \tilde{\mathcal{Y}} $
\end{enumerate}
where $\tilde{\mathcal{X}}, \tilde{\mathcal{Y}} = \OO(\sqrt{\log(T/\delta)})$, with probability at least $1-3\delta$. Note that this shows the ARX systems are estimated accurate enough that the system stays stable throughout the exploitation or adaptive control of Algorithm \ref{algo_quad}.

In the following precise statements of Theorem \ref{thm:reg_cvx_exp_commit} and \ref{thm:reg_cvx_adapt} are given and since their proofs differ only at one place, only the proof of Theorem \ref{thm:reg_cvx_exp_commit} is given with the explanation about the difference.

\begin{theorem}[Precise Statement of Theorem \ref{thm:reg_cvx_exp_commit}]\label{thm:apx_reg_explore}
Let $\delta \in (0,1)$. Given an unknown ARX system $\Theta = (A,B,C,F)$, and regulating parameters $Q\succeq 0$ and $R\succ 0$, if Algorithm \ref{algo_quad} with explore and commit approach with warm-up duration of $T_{warm} \geq \max \{h, T_{wp}, T_{N}, T_{u}, T_y\} $ interacts with the system for $T$ time steps in total such that $T>T_{warm}$, with probability at least 
$1- 5\delta$, the regret of Algorithm \ref{algo_quad} is bounded as follows,
\begin{equation}
    \reg(T) = \tilde{\OO}\left( T_{warm} + \frac{T-T_{warm}}{\sqrt{T_{warm}}} \right).
\end{equation}
The choice of $T_{warm} = T^{2/3}$, \textit{i.e.} $T_{warm} \geq \Tw$ in (\ref{warm_up_algo_exp}), guarantees that the regret is $\tilde{\OO}(T^{2/3})$ with probability at least $1-5\delta$.
\end{theorem}

\begin{theorem}[Precise Statement of Theorem \ref{thm:reg_cvx_adapt}]\label{thm:apx_reg_adapt}
Let $\delta \in (0,1)$. Given an unknown ARX system $\Theta = (A,B,C,F)$, and regulating parameters $Q\succeq 0$ and $R\succ 0$ such that the optimal controller for $\Theta$ persistently excites the system, as given in Section \ref{subsec:persistence_arx_opt}, if Algorithm \ref{algo_quad} with closed-loop model estimate updates with warm-up duration of $T_{warm} \geq \tau $ as given in (\ref{warm_up_algo_adapt}) interacts with the system for $T$ time steps in total such that $T>T_{warm}$, with probability at least 
$1- 5\delta$, the regret of Algorithm \ref{algo_quad} is bounded as follows,
\begin{equation}
    \reg(T) = \tilde{\OO}\left(\sqrt{T} \right).
\end{equation}
\end{theorem}
\begin{proof}

For an average cost per stage problem in infinite state and control space like the given control system $\Theta = \left(A,B,C,F\right)$ with regulating parameters $Q$ and $R$, using the optimal average cost per stage $J_\star(\Theta)$ and guessing the correct differential(relative) cost, where $(A+FC,B)$ is controllable, $(A,C)$ is  observable, $Q$ is positive semidefinite and $R$ is positive definite, one can verify that they satisfy Bellman optimality equation~\citep{bertsekas1995dynamic}. The lemma below shows the Bellman optimality equation the given system $\Theta$, which will be critical in regret analysis.
\begin{lemma} [Bellman Optimality Equation for ARX System] \label{lem:bellman}
Given state $x_{t} \in \R^{n}$ and an observation $y_t \in \R^{m}$ pair at time $t$, Bellman optimality equation of average cost per stage control of the system $\Theta = (A,B,C,F)$ with regulating parameters $Q$ and $R$ is 
\begin{align}
    J_\star(\Theta) &+ (A x_t + F y_t)^\top  \left(\mathbf{P} - \mathbf{P} B (R + B^\top \mathbf{P} B)^{-1} B^\top \mathbf{P} \right) (A x_t + F y_t) + y_t^\top Q y_t \nonumber \\ &= y_t^\top Q y_t + u_t^\top R u_t
    +\mathbb{E} \bigg[ (A x_{t+1} + F y_{t+1})^\top \left(\mathbf{P} - \mathbf{P} B (R + B^\top \mathbf{P} B)^{-1} B^\top \mathbf{P} \right) (A x_{t+1} + F y_{t+1}) \nonumber \\
    &\qquad \qquad \qquad \qquad\qquad \qquad+ y_{t+1}^\top Q y_{t+1}\bigg]
\end{align}
\end{lemma}

We give the proof and the expression for $J_\star(\Theta)$ in Section \ref{apx:bellman}. Using Bellman optimality equation for the optimistic system, $\tilde{\Theta}$, we derive a regret decomposition for applying the optimal policy of $\tilde{\Theta}$ on $\Theta$. Note that similarity transformation in recovering the ARX system parameters do not affect the regret decomposition. Thus, without loss of generality, we consider that the similarity transformation is identity.

Let $M = (R + B^\top \mathbf{P} B)^{-1}B^\top \mathbf{P}$ and $\Tilde{M} = (R + \tb^\top \mathbf{\Tilde{P}} \tb)^{-1}\tb^\top \mathbf{\Tilde{P}}$ where $\mathbf{P}$ and $\tp$ are solutions for the corresponding DARE of $\Theta$ and $\tth$ respectively. For given $x_{t}$ and $y_t$, let 
\begin{align*}
    \bar{x}_{t,\tth} &= \ta x_t + \tf y_t \\
    \bar{x}_{t,\Theta} &= A x_t + F y_t
\end{align*}
and the optimal control for $\tth$ and $\Theta$ are defined as $u_{\star, \tth} = -\Tilde{M} \bar{x}_{t,\tth}$ and $u_{\star, \Theta} = -M \bar{x}_{t,\Theta}$. Define the following expressions for time step $t+1$ using the model specified as subscript,

\begin{align}
    x_{t+1,\tth} &= \bar{x}_{t,\tth} - \tb \tm \bar{x}_{t,\tth} = \left( \ta - \tb \tm \ta \right)  x_{t} + \left( \tf - \tb \tm \tf \right)  y_t \label{firstdef} \\
    y_{t+1,\tth} &= \tc (I -\tb \tm) \bar{x}_{t,\tth} + e_{t+1} =  \tc \left( \ta - \tb \tm \ta \right)  x_{t} + \tc \left( \tf - \tb \tm \tf \right)  y_t + e_{t+1} \label{seconddef}\\
    x_{t+1,\Theta} &= \bar{x}_{t,\Theta} - B \tm \bar{x}_{t,\tth} = \left( A - B \tm \ta \right) x_t + \left( F - B \tm \tf \right)  y_t \label{thirddef}\\
    y_{t+1, \Theta} &= C\bar{x}_{t,\Theta} - CB \tm \bar{x}_{t,\tth} + e_{t+1} = C\left( A - B \tm \ta \right) x_t + C \left( F - B \tm \tf \right)  y_t + e_{t+1} \label{lastdef} 
    % \\
    % \hat{x}_{t+1|t+1,\Theta} &= (I -LC)(A\hat{x}_{t|t,\Theta} - B\tk\hat{x}_{t|t,\tth}) + L y_{t+1, \Theta} \\
    % &= (I-LC)(A-B\tk)\hat{x}_{t|t,\tth} + (I-LC)A(\hat{x}_{t|t,\Theta} - \hat{x}_{t|t,\tth}) + L y_{t+1, \Theta} \\
    % &= (I-LC)(A-B\tk)\hat{x}_{t|t,\tth} + LC(A-B\tk)\hat{x}_{t|t,\tth} + LC w_t + LCA(x_t - \hat{x}_{t|t,\tth}) \nonumber \\
    % \qquad \qquad & \qquad + (I-LC)A(\hat{x}_{t|t,\Theta} - \hat{x}_{t|t,\tth}) + Lz_{t+1} \\
    % &= (A\!-\!B\tk)\hat{x}_{t|t,\tth} \!+\! LCw_t \!+\! LCA(x_t \!-\! \hat{x}_{t|t,\tth}) \!+\! (I\!-\!LC)A(\hat{x}_{t|t,\Theta} \!-\! \hat{x}_{t|t,\tth}) \!+\! Lz_{t+1} \label{lastdef}
\end{align}
Let $\tpb = \tp - \tp \tb (R + \tb^\top \tp \tb)^{-1} \tb^\top \tp $. Using the defined quantities above, writing the Bellman optimality equation for the optimistic model we get,
\begin{align}
    &J_\star(\tth) + \bar{x}_{t,\tth}^\top  \tpb \bar{x}_{t,\tth} + y_t^\top Q y_t \nonumber\\
    &= y_t^\top Q y_t + u_t^\top R u_t + \mathbb{E} \left[ y_{t+1,\tth}^\top Q y_{t+1,\tth} | x_t, y_t \right] \nonumber \\ &
    +\mathbb{E} \left[ ( (\ta+\tf \tc) x_{t+1, \tth} + \tf e_{t+1})^\top \tpb ((\ta+\tf \tc) x_{t+1, \tth} + \tf e_{t+1}) | x_t, y_t \right] \nonumber \\
    &= y_t^\top Q y_t + u_t^\top R u_t +  \bar{x}_{t,\tth}^\top (I-\tb \tm)^\top \tc^\top Q \tc (I-\tb \tm) \bar{x}_{t,\tth} + \mathbb{E}\left[ e_{t+1}^\top Q e_{t+1}^\top \right] \label{inserting} \\
    &+ \bar{x}_{t,\tth}^\top (I-\tb \tm)^\top (\ta + \tf \tc)^\top \tpb (\ta + \tf \tc) (I-\tb \tm) \bar{x}_{t,\tth} + \mathbb{E}\left[ e_{t+1}^\top \tf^\top \tpb \tf  e_{t+1}^\top \right] \nonumber
\end{align}

From the definitions given in (\ref{firstdef})-(\ref{lastdef}), we have the following equalities,

\begin{align*}
    \mathbb{E}\left[ e_{t+1}^\top Q e_{t+1}^\top \right] = \mathbb{E}\left[ y_{t+1,\Theta}^\top Q y_{t+1,\Theta} \right] - x_{t+1,\Theta}^\top C^\top Q C  x_{t+1,\Theta}
\end{align*}
\begin{align*}
    &\mathbb{E}\left[ e_{t+1}^\top (\tf+F-F)^\top \tpb (\tf+F-F)  e_{t+1}^\top \right] \\
    &=  \mathbb{E}\left[ e_{t+1}^\top F^\top \tpb F e_{t+1} \right] + 2 \mathbb{E}\left[ e_{t+1}^\top F^\top \tpb (\tf-F)  e_{t+1} \right] +  \mathbb{E}\left[ e_{t+1}^\top (\tf-F)^\top \tpb (\tf-F)  e_{t+1} \right]\\
    &= \mathbb{E} \left[ (A x_{t+1,\Theta} + F y_{t+1, \Theta})^\top \tpb (A x_{t+1,\Theta} + F y_{t+1, \Theta})\right] - x_{t+1,\Theta}^\top (A+FC)^\top \tpb (A+FC) x_{t+1,\Theta} \\
    &+2 \mathbb{E}\left[ e_{t+1}^\top F^\top \tpb (\tf-F)  e_{t+1} \right] +  \mathbb{E}\left[ e_{t+1}^\top (\tf-F)^\top \tpb (\tf-F)  e_{t+1} \right]
\end{align*}
Inserting these to (\ref{inserting}), we get
\begin{align*}
    &J_\star(\tth) + \bar{x}_{t,\tth}^\top  \tpb \bar{x}_{t,\tth} + y_t^\top Q y_t \\
    &= y_t^\top Q y_t + u_t^\top R u_t +  \bar{x}_{t,\tth}^\top (I-\tb \tm)^\top \tc^\top Q \tc (I-\tb \tm) \bar{x}_{t,\tth} +\mathbb{E}\left[ y_{t+1,\Theta}^\top Q y_{t+1,\Theta} \right] - x_{t+1,\Theta}^\top C^\top Q C  x_{t+1,\Theta} \\
    &+ \bar{x}_{t,\tth}^\top (I-\tb \tm)^\top (\ta + \tf \tc)^\top \tpb (\ta + \tf \tc) (I-\tb \tm) \bar{x}_{t,\tth} + \mathbb{E} \left[ (A x_{t+1,\Theta} + F y_{t+1, \Theta})^\top \tpb (A x_{t+1,\Theta} + F y_{t+1, \Theta})\right] \\
    &- x_{t+1,\Theta}^\top (A+FC)^\top \tpb (A+FC) x_{t+1,\Theta} +2 \mathbb{E}\left[ e_{t+1}^\top F^\top \tpb (\tf-F)  e_{t+1} \right] +  \mathbb{E}\left[ e_{t+1}^\top (\tf-F)^\top \tpb (\tf-F)  e_{t+1} \right]
\end{align*}
Hence,
\begin{equation*}
    \sum_{t=0}^{T-T_{warm}}  J_*(\tth) \!+\! R_1 \!+\! R_2 = \sum_{t=0}^{T-T_{warm}} \left( y_t^\top Q y_t \!+\! u_t^\top R u_t \right) \!+\! R_3 \!+\! R_4 \!+\! R_5
\end{equation*}
where 
\begin{align*}
    R_1& \!=\!\! \sum_{t=0}^{T-T_{warm}} \left \{ (\ta x_t + \tf y_t)^\top \tpb (\ta x_t + \tf y_t) - \mathbb{E} \left[ (A x_{t+1,\Theta} + F y_{t+1, \Theta})^\top \tpb (A x_{t+1,\Theta} + F y_{t+1, \Theta}) | x_t, y_t, u_t \right]  \right \} \\
    R_2& \!=\!\! \sum_{t=0}^{T-T_{warm}} \left \{ y_t^\top Q y_t - 
    \mathbb{E}\left[ y_{t+1,\Theta}^\top Q y_{t+1,\Theta} \Big | x_t, y_t, u_t \right] \right \} \\
    R_3& \!=\!\! \sum_{t=0}^{T-T_{warm}} \left \{ \bar{x}_{t,\tth}^\top (I-\tb \tm)^\top \tc^\top Q \tc (I-\tb \tm) \bar{x}_{t,\tth} - x_{t+1,\Theta}^\top C^\top Q C  x_{t+1,\Theta}  \right \} \\
    R_4& \!=\!\! \sum_{t=0}^{T-T_{warm}} \! \left\{ \bar{x}_{t,\tth}^\top (I-\tb \tm)^\top (\ta + \tf \tc)^\top \tpb (\ta + \tf \tc) (I-\tb \tm) \bar{x}_{t,\tth} - x_{t+1,\Theta}^\top (A+FC)^\top \tpb (A+FC) x_{t+1,\Theta}
    \right \} \\
    R_{5}& \!=\!\! \sum_{t=0}^{T-T_{warm}} \bigg \{ 2 \mathbb{E}\left[ e_{t+1}^\top F^\top \tpb (\tf-F)  e_{t+1}^\top \right] +  \mathbb{E}\left[ e_{t+1}^\top (\tf-F)^\top \tpb (\tf-F)  e_{t+1}^\top \right]  \bigg \}
\end{align*}

Note that these terms follow similarly with \citet{lale2020regret,lale2020root}. Using the same decompositions, we can trivially show that for explore and commit approach of Algorithm \ref{algo_quad}:
\begin{align*}
    R_1\!=\!\tilde{\OO}\left(\! \frac{T\!-\!T_{warm}}{\sqrt{T_{warm}}}\!\right),& R_2 =  \tilde{\OO}\left( \sqrt{T\!-\!T_{warm}} \right), |R_3|\!=\!\tilde{\OO}\left(\! \frac{T\!-\!T_{warm}}{\sqrt{T_{warm}}}\!\right), \\
    |R_4|\!&=\!\tilde{\OO}\left(\! \frac{T\!-\!T_{warm}}{\sqrt{T_{warm}}}\!\right), |R_5|\!=\!\tilde{\OO}\left(\! \frac{T\!-\!T_{warm}}{\sqrt{T_{warm}}}\!\right)
\end{align*}
Combining these gives the advertised bound in Theorem \ref{thm:apx_reg_explore} by the choice of $T_{warm} \geq \Tw$.

For Algorithm \ref{algo_quad} with closed-loop model estimate updates, due to doubling epoch lengths and the choice of $T_{warm} \geq \tau$, we have 
\begin{align*}
    R_{i} = \tilde{\OO} \left( \frac{T_{warm}}{\sqrt{T_{warm}}} + \frac{2T_{warm}}{\sqrt{2T_{warm}}} + \frac{4T_{warm}}{\sqrt{4T_{warm}}} + \ldots \right) = \tilde{\OO} \left( \sqrt{T} \right)
\end{align*}
for $i=1,3,4,5$ via Lemma \ref{doublingtrick} (doubling trick) and $R_2 =  \tilde{\OO}\left( \sqrt{T\!-\!T_{warm}} \right)$. Combining these gives the advertised bound in Theorem \ref{thm:apx_reg_adapt}.  
\end{proof}

\section{Optimal Control of ARX System with Convex Quadratic Cost and Bellman Optimality}
\label{apx:bellman}
From the first principles \citep{bertsekas1995dynamic}, the value function of the given system is quadratic and due to stochasticity we have the following format:
\begin{equation*}
     V(x,y) = \begin{bmatrix}
        x \\
        y
\end{bmatrix}^\top 
\begin{bmatrix}
P_{11} & P_{12} \\
P_{21} & P_{22}
\end{bmatrix}
  \begin{bmatrix}
        x \\
        y
\end{bmatrix} + \lambda 
\end{equation*}

Using average cost optimality equation, we can determine the value function for the given system $\Theta$ as follows: 

\begin{align*}
    \begin{bmatrix}
        x \\
        y
\end{bmatrix}^\top 
\begin{bmatrix}
P_{11} & P_{12} \\
P_{21} & P_{22}
\end{bmatrix}
  \begin{bmatrix}
        x \\
        y
\end{bmatrix} + \lambda &= \min_u \Bigg\{ y^\top Q y + u^\top R u \\
&+  \mathbb{E}\left[ \begin{bmatrix}
        Ax + Bu + Fy \\
        CAx + CBu + CFy + e
\end{bmatrix}^\top 
\begin{bmatrix}
P_{11} & P_{12} \\
P_{21} & P_{22}
\end{bmatrix}
  \begin{bmatrix}
        Ax + Bu + Fy \\
        CAx + CBu + CFy + e
\end{bmatrix}\right] \Bigg\} 
\end{align*}
Expanding all and minimizing for $u$ gives the optimal control of 
\begin{equation*}
    u = - (R + B^\top \mathbf{P} B)^{-1} \left[ B^\top \mathbf{P} A x + B^\top \mathbf{P} F y \right]
\end{equation*}
where $\mathbf{P} = P_{11}+ P_{12}C + C^\top P_{21} + C^\top P_{22} C$. Inserting the expression for $u$, we have 

\begin{align*}
    &x^\top P_{11} x + x^\top P_{12} y + y^\top P_{21} x + y^\top P_{22} y + \lambda = \\
& x^\top A^\top \mathbf{P} B (R + B^\top \mathbf{P} B)^{-1} B^\top \mathbf{P} A x + 2x^\top A^\top \mathbf{P} B (R + B^\top \mathbf{P} B)^{-1} B^\top \mathbf{P} F y + y^\top F^\top \mathbf{P} B (R + B^\top \mathbf{P} B)^{-1} B^\top \mathbf{P} F y \\
&- 2x^\top A^\top \mathbf{P} B (R + B^\top \mathbf{P} B)^{-1} B^\top \mathbf{P} A x -2x^\top A^\top \mathbf{P} B (R + B^\top \mathbf{P} B)^{-1} B^\top \mathbf{P} F y \\
&-2y^\top F^\top \mathbf{P} B (R + B^\top \mathbf{P} B)^{-1} B^\top \mathbf{P} A x - 2y^\top F^\top \mathbf{P} B (R + B^\top \mathbf{P} B)^{-1} B^\top \mathbf{P} F y \\
&+y^\top Q y + x^\top A^\top \mathbf{P} A x + y^\top F^\top \mathbf{P} F y + 2x^\top A^\top \mathbf{P} F y + \Tr(P_{22} E)
\end{align*}
From this, we get $\lambda = \Tr(P_{22} E)$ where 
\begin{align*}
 &   \begin{bmatrix}
        x \\
        y
\end{bmatrix}^\top 
\begin{bmatrix}
P_{11} & P_{12} \\
P_{21} & P_{22}
\end{bmatrix}
  \begin{bmatrix}
        x \\
        y
\end{bmatrix} \\
&= \begin{bmatrix}
        x \\
        y
\end{bmatrix}^\top 
\begin{bmatrix}
 A^\top \left(\mathbf{P} - \mathbf{P} B (R + B^\top \mathbf{P} B)^{-1} B^\top \mathbf{P} \right) A  & A^\top \left(\mathbf{P} - \mathbf{P} B (R + B^\top \mathbf{P} B)^{-1} B^\top \mathbf{P} \right) F  \\
F^\top \left(\mathbf{P} - \mathbf{P} B (R + B^\top \mathbf{P} B)^{-1} B^\top \mathbf{P} \right) A  & Q + F^\top \left(\mathbf{P} - \mathbf{P} B (R + B^\top \mathbf{P} B)^{-1} B^\top \mathbf{P} \right) F 
\end{bmatrix}
  \begin{bmatrix}
        x \\
        y
\end{bmatrix}
\end{align*}
This must hold for all $x$ and $y$. Therefore, using the definition of $\mathbf{P}$, we conclude that $\mathbf{P}$ satisfies the DARE
\begin{equation}
    \mathbf{P} = C^\top Q C + (A+FC)^\top \mathbf{P} (A+FC) - (A+FC)^\top \mathbf{P} B (R+B^\top \mathbf{P} B)^{-1} B^\top \mathbf{P} (A+FC)
\end{equation}
and the infinite horizon optimal cost this system is 
\begin{equation}
    J_\star (\Theta) = \Tr\left(E \left(Q + F^\top \left(\mathbf{P} - \mathbf{P} B (R + B^\top \mathbf{P} B)^{-1} B^\top \mathbf{P} \right) F \right)\right)
\end{equation}

\noindent \textbf{Proof of Lemma \ref{lem:bellman}}
Suppose the differential cost $h$ is a quadratic function of $s_t$ where $s_t = [x_{t}^\top~y_t^\top]^\top \in \R^{n+m}$,\textit{i.e.}
\[ h(s_t) = s_t^\top  
\left[
    \begin{array}{cc}{G_1} & {  G_2  } \\ {G_2^\top} & {  G_3  }  \end{array} \right] s_t = x_t^\top G_1 x_t + 2 x_t^\top G_2 y_t + y_t^\top G_3 y_t. \] 
One needs to verify that there exists $G_1$, $G_2$, $G_3$ such that they satisfy Bellman optimality equation for the chosen differential cost:
\begin{align}
    &J_\star(\Theta) + x_t^\top G_1 x_t + 2 x_t^\top G_2 y_t + y_t^\top G_3 y_t  = \label{bellman}\\
    & y_t^\top Q y_t + u_\star^\top R u_\star + \mathbb{E}\left[ x_{t+1, u_\star}^\top G_1 x_{t+1, u_\star} + 2 x_{t+1, u_\star}^\top G_2 y_{t+1, u_\star} + y_{t+1, u_\star}^\top G_3 y_{t+1, u_\star} \right] \nonumber
\end{align}
for $u_\star = - (R + B^\top \mathbf{P} B)^{-1}B^\top \mathbf{P} \left[  A x_t + F y_t \right]$. In order to match the stochastic term due to $e_t$ in expectation with $J_\star(\Theta) = \Tr\left(E \left(Q + F^\top \left(\mathbf{P} - \mathbf{P} B (R + B^\top \mathbf{P} B)^{-1} B^\top \mathbf{P} \right) F \right)\right)$, set $G_3 = Q + F^\top \left(\mathbf{P} - \mathbf{P} B (R + B^\top \mathbf{P} B)^{-1} B^\top \mathbf{P} \right) F $. Let $M = (R + B^\top \mathbf{P} B)^{-1}B^\top \mathbf{P}$. Inserting $G_3$ to (\ref{bellman}), we get following 3 equations to solve for $G_1$ and $G_2$:

\textbf{1)} From quadratic terms of $y_t$:
\begin{align*}
    Q &+ F^\top \left(\mathbf{P} - \mathbf{P} B (R + B^\top \mathbf{P} B)^{-1} B^\top \mathbf{P} \right) F \\
    &= Q + F^\top M^\top R M F + F^\top (I-BM)^\top G_1 (I-BM) F + 2 F^\top (I-BM) G_2 C (I-BM) F \\
    &+ F^\top C^\top(I - BM)^\top \left( Q + F^\top \left(\mathbf{P} - \mathbf{P} B (R + B^\top \mathbf{P} B)^{-1} B^\top \mathbf{P} \right) F \right) C (I-BM) F  
\end{align*}

\textbf{2)} From quadratic terms of $x_t$:
\begin{align*}
    G_1 &= A^\top M^\top R M A + A^\top (I - BM)^\top G_1 (I-BM) A + 2 A^\top (I-BM)^\top G_2 C (I-BM) A \\
    &+ A^\top C^\top(I - BM)^\top \left( Q + F^\top \left(\mathbf{P} - \mathbf{P} B (R + B^\top \mathbf{P} B)^{-1} B^\top \mathbf{P} \right) F \right) C (I-BM) A 
\end{align*}

\textbf{3)} From bilinear terms of $x_t$ and $y_t$:
\begin{align*}
    G_2 &= A^\top M^\top R M F + A^\top (I - BM)^\top G_1 (I-BM) F + 2 A^\top (I-BM)^\top G_2 C (I-BM) F \\
    &+ A^\top C^\top(I - BM)^\top \left( Q + F^\top \left(\mathbf{P} - \mathbf{P} B (R + B^\top \mathbf{P} B)^{-1} B^\top \mathbf{P} \right) F \right) C (I-BM) F 
\end{align*}
$G_1 = A^\top\left(\mathbf{P} - \mathbf{P} B (R + B^\top \mathbf{P} B)^{-1} B^\top \mathbf{P} \right)A$ and $G_2 = A^\top \left(\mathbf{P} - \mathbf{P} B (R + B^\top \mathbf{P} B)^{-1} B^\top \mathbf{P} \right) F$ satisfy all 3 equations. This one can write Bellman optimality equation as 
\begin{align*}
    &J_\star(\Theta) + x_t^\top A^\top\left(\mathbf{P} - \mathbf{P} B (R + B^\top \mathbf{P} B)^{-1} B^\top \mathbf{P} \right)A x_t + 2 x_t^\top A^\top \left(\mathbf{P} - \mathbf{P} B (R + B^\top \mathbf{P} B)^{-1} B^\top \mathbf{P} \right) F y_t \\
    &+ y_t^\top Q y_t + y_t^\top F^\top \left(\mathbf{P} - \mathbf{P} B (R + B^\top \mathbf{P} B)^{-1} B^\top \mathbf{P} \right) F y_t  = \\
    & y_t^\top Q y_t + u_t^\top R u_t + \mathbb{E}\bigg[ x_{t+1}^\top A^\top\left(\mathbf{P} - \mathbf{P} B (R + B^\top \mathbf{P} B)^{-1} B^\top \mathbf{P} \right)A x_{t+1} \\
    &+ 2 x_{t+1}^\top A^\top \left(\mathbf{P} - \mathbf{P} B (R + B^\top \mathbf{P} B)^{-1} B^\top \mathbf{P} \right) F y_{t+1} + y_{t+1}^\top Q y_{t+1} \\
    &+ y_{t+1}^\top F^\top \left(\mathbf{P} - \mathbf{P} B (R + B^\top \mathbf{P} B)^{-1} B^\top \mathbf{P} \right) F y_{t+1} \bigg] 
\end{align*}
Thus, we get the following Bellman optimality equation:
\begin{align*}
    J_\star(\Theta) &+ (A x_t + F y_t)^\top  \left(\mathbf{P} - \mathbf{P} B (R + B^\top \mathbf{P} B)^{-1} B^\top \mathbf{P} \right) (A x_t + F y_t) + y_t^\top Q y_t \\ &= y_t^\top Q y_t + u_t^\top R u_t
   \\
   &+\mathbb{E} \left[ (A x_{t+1} + F y_{t+1})^\top \left(\mathbf{P} - \mathbf{P} B (R + B^\top \mathbf{P} B)^{-1} B^\top \mathbf{P} \right) (A x_{t+1} + F y_{t+1}) + y_{t+1}^\top Q y_{t+1}\right]
\end{align*}

\null\hfill$\square$

\section{Technical Lemmas and Theorems} \label{technical}
\begin{theorem}[Matrix Azuma \citep{tropp2012user}]\label{azuma} Consider a finite adapted sequence $\left\{\boldsymbol{X}_{k}\right\}$ of self-adjoint matrices
in dimension $d,$ and a fixed sequence $\left\{\boldsymbol{A}_{k}\right\}$ of self-adjoint matrices that satisfy
\begin{equation*}
\mathbb{E}_{k-1} \boldsymbol{X}_{k}=\mathbf{0} \text { and } \boldsymbol{A}_{k}^{2} \succeq \boldsymbol{X}_{k}^{2}  \text { almost surely. }
\end{equation*}
Compute the variance parameter
\begin{equation*}
\sigma^{2}\coloneqq \left\|\sum_{k} \boldsymbol{A}_{k}^{2} \right\|
\end{equation*}
Then, for all $t \geq 0$
\begin{equation*}
\mathbb{P}\left\{\lambda_{\max }\left(\sum_{k} \boldsymbol{X}_{k} \right) \geq t\right\} \leq d \cdot \mathrm{e}^{-t^{2} / 8 \sigma^{2}}
\end{equation*}

\end{theorem}

\begin{lemma}[Regularized Design Matrix Lemma \citep{abbasi2011improved}] \label{lem:logdet}
When the covariates satisfy $\left\|z_{t}\right\| \leq c_{m},$ with some $c_{m}>0$ w.p.1 then
\[
\log \frac{\operatorname{det}\left(V_{t}\right)}{\operatorname{det}(\lambda I)} \leq d \log \left(\frac{\lambda d +t c_{m}^{2}}{\lambda d}\right)
\]
where $V_t = \lambda I + \sum_{i=1}^t z_i z_i^\top$ for $z_i \in \mathbb{R}^d$.
\end{lemma}

\begin{lemma}[Lemma 8.1 of \citet{simchowitz2020improper}] \label{Lemma8.1}
For any $\Mcontrol \in \mathcal{M}$, let $f_{t}^{\text {pred }}(\Mcontrol)$ denote the unary counterfactual loss function induced by true truncated counterfactuals (Definition 8.1 of \citet{simchowitz2020improper}). During the $i$'th epoch of adaptive control period, at any time step $t\in[t_i,~ \ldots,~t_{i+1}-1]$, for all $i$, we have that
\begin{equation*}
\left\|\nabla f_t\left(\Mcontrol,\wh \Markov_i,\nature_1(\wh \Markov_i),\ldots,\nature_t(\wh \Markov_i)\right)-\nabla f_{t}^{\text {pred }}(\Mcontrol)\right\|_{\mathrm{F}} \leq C_{\text {approx }} \epsilon_\Markov(i, \delta),
\end{equation*}
where $C_{\text {approx }} \coloneqq \sqrt{H'} \kappa_\Markov \kappa_{\Mcontrolset} \kappa_\nature^2 \left(16\smooth +24 L\right)$.
\end{lemma}

\begin{theorem}[Theorem 8 of \citet{simchowitz2020improper}] \label{theo:fpred}
Suppose that $\mathcal{K} \subset \mathbb{R}^{d}$ and $h \geq 1$. Let $F_{t}:=\mathcal{K}^{h+1} \rightarrow \mathbb{R}$ be a sequence of $L_{c}$ coordinatewise-Lipschitz functions with the induced unary functions
$f_{t}(x):=F_{t}(x, \ldots, x)$ which are $L_{\mathrm{f}}$-Lipschitz and $\beta$-smooth. Let $f_{t ; k}(x):=\mathbb{E}\left[f_{t}(x) | \mathcal{F}_{t-k}\right]$ be $\alpha$-strongly convex on $\mathcal{K}$ for a filtration $\left(\mathcal{F}_{t}\right)_{t \geq 1}$. Suppose that $z_{t+1}=\Pi_{\mathcal{K}}\left(z_{t}-\eta \boldsymbol{g}_{t} \right)$, where $\boldsymbol{g}_{t}=\nabla f_{t}\left(z_{t}\right)+\epsilon_{t}$ for  $\left\|\boldsymbol{g}_{t}\right\|_{2} \leq L_{\mathbf{g}},$ and $\operatorname{Diam}(\mathcal{K}) \leq D$. Let the gradient descent iterates be applied for $t \geq t_{0}$ for some $t_{0} \leq k,$ with $z_{0}=z_{1}=\cdots=z_{t_{0}} \in \mathcal{K}$ for $k\geq 1$. Then with step size $\eta_{t}=\frac{3}{\alpha t},$ the
following bound holds with probability $1 - \delta$ for all comparators $z_{\star} \in \mathcal{K}$ simultaneously:
\begin{align*}
\sum_{t=k+1}^{T} &f_{t}\left(z_{t}\right)-f_{t}\left(z_{\star}\right)-\left(\frac{6}{\alpha} \sum_{t=k+1}^{T}\left\|\boldsymbol{\epsilon}_{t}\right\|_{2}^{2}-\frac{\alpha}{12} \sum_{t=1}^{T}\left\|z_{t}-z_{\star}\right\|_{2}^{2}\right) \\ & \lesssim \alpha k D^{2}+\frac{\left(k L_{\mathrm{f}}+h^{2} L_{\mathrm{c}}\right) L_{\mathrm{g}}+k d L_{\mathrm{f}}^{2}+k \beta L_{\mathrm{g}}}{\alpha} \log (T)+\frac{k L_{\mathrm{f}}^{2}}{\alpha} \log \left(\frac{1+\log \left(e+\alpha D^{2}\right)}{\delta}\right)
\end{align*}
\end{theorem}

\begin{lemma}[Lemma 8.2 of \citet{simchowitz2020improper}] \label{lem:smooth}
For any $\Mcontrol \in \mathcal{M}$, $f_{t}^{\text {pred }}(\Mcontrol)$ is $\beta$-smooth, where $\beta = 16H'\kappa_\nature^2 \kappa_\Markov^2\smooth$.
\end{lemma}

\begin{lemma}[Lemma 8.3 of \citet{simchowitz2020improper}] \label{lem:strongcvx}
For any $\Mcontrol \in \mathcal{M}$, given $\epsilon_\Markov(i, \delta) \leq \frac{1}{4\kappa_\nature\kappa_\Mcontrolset\kappa_{\Markov}}\sqrt{\frac{\alpha}{H'\strong}}$, conditional unary counterfactual loss function induced by true counterfactuals are $\alpha/4$ strongly convex. 
\end{lemma}

\begin{lemma}[Lemma 8.4 of \citet{simchowitz2020improper}] \label{lem:lipschitz}
Let $L_f = 4L\sqrt{H'}\kappa_\nature^2\kappa_\Markov^2\kappa_\Mcontrolset$. For any $\Mcontrol \in \mathcal{M}$ and for $\Tburn \geq \Tmax$, $f_{t}^{\text{pred}}(\Mcontrol)$ is $4 L_{f}$-Lipschitz, $f_{t}^{\text{pred}}\left[\Mcontrol_{t:t-H}\right]$ is $4 L_{f}$ coordinate Lipschitz. Moreover,
$\max_{\Mcontrol \in \mathcal{M}}\left\| \nabla f_t\left(\Mcontrol,\wh \Markov_i,\nature_1(\wh \Markov_i),\ldots,\nature_t(\wh \Markov_i)\right) \right\|_{2} \leq 4 L_{f}$.
\end{lemma}

\begin{lemma}[Doubling Trick \citep{jaksch2010near}]\label{doublingtrick}
For any sequence of numbers $z_{1}, \ldots, z_{n}$ with $0 \leq z_{k} \leq Z_{k-1} \coloneqq \max \left\{1, \sum_{i=1}^{k-1} z_{i}\right\}$
\begin{equation*}
   \sum_{k=1}^{n} \frac{z_{k}}{\sqrt{Z_{k-1}}} \leq(\sqrt{2}+1) \sqrt{Z_{n}}
\end{equation*}
\end{lemma}

\end{document}